\documentclass[11pt]{article}

\usepackage[preprint]{acl}

\usepackage{times}
\usepackage{latexsym}
\usepackage[T1]{fontenc}
\usepackage[utf8]{inputenc}
\usepackage{microtype}
\usepackage{inconsolata}
\usepackage{graphicx}
\usepackage{amsmath, amssymb, amsthm}
\usepackage{tikz}
\usepackage{pgfplots}
\usepgfplotslibrary{groupplots}
\usepackage{booktabs}
\usepackage{multirow}
\usepackage{enumitem}
\pgfplotsset{compat=newest}
\usepackage{pgfplotstable}
\usepackage{float}
\usepackage{dblfloatfix}
\usepackage[most]{tcolorbox}
\usepackage{pifont}
\usepackage{algorithm}
\usepackage{algpseudocode}

\newtheorem{proposition}{Proposition}

\title{Failed Reasoning Traces Tell You What Is Fixable\\(But Not by Reading Them)}

\author{
  Nizar Islah$^{1,2}$ \quad Istabrak Abbes$^{1,2}$ \quad Irina Rish$^{1,2}$ \quad Sarath Chandar$^{1,3}$ \quad Eilif B. Muller$^{1,2,4}$ \\
  $^1$Mila - Quebec AI Institute \quad $^2$Universit\'e de Montr\'eal \\
  $^3$Polytechnique Montr\'eal \quad $^4$CHU Sainte-Justine \\
  \texttt{nizar.islah@mila.quebec} \\}

\begin{document}
\maketitle

\begin{abstract}

When post-trained language models fail on reasoning problems, the common test-time-scaling response is to spend more compute on additional attempts, and the failed traces play no further role. We argue this discards a crucial signal; some failures come from unlucky sampling, where more rollouts help, while others are structural and resist resampling regardless of budget. We propose that failed traces encode recoverability structure: the inference-time signature of which test-time interventions can rescue a given failure. Three problem-level trajectory features, derived from the structure of available interventions, recover this structure from the distributional signature of failed rollouts, not their text. They cluster failures into stable regimes, characterize the failure topography of different post-training methods ($84.3{\pm}4.3\%$ accuracy, $+20\%$ over a majority-class baseline), and support a training-free routing rule that lifts rescue by $+12.2\%$ on the deployment-relevant Steerable-Hard subset (failures where retry is insufficient and a bounded intervention is reachable). The features and the routing rule transfer across two cross-family probes. The same three features thus convert failed traces from discarded data into a diagnostic object, supporting test-time routing and post-training analysis without training-time or weight-space access.
\end{abstract}

\section{Introduction}
\label{sec:intro}

\begin{figure*}[!t]
    \centering
    \includegraphics[width=\textwidth]{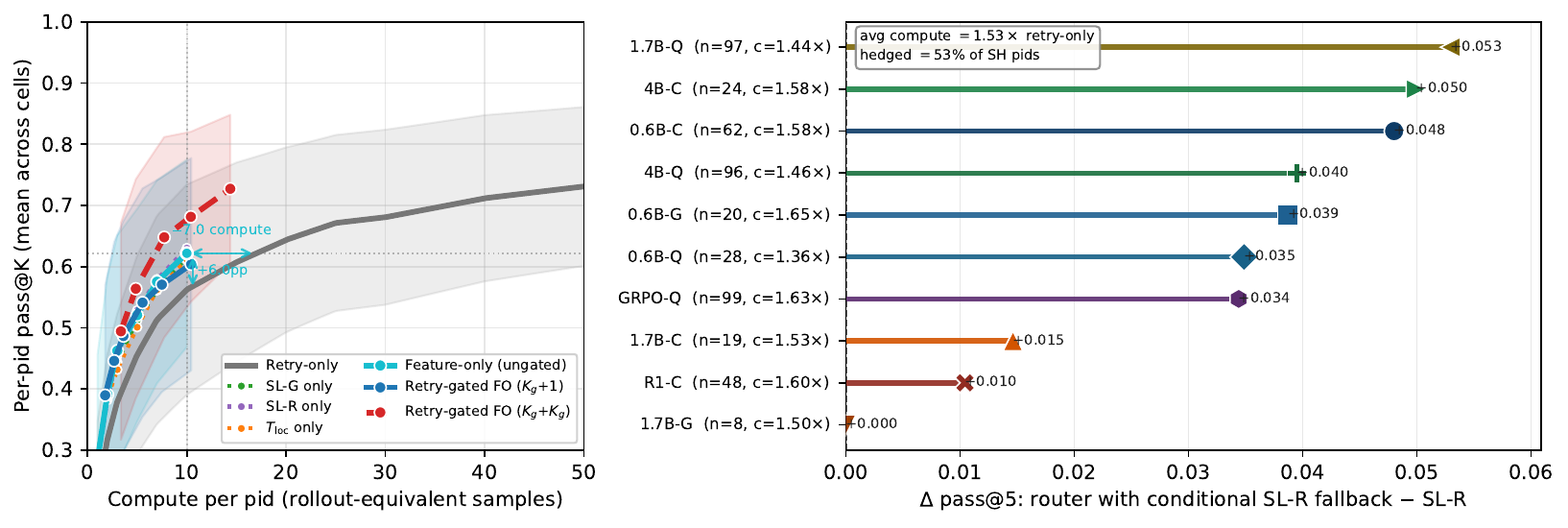}
    \caption{\textbf{Motivation: routing test-time compute.} A router
    that reads the three trajectory features turns the structure of failed
    traces into a test-time-compute policy. The two panels make the case
    from two angles: Panel (a) is the accuracy-vs-compute Pareto curve,
    showing the router reaches retry-level rescue at far lower compute, and
    Panel (b) shows that the router rescues Steerable-Hard (SH) problems
    that no single fixed operator covers.
    \textbf{(a)} Pareto frontier of per-problem pass rate against compute,
    averaged over 10 (model, task) cells; better policies sit toward the
    upper left (higher rescue, less compute). The retry-only baseline
    (solid grey) is measured at $K{=}50$ rollouts; the single operators
    (SL-G, SL-R, $T_{\rm loc}$), Feature-only routing (cyan), and its
    retry-gated variants ($K_g{+}1$, $K_g{+}K_g$) are all measured at
    $K{=}10$. The cyan crosshairs mark Feature-only routing at $K{=}10$: it
    matches the per-problem rescue rate of retry at $K{=}50$ using
    substantially less compute.
    \textbf{(b)} Per-cell lift of the router over its single fallback
    operator on the Steerable-Hard subset (problems with retry$@32{=}0$),
    across 11 cells (one more than Panel a). Each horizontal bar gives the
    mean per-problem gain $\Delta = (\text{router with conditional SL-R
    fallback}) - \text{SL-R}$ in unbiased pass$@5$. Because the fallback
    policy is a strict superset of SL-R, $\Delta$ is non-negative by
    construction; the $x$-axis therefore starts at $0$, and longer bars
    mark cells where the router rescues problems SL-R alone cannot. Marker
    shape and color identify the cell; $n$ is its SH problem count and $c$
    the average compute in operator-equivalents ($1.0{\times}$ when the
    router's choice is kept, $2.0{\times}$ when the SL-R fallback fires).}
    \label{fig:flops-pareto}
\end{figure*}


Reasoning models are increasingly evaluated and deployed under test-time compute budgets. When a model fails, the standard response is to sample again. Best-of-$N$ \citep{stiennon2020summarize,cobbe2021verifiers}, self-consistency \citep{wang2023selfconsistency}, and retry-style evaluation \citep{chen2021codex}, along with recent test-time scaling that allocates more inference compute \citep{snell2024scaling,deepseekai2025r1,muennighoff2025s1}, all spend additional compute by drawing more trajectories from the same post-trained model distribution. This is effective when the first failure was an unlucky sample. But it is wasteful when the model is locked into a stable wrong trajectory, and incomplete when the failure is locally recoverable only by changing the trajectory rather than resampling it. In deployment, the central question becomes \emph{given that the model has already failed, what should the next unit of compute do?}


Aggregate accuracy cannot answer this question. It maps every failed trace to the same label, although failures may arise from sampling noise, a single demoted reasoning step, a localized misrouting event, or a trace-wide deformation of the model's reasoning dynamics. These cases require different responses: retry, local perturbation, logit-space steering, or no further spending under the available budget. We argue that the failed trace itself is the diagnostic object that distinguishes them. The signal we read is distributional, the per-token probability signature of the trace rather than its natural-language content, which separates this diagnostic from verbal self-correction that re-reads and critiques the text.

We study failed traces through an \textbf{inference-time, post-hoc, trace-level diagnostic}. Retry and temperature-based resampling are rank-preserving: they can reweight the specialist's local distribution but cannot make a lower-ranked token become the local mode. Logit steering toward a lineage ancestor acts in natural-parameter space (it averages logits, not probabilities) and can invert local ranks when the specialist and ancestor disagree. This operator-class distinction determines what a trace-level diagnostic should measure.

We focus on the subset where routing has deployment value. A problem-unit is \emph{Steerable-Hard} when standard retry is insufficient but at least one bounded test-time intervention is reachable. Outside this subset, routing is either unnecessary, because retry already works, or futile, because none of the available operators rescues the problem.

From failed rollouts alone, we extract three problem-level features derived from the operator geometry: deformation spread, junction concentration, and local displacement budget. These features support both a routing example (\S\ref{sec:routing-result}) and an audit channel (\S\ref{sec:population}). The routing example uses a training-free rule fixed by the operator class: diffuse deformation routes to dense steering, concentrated junctions to sparse steering, and high local headroom to local temperature lifting.

\paragraph{Contributions.}
We contribute (i) a \emph{reframing} of failed reasoning traces as
diagnostic objects, asking what the trace itself tells us about
test-time recoverability: which test-time interventions can rescue
which failures, and on what scope (\emph{Steerable-Hard}: retry
insufficient and a bounded intervention reachable); (ii) a \emph{theoretical lens}
that derives the three trajectory features from operator-class
structure (rank-preserving vs natural-parameter, with a closed-form
sufficient-not-necessary inversion bound;
Props.~\ref{thm:rank-preservation}, \ref{thm:rank-inversion};
\S\ref{sec:framework}); (iii) a \emph{diagnostic audit channel}: the same features recover
SFT-vs-RL post-training distinctions from failed rollouts alone,
without training-time, weight-space, or paired-comparison access
(\S\ref{sec:population}); and
(iv) a \emph{routing example} demonstrating the features are
deployment-actionable, via a training-free operator-class routing rule
on \emph{Steerable-Hard} (Eq.~\ref{eq:steerable-hard}: retry insufficient
\emph{and} a bounded intervention reachable;
\S\ref{sec:prospective}); and (v) \emph{cross-family transfer} of
both the audit signatures and the routing rule between Qwen3,
R1-Distill-Qwen-Math-1.5B, and Phi-4-mini-reasoning, giving evidence
that the features encode operator-class structure rather than
specialist-specific artifacts (\S\ref{sec:population},
App.~\ref{app:cross-family}). Code and the feature-extraction
pipeline will be released upon publication.

\section{Setup}
\label{sec:setup}

\paragraph{Operators.} Five test-time interventions span two
mechanism classes (Table~\ref{tab:operators}; Proposition derivations in
App.~\ref{app:moved-methods}). \textbf{Rank-preserving}: \textbf{retry}
(re-sample at $T{=}0.6$)\footnote{\emph{retry@$K$} is functionally
\emph{best-of-$K$} sampling: take $K$ independent rollouts and
succeed if any pass. We use the term ``retry'' because the dispatch's
failure set is conditioned on an already-failed rollout pool
(the initial $K{=}10$ rollouts on which Fail@$K{=}10$ is
defined); each additional ``retry'' is functionally one more
best-of sample on top of that pool.} and \textbf{local temperature}
($T_{\mathrm{loc}}{=}1.5$ at a detected junction) re-weight the
support but cannot invert local token ranks. \textbf{Natural-parameter}:
\textbf{SL-G}, \textbf{SL-R}, and \textbf{DL} apply logit
mixing $z_\alpha = \alpha z_S + (1-\alpha) z_A$ at $\alpha{=}0.7$
(an $e$-geodesic toward the pre-training checkpoint
$z_A$; App.~\ref{app:info-geom}, all $\alpha \in [0.5, 0.8]$ beat
retry; per-cell optima vary and are reported in
App.~\ref{app:alpha-sensitivity}\footnote{The
canonical $\alpha{=}0.7$ choice is near-optimal on most cells but
not all; e.g., Qwen3 SFT-1.7B $\times$ CruxEval peaks at $\alpha{=}0.6$
with a $+25$ pp SH lift gap, per the family-wise grid in
App.~\ref{app:alpha-sensitivity}.}), differing
in \emph{where}: SL-G at the detector-detected junction
(App.~\ref{app:junction-detection}), SL-R at a uniformly random
position, DL at every position. SL-G$\leftrightarrow$SL-R
isolates the value of position selection at matched steering strength.

\begin{table}[t]
\centering\small
\begin{tabular}{lll}
\toprule
Operator & Acts at & Class \\
\midrule
retry & full re-sample & rank-preserving \\
$T_{\mathrm{loc}}$ & detected junction & rank-preserving \\
SL-G & detected junction & rank-inverting \\
SL-R & random position & rank-inverting \\
DL & every position & rank-inverting \\
\bottomrule
\end{tabular}
\caption{The five test-time operators: where each acts and its
mechanism class (\S\ref{sec:framework}). Rank-preserving operators
re-weight the specialist's distribution; rank-inverting operators
(logit mixing toward the ancestor at $\alpha{=}0.7$) can move a
lower-ranked token to the local mode.}
\label{tab:operators}
\end{table}

\subsection{Why these features: the operator class determines what is measurable}
\label{sec:framework}

\paragraph{What a test-time operator can actually do locally.}
When a trace has already failed, the question is not whether to spend
more compute but \emph{what the available operators can change about
it}. Retry re-samples from the same specialist distribution: the most
likely next token stays the most likely. Local temperature flattens
that distribution but cannot move a different token to the top. Logit
mixing with the lineage ancestor is qualitatively different: a token
the specialist confidently prefers can be suppressed if the ancestor
disagrees. Two operator classes, two reaches.

\paragraph{The geometric distinction.}
Probability mixing (the \emph{$m$-geodesic},
$\tilde p = \beta p_S + (1-\beta)p_A$) is a weighted average; the
argmax token survives whenever the specialist's top margin is large
enough. Logit steering (the \emph{$e$-geodesic},
$p_\alpha \propto p_S^\alpha p_A^{1-\alpha}$) is a product-of-experts:
a token must have non-negligible probability under \emph{both} models
to stay top. Only the second class can invert local ranks. Retry and
$T_{\mathrm{loc}}$ are rank-preserving; SL-G, SL-R, and DL are
rank-inverting.

\paragraph{What the operator class forbids, and what it permits.}
Rank-preserving operators cannot rescue a Rank Misrouting failure
(Prop.~\ref{thm:rank-preservation}), while the $e$-geodesic admits a
closed-form rank-inversion threshold $\alpha^\star$ with local
displacement budget equal to the Fisher information of the
specialist's local temperature submodel
(Prop.~\ref{thm:rank-inversion}; App.~\ref{app:moved-methods}),
measured as $V_{t^\star}$ at the worst pivot the trace exposes. The
$\alpha^\star$ condition of Prop.~\ref{thm:rank-inversion} is
sufficient for single-position rank inversion at $\alpha\!{<}\!\alpha^\star$;
it is not necessary for rescue. Empirically, only a minority of
SH-rescued pids meet the condition at the engine-detected junction,
and most engine-selected injection positions place the specialist
and ancestor on the same argmax token (App.~\ref{app:bucket-analysis}). The
operator class therefore predicts the right \emph{class} of
intervention (rank-preserving vs.\ natural-parameter), even when the
per-pid rescue mechanism is more diffuse than the closed-form picture.

\paragraph{From the geometry to the features.}
Each of the three problem-level features measures one geometric
quantity the operator class makes actionable:
$\bar J_{\mathrm{frac+}}$ (density of $e$-geodesic deformation across
the trace $\to$ dense intervention $\mathrm{DL}$);
$\bar C{=}\bar J_{\max}/\bar J_{\mathrm{mean}}$ (concentration $\to$
sparse intervention at the right position $\mathrm{SL\text{-}G}$);
$V_{t^\star}$ (local displacement budget at the worst pivot $\to$
$T_{\mathrm{loc}}$ when high, where bounded steering still permits
rank inversion). The operator the geometry prescribes is the one
whose dominant feature is highest. This is not an empirical fit; it
is the rule Eq.~\ref{eq:prospective-rule} implements in
\S\ref{sec:prospective}, tested by the dispatch in
\S\ref{sec:routing-result}.

As a motivating example we contrast $V_{t^\star}$ on trap vs.\ logic
tokens (App.~Fig.~\ref{fig:h-vs-vt}).

\paragraph{Specialists and tasks.} Two post-training regimes against
the same Qwen3 lineage ancestors carry the headline dispatch claim:
SFT and GRPO at three scales (0.6B, 1.7B, 4B), trained on
Bespoke-Stratos-17k \citep{bespokelabs2025stratos}. Two
\textbf{cross-family probes} on structurally different model families
and post-training methods cover generalization beyond the Qwen3
lineage: \textbf{R1-Distill-Qwen-Math-1.5B} (Qwen2.5-Math base,
RL-distilled from R1) and \textbf{Phi-4-mini-reasoning}
(Phi-3 architecture, instruction-tuned reasoning); details in
App.~\ref{app:cross-family}. Three held-out evaluation tasks:
\textbf{CruxEval} \citep{gu2024cruxeval} (code input/output
prediction), \textbf{GSM8K} \citep{cobbe2021verifiers} (grade-school
math), \textbf{GPQA} \citep{rein2023gpqa} (graduate scientific
reasoning), each at $k{=}10$ rollouts per problem, $T{=}0.6$.
Bootstrap 95\% CIs and full hyperparameters: App.~\ref{app:bootstrap-cis},
\ref{app:training-details}.

\paragraph{Problem-unit framing.} We evaluate per-problem-unit (one
problem ID summarized over its $k$ failed rollouts) rather than
per-cell to avoid averaging over heterogeneous failure populations.
Across the 10 evaluation cells, $N{=}1{,}625$ problem-units fail at
least once; $1{,}423$ of these carry operator outcomes (the pool on
which all routing claims below are computed), and $718$ survive
Fail@$K{\geq}10$ (all $K$ rollouts incorrect at $T{=}0.6$). A problem is in Fail@$K$ when standard sampling has
already failed, so any recovery measures something beyond exhaustive
search.

\paragraph{Dispatch protocol.}
\textbf{Repair@3}: each (pid, operator) gets three single-position
injection attempts drawn from the pid's three deepest available
conditioning levels (App.~\ref{app:led-cache}); any-of-three defines
the rescue. Per dispatch cell (regime $\times$ $\bar V_{\mathrm{HL}}$)
and operator we compute $U_i$ (unique rescues), $R_i = U_i/N_{\mathrm{cell}}$,
and the dispatch score
\begin{equation}
\label{eq:dispatch-score}
S_i = U_i \cdot R_i^{\gamma}, \qquad \gamma = 0.5
\end{equation}
(a coverage-rate blend stable across $\gamma \in [0.3, 0.8]$;
App.~\ref{app:gamma-sensitivity}). The policy routes each pid to its
cell's $\arg\max_i S_i$ operator (§\ref{sec:routing-result}).

\paragraph{Trajectory features (used by every subsequent section).}
Three problem-level trajectory features ($\bar J_{\mathrm{frac+}}$,
$\bar C$, $V_{t^\star}$) and the routing key $\bar V_{\mathrm{traj}}$
are defined formally in Box~\ref{box:features}; their geometric
motivation is in \S\ref{sec:framework}. The reference distribution
for the steering operators is the lineage pre-training checkpoint
(App.~\ref{app:ancestor-choice} shows ancestor identity is
interchangeable within reasonable capability range under a
well-calibrated detector).


\paragraph{Recoverability regimes (used in every subsequent section).}
Running the three features through k-means ($k{=}4$, initialised at
canonical centroids from a cell-level rule classifier; details in
\S\ref{sec:features} and App.~\ref{app:regime-classifier}) partitions
failed problems into \textsf{RM-G} (Rank Misrouting, geo-local),
\textsf{RM-D} (Rank Misrouting, junction-diffuse), \textsf{DD}
(Distributed Deformation), and Unresolved
(Table~\ref{tab:regime-glossary}).
Figs.~\ref{fig:regime-clustering}--\ref{fig:regime-distribution} ground
this partition: the four regimes occupy distinct corners of feature
space (Fig.~\ref{fig:regime-clustering}) and their per-cell composition
already foreshadows the SFT-vs-RL contrast we unpack in
\S\ref{sec:population} (Fig.~\ref{fig:regime-distribution}).

\begin{figure*}[!t]
  \centering
  \includegraphics[width=\textwidth]{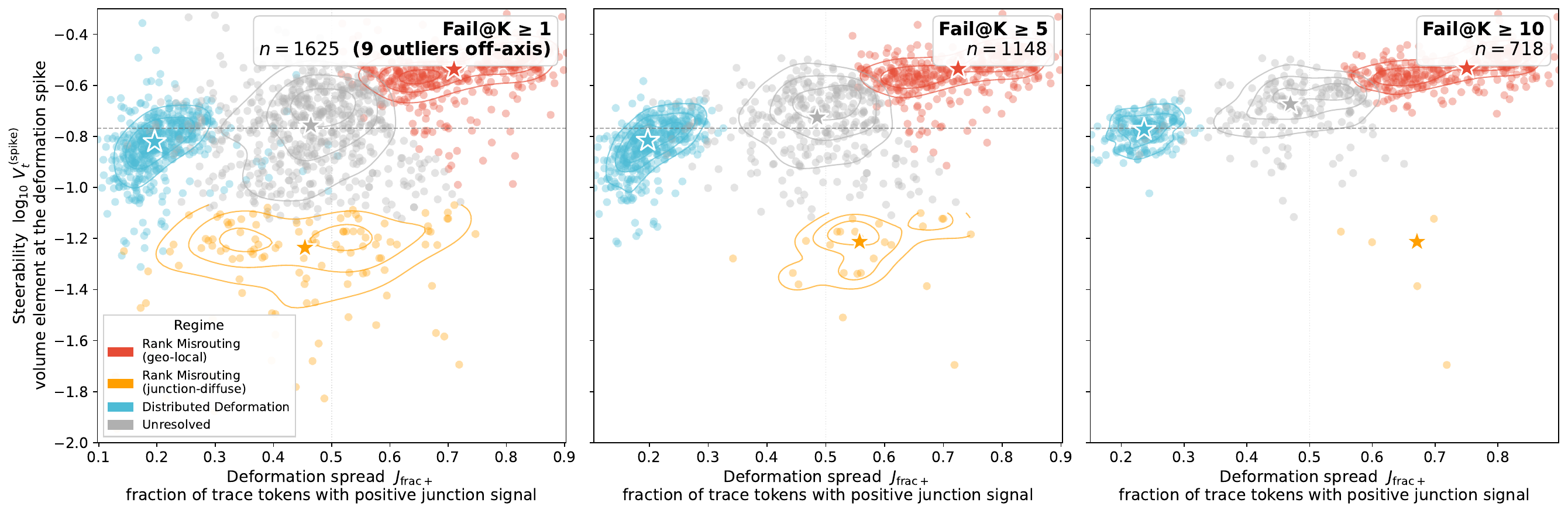}
  \caption{\textbf{Problem-level regimes in (deformation spread,
    junction steerability) space.} Axes are
    $\bar J_{\mathrm{frac+}}$ and $\log_{10}\!\bar V_{t^\star}$
    (the classifier features). Each point is one problem-unit;
    color encodes the nearest-centroid regime. The dashed horizontal
    line marks the median of $\log_{10}\!\bar V_{\mathrm{traj}}$
    (the dispatch H/L routing key, a different stable feature; see
    Box~\ref{box:features}). Star markers are per-regime medians;
    contours are KDE iso-density. As Fail@$K$ deepens, shallow
    sampling accidents are filtered out and the feature clusters
    tighten. The
    $(\bar J_{\mathrm{frac+}}, \log_{10}\!\bar C)$ projection appears
    as Figure~\ref{fig:regime-clustering-appendix} (Appendix).}
  \label{fig:regime-clustering}
\end{figure*}

\begin{figure}[!t]
  \centering
  \includegraphics[width=\linewidth]{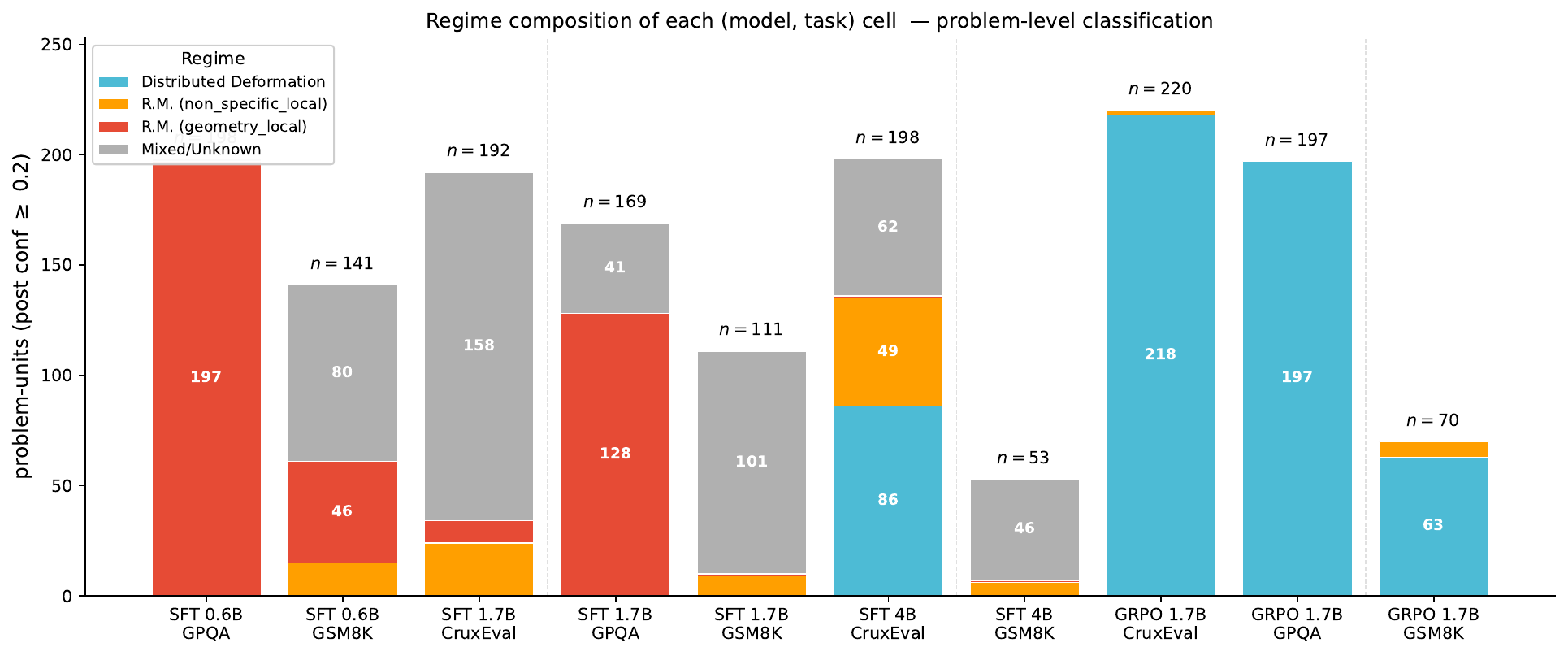}
  \caption{\textbf{Regime composition of every (model, task) cell.}
    Stacked bars; bar height $=$ number of problem-units in the cell.
    GRPO 1.7B collapses into Distributed Deformation across all three
    tasks, consistent with reported RL-driven support compression
    \cite{chu2025sftmemorizes,zhu2025rlvrprincipals,shenfeld2025rlsrazor}.
    SFT spreads across the remaining three regimes by task identity.
    Audit-channel analysis in \S\ref{sec:population}.}
  \label{fig:regime-distribution}
\end{figure}

\paragraph{Steerable-Hard problem-units (the routable target population).}
\label{sec:steerable-hard}
The pids on which any test-time router can show measurable gain are
those for which retry alone is inadequate \emph{and} at least one
intervention beats retry. Concretely, a problem-unit is
\textbf{Steerable-Hard} when, at the per-pid level,
\begin{equation}
\label{eq:steerable-hard}
\resizebox{\columnwidth}{!}{$
\underbrace{\max_{o \in \{\text{SL-G},\,\text{SL-R},\,T_{\mathrm{loc}},\,\text{DL}\}} r_o
   \;-\; r_{\text{retry}}}_{\text{steerable: }\geq\,5\,\text{pp}}
\quad\text{AND}\quad
\underbrace{1 - r_{\text{retry}}}_{\text{hard: }\geq\,0.5}
$}
\end{equation}
where $r_o$ is the rescue rate of operator $o$ at repair@3.\footnote{Operator
rescue rates in the Steerable-Hard definition and the prospective routing
results (\S\ref{sec:prospective}) are mean per-attempt success probabilities,
averaged over the available background-conditioning depths; the dispatch
(\S\ref{sec:routing-result}) and causal ladder report any-of-$k$ rescue
instead.}
$37.1\%$ ($528/1{,}423$) of our failed problem-units satisfy both
conditions; outside this subset routing is either unnecessary
(retry alone rescues, $43\%$) or futile (no operator rescues, $20\%$).
This is the subset that the framework of \S\ref{sec:framework}
identifies as the operator class's target population: where
Prop.~\ref{thm:rank-preservation} forbids retry recovery (a
confidently-wrong specialist preference cannot be inverted by
rank-preserving operators) and at least one $e$-geodesic intervention
reaches its $\alpha^\star$ threshold (Prop.~\ref{thm:rank-inversion})
within our $\alpha$ grid.
All routing claims below report numbers both on the full pool and
on this subset, because gains on \emph{Steerable-Hard} are the
gains that translate to deployment value.

\paragraph{Three demonstrations of one framework.}
The three trajectory features are the central object of the paper:
they are derived from operator-class geometry, i.e.\ the structure of available interventions and what they can change about a failed trace (\S\ref{sec:framework})
and we demonstrate that they encode operator-relevant structure
through three downstream uses. The features cluster failures into
stable recoverability regimes (\S\ref{sec:features}); the regime
distribution of a (model, task) cell audits its post-training
method (\S\ref{sec:population}); and the same features support a
training-free operator-class routing rule
(\S\ref{sec:routing-result}--\ref{sec:prospective}). The unifying
claim is that failed traces contain operator-class signature,
recoverable from three problem-level features, and that signature
is deployment-actionable along multiple axes.


\section{Routing test-time compute by operator class}
\label{sec:routing-result}

A cell-aware dispatch over the five test-time operators tests the framework's prediction that regime-conditional operator choice should beat retry (rank-preserving; Prop.~\ref{thm:rank-preservation}) and any uniform single operator.

\textbf{The dispatch headline.} Using the regime label and the
steerability split $\bar V_{\mathrm{HL}}$ from \S\ref{sec:setup},
each pid's cell routes to the operator with the largest
cell-aggregate $S_i = U_i\cdot R_i^{0.5}$. Dispatch sends all failed
problem-units to a non-retry operator ($81.8\%$ to sparse
logit-steering, $18.2\%$ to $T_{\mathrm{loc}}{=}1.5$), beating retry
by $\Delta\Sigma S{=}+35.7$ ($\Sigma S{=}693.4$ vs.\ $657.7$).
$\Sigma S = U \cdot R^{0.5}$ is a compute-aware aggregate; the
deployment-realistic per-pid pass-rate lift on the Steerable-Hard
subset is shown in Fig.~\ref{fig:flops-pareto}b. The
strongest single operator (uniform SL-R) reaches $\Sigma S{=}678.7$,
so no operator wins by more than $\sim21$ $\Sigma S$; the lift comes
from \emph{cell-level} operator preferences. The key decision is binary, whether to
inject at all, and within the sparse class the choice
(SL-G vs.\ SL-R) is empirically fungible
(\S\ref{subsec:bucket-analysis}). The trace-level geometry to
inference-time operator-class connection is what prior weight-space
\cite{zhu2025rlvrprincipals} and trace-level \cite{wei2023doa}
accounts do not provide.

\textbf{The feature space predicts where retry fails.} Retry rescues only $39.6\%$ of Rank Misrouting (geo-local) problems but $65.1\%$ of Unresolved problems ($\Delta{=}-25.5$ pp, $p{<}0.0001$ unpaired bootstrap; full pairwise contrasts in App.~\ref{app:retry-by-regime}), so the features predict where additional sampling will pay off before any intervention is run. Cases A--C (Figure~\ref{fig:routing-case-study}, App.~\ref{app:qualitative-cases}) illustrate the binary decision on prototypical traces.

\begin{table*}[!t]
\centering
\small
\setlength{\tabcolsep}{4pt}%
\resizebox{\textwidth}{!}{%
\begin{tabular}{l c r l r r r r r r}
\toprule
 & & & & \multicolumn{3}{c}{Rescue rate $R$} & \multicolumn{3}{c}{Dispatch score $S = U R^{0.5}$} \\
\cmidrule(lr){5-7}\cmidrule(lr){8-10}
Regime & $V_t$ & $n$ & Op & Dispatch & retry & SL-R & Dispatch & $\Delta$ vs retry & $\Delta$ vs SL-R \\
\midrule
  DD & H & 430 & \textbf{SL-R} & 0.533 & 0.521 & 0.533 & \textbf{167.2} & +5.5 & +0.0 \\
  DD & L & 93 & \textbf{SL-G} & 0.968 & 0.914 & 0.957 & \textbf{88.5} & +7.3 & +1.5 \\
  RM-G & L & 422 & \textbf{SL-R} & 0.422 & 0.396 & 0.422 & \textbf{115.6} & +10.5 & +0.0 \\
  Unresolved & H & 184 & \textbf{SL-R} & 0.707 & 0.679 & 0.707 & \textbf{109.3} & +6.3 & +0.0 \\
  Unresolved & L & 280 & $T_{\mathrm{loc}}{=}1.5$ & 0.636 & 0.632 & 0.621 & \textbf{141.9} & +1.2 & +4.8 \\
  RM-D & H & 38 & \textbf{SL-G} & 0.658 & 0.579 & 0.579 & \textbf{20.3} & +3.5 & +3.5 \\
  RM-D & L & 89 & \textbf{SL-G} & 0.685 & 0.674 & 0.640 & \textbf{50.5} & +1.2 & +4.9 \\
\midrule
\multicolumn{3}{l}{$\Sigma$ across 10 cells} & — & — & — & — & \textbf{693.4} & +35.7 & +14.7 \\
\bottomrule
\end{tabular}}
\caption{\textbf{Regime-aware dispatch by (regime, $V_t$ H/L). }Each row is one dispatch cell; \emph{Op} is the per-cell dispatch choice (operator with the highest cell-aggregate $S = U \cdot R^{0.5}$). Three rescue rates per cell: the dispatch's R, retry R (every cell does retry), and uniform-SL-R R (every cell does SL-R, the globally strongest single operator). $\Delta$ vs retry and $\Delta$ vs SL-R are dispatch $S$ minus each baseline. Rows in \textcolor{gray}{gray} (if any) are small-$n$ cells ($n{<}30$) where the cell-aggregate score is dominated by sampling noise; included for completeness but not load-bearing for headline comparisons. The dispatch totals $U{=}891$, $\Sigma S{=}693.4$ vs.\ $U{=}860$, $\Sigma S{=}657.7$ for retry ($\Delta\Sigma S{=}+35.7$) and $U{=}879$, $\Sigma S{=}678.7$ for uniform SL-R ($\Delta\Sigma S{=}+14.7$). $V_t$ H/L = median split on $\log_{10}\bar V_{\mathrm{traj}}$ (per-pid trajectory-mean Fisher info, pooled ICC $0.88$).}
\label{tab:dispatch}
\end{table*}

\textbf{Steerability is the routing mechanism.} Every regime flips
its winning operator across the $\bar V_{\mathrm{HL}}$ boundary
(Table~\ref{tab:dispatch}); six of ten cells route by steerability
side. The flips are class-level (sparse vs.\ retry vs.\
$T_{\mathrm{loc}}$), not within the sparse class. Steerability
supplies the operator-class selection work the regime label alone
cannot do. Per-pid winner \emph{prediction} (the exact operator that wins for a
specific pid) is not recoverable above baseline
(App.~\ref{app:stability}, Table~\ref{tab:cell-vs-pid}),
the dispatch's lift is a population-scale claim. The weaker per-pid
\emph{routing-decision} signal (which broad action class is
appropriate) does carry actionable information on the
Steerable-Hard subset; we return to a training-free per-pid variant
in \S\ref{sec:prospective}.

\textbf{Per-FLOP efficiency: dispatch is Pareto-dominant against
scaled retry.} On the two cells where $K{=}64$ retry pools exist,
dispatch@3 is Pareto-dominant against scaled retry at
${\sim}1.5{\times}$ iso-rescue compute (Fig.~\ref{fig:flops-pareto};
full per-cell curves in App.~\ref{app:retry-by-regime}).

\phantomsection\label{subsec:bucket-analysis}%
\textbf{Operators are mechanistically distinct; the binary inject-vs-retry call dominates within-class choice.}
On Steerable-Hard, pairwise Jaccard of rescued-pid sets is $0.10$--$0.59$ across families; the all-three-operator intersection is at most $24\%$ while the any-operator union covers $100\%$, so each operator pulls a partially-different pid subset and operator-class routing extracts lift no uniform single operator can match. Within the sparse class, SL-G vs.\ SL-R is family-dependent in direction but small in magnitude (Qwen3 favors SL-R by $+6.9$ pp; R1-Distill favors SL-G by $+11.4$ pp; Phi-4 tied); the key call is inject vs.\ retry, not position within the sparse class (App.~\ref{app:bucket-analysis}, Table~\ref{tab:bucket-analysis}).

\section{Prospective test-time routing policy}
\label{sec:prospective}

\S\ref{sec:framework} prescribes a per-problem rule: route to the
operator whose corresponding feature is the dominant geometric
signature for this trace. The dispatch in \S\ref{sec:routing-result}
consumes operator outcomes to define its regime labels and
steerability split; we now ask the complementary, prospectively
actionable question: can the same three features \emph{alone},
computed from a
cell's failed rollouts with no operator assay on the test cell,
instantiate that rule?

\paragraph{The rule.} For each failed pid the rule
z-normalises the three classifier features and applies the
\emph{paper-derived} mapping
\begin{equation}
\label{eq:prospective-rule}
\resizebox{\columnwidth}{!}{$
\hat o(p) =
\begin{cases}
\text{DL}    & \text{if } \arg\max(\hat z_{\!J_+},\,\hat z_{\!\log C},\,\hat z_{\!\log V_{t^\star}}) = \hat z_{\!J_+}, \\
\text{SL-G}  & \text{if } \arg\max(\cdot) = \hat z_{\!\log C}, \\
T_{\mathrm{loc}}{=}1.5 & \text{if } \arg\max(\cdot) = \hat z_{\!\log V_{t^\star}},
\end{cases}
$}
\end{equation}
mirroring \S\ref{sec:features}'s regime semantics (broad
junction signal $\to$ distributed deformation $\to$ DL;
sharp single junction $\to$ geometry-local $\to$ SL-G;
high entropy headroom at the spike $\to$ entropy-brittle
$\to$ local temperature lift). No gate is applied, every failed
pid is routed to the argmax operator. Per-pid cost is dominated by
the routing operator's single attempt (mean ${\approx}1.4$
single-rollout units across the chosen operator mix); feature
extraction is O($N_{\text{tokens}}$) arithmetic on top-$k$ logprobs
that vLLM can save during the original spec rollout, and is
negligible relative to the intervention itself, so the policy
runs at \emph{comparable retry cost}.

\begin{table*}[t]
\centering
\footnotesize
\caption{\textbf{Prospective test-time routing policy compared to baselines and oracles.} \textbf{Feature-only} is our training-free rule (\S\ref{sec:prospective}, Eq.~\ref{eq:prospective-rule}); it consumes only failed-rollout features at test time. \textbf{Compute} is measured in single-rollout-equivalent units (one retry attempt $=1.0$) and reflects only the policy's deployment-time generation cost; feature extraction from saved logprobs is $O(N_{\text{tokens}})$ arithmetic and negligible relative to a rollout. \emph{All failed pids}: pool-weighted across the 10-cell evaluation set ($n{=}1{,}423$). \emph{Steerable-Hard subset}: the $37.1\%$ of failed pids where retry fails the majority of the time AND at least one intervention beats retry by $\geq 5$ pp (Eq.~\ref{eq:steerable-hard}). The dispatch in \S\ref{sec:routing-result} consumes operator outcomes to define regimes; the prospective policy does not.}
\label{tab:prospective-routing}
\resizebox{\textwidth}{!}{%
\begin{tabular}{l l c r rr rr}
\toprule
\multicolumn{2}{l}{\textbf{Policy}} & \textbf{Needs} & \textbf{Compute} & \multicolumn{2}{c}{\textbf{All failed pids}} & \multicolumn{2}{c}{\textbf{Steerable-Hard subset}} \\
\cmidrule(lr){5-6}\cmidrule(lr){7-8}
\textbf{Name} & \textbf{Routing info uses\dots} & \textbf{labels?} & \textbf{$\times$ retry} & rescue & $\Delta$retry & rescue & $\Delta$retry \\
\midrule
Retry & no information, no test-cell compute beyond resampling & no & 1.0 & 0.372 & +0.000 & 0.163 & +0.000 \\
\textbf{Feature-only} & winner-take-all over 3 trajectory features; no training & no & 1.4 & 0.373 & +0.000 & 0.285 & +0.122 \\
Outcome-trained & RF classifier on the cell's own per-pid operator outcomes (assay cost = DiagnosticOracle's) & yes & 66.5 & 0.374 & +0.001 & 0.302 & +0.139 \\
DiagnosticOracle & runs every operator on the test cell, picks modal-best & yes & 66.5 & 0.390 & +0.017 & 0.321 & +0.157 \\
ProblemOracle & per problem-unit $\arg\max$ over operators (not deployable) & yes & 66.5 & 0.514 & +0.142 & 0.468 & +0.304 \\
\bottomrule
\end{tabular}%
}
\\[2pt]
\footnotesize ``Needs labels?'' = the policy requires per-pid operator outcomes from the test cell (i.e., running every operator on every pid; $\sim$5 operators $\times$ $K{=}10$ attempts).
\end{table*}

\paragraph{Result (Table~\ref{tab:prospective-routing}).}
On the full $n{=}1{,}423$ pool, the prospective policy ties retry:
$\Delta = +0.01$ pp, 95\% CI $[{-}1.22, {+}1.24]$ (paired bootstrap,
not distinguishable from zero). On the \emph{Steerable-Hard} subset
(\S\ref{sec:steerable-hard}; $37.1\%$ of pids), the rule lifts rescue
by ${+}12.2$ pp (from $0.163$ to $0.285$; 95\% CI $[{+}10.0, {+}14.4]$,
$n{=}528$, $p < 10^{-3}$). The matched $-7.2$ pp drop on non-SH
($63\%$; 95\% CI $[{-}8.5, {-}5.9]$) is the cancellation that
preserves global-mean accuracy; this is the expected behavior of a
subset-conditional method. Zero pids are routed to retry; the
fallback for low-confidence pids is SL-R (\S\ref{sec:framework}).
For context, \textbf{DiagnosticOracle} (the full per-cell operator
assay; Table~\ref{tab:prospective-routing}) reaches $+15.7$ pp on
Steerable-Hard at ${\sim}\,47{\times}$ the per-pid compute, and
\textbf{ProblemOracle} (per problem-unit $\arg\max$, not deployable)
reaches $+30.4$ pp. A fully \textbf{Outcome-trained} RF classifier
fit on each cell's own per-pid operator outcomes reaches only
$+13.9$ pp at the same assay cost as DiagnosticOracle
($66.5{\times}$ retry), so training on operator labels adds little
over the training-free rule while still requiring them.
\textbf{Feature-only} therefore recovers $78\%$
of DiagnosticOracle's Steerable-Hard lift at ${\sim}1/47$ of its
cost, using only features whose computation is negligible relative
to a rollout.

\textbf{The framework predicts the routing-relevant quantity, not a
specific feature set.}
Multiple Fisher-aggregate feature sets achieve comparable SH lift
($+11.7$ to $+13.8$ pp at near-zero global cost; App.~\ref{app:naive-ablation}),
confirming the framework's prediction that local Fisher information
is the routing-relevant quantity. The geometric instance reported in
the body is one interpretable derivation within that family. Substituting SL-R for SL-G at the
$\log_{10}\bar C$ branch loses only $-0.4$ pp on the Qwen3 SH headline,
so a deployer without a junction detector can substitute SL-R and
retain $>96\%$ of the lift; the dispatch's binary decision is
operator-class (sparse vs.\ temperature vs.\ dense), not position
within the sparse class (\S\ref{subsec:bucket-analysis},
App.~\ref{app:bucket-analysis}).

\paragraph{What this validates.}
\S\ref{sec:routing-result} showed that the feature space supports
population-level operator routing; \S\ref{sec:prospective} shows the
complementary, \emph{prospectively actionable} variant that is usable
at test time with no operator assay on the test cell. This converts the dispatch's regime structure from
a diagnostic into a deployment artifact: a fixed rule table that
takes failed rollouts in and emits an operator choice. We discuss
remaining limits (per-pid winner identifiability, cross-task
transfer of the rule table) in \S\ref{sec:limitations}.

\section{Features that characterize the regimes}
\label{sec:features}

\phantomsection\label{subsec:regimes}%
\phantomsection\label{subsec:diagnostic-topography}%
The three features partition failed problems into four regimes via
k-means initialised at canonical centroids
(App.~\ref{app:regime-classifier}). Two properties of the resulting
partition matter for the dispatch: (i) $99.6\%$ of \textsf{RM-G} sits
above the $\log_{10}\bar V$ median while $85\%$ of \textsf{DD} sits
below, the two regimes the dispatch differentiates most sharply
occupy opposite ends of the steerability axis; (ii) cluster tightness improves monotonically with Fail@$K$ depth: silhouette $0.45\!\to\!0.58$ and median within-cluster distance $0.68\!\to\!0.48$ ($-30\%$) between Fail@$K{\geq}1$ and Fail@$K{\geq}10$ (matched-$n$ bootstrap, $n{=}718$, $95\%$ CIs disjoint at every stratum; Fig.~\ref{fig:regime-clustering}), so the harder the failure under standard sampling, the more informative the failed trace.


\section{Failure topography as an audit channel for post-training}
\label{sec:population}

Different post-training methods produce different recoverability
topographies. This section shows the topography is recoverable from
failed traces alone, providing an audit channel for post-training
without training-time access. Without any training-time access and
without any label on the rollouts saying \textit{this is from GRPO}
or \textit{this is from SFT}, the regime distribution of a (model,
task) cell identifies which post-training method produced the model,
recovering the SFT-vs-RL contrast
\cite{chu2025sftmemorizes,zhu2025rlvrprincipals,shenfeld2025rlsrazor}
from failed rollouts alone. The same features drive the routing
example in \S\ref{sec:routing-result}.


\textbf{GRPO concentrates $97.4\%$ of problem-units} ($487/500$) in
Distributed Deformation across CruxEval, GPQA, and GSM8K. \S\ref{sec:framework}
predicts this: the diffuse, off-principal RLVR weight updates
characterized at the weight- \cite{zhu2025rlvrprincipals} and
aggregate-output \cite{chu2025sftmemorizes,shenfeld2025rlsrazor}
levels manifest in the trace as high $\bar J_{\mathrm{frac+}}$ and
low $\bar C$, exactly the Distributed Deformation signature,
recovered here from failed rollouts alone with no training-time access.
\textbf{SFT distributes by task identity, not scale}: $79.7\%$ of
Rank Misrouting (geo-local) problem-units ($341/428$) come from SFT
0.6B and 1.7B on GPQA; $46.5\%$ of Unresolved ($266/572$) comes from
SFT 0.6B and 1.7B on GSM8K and SFT 1.7B on CruxEval. Detailed
per-task percentages are in Appendix~\ref{app:cruxeval-mixing};
the operator-relevance of these regime distinctions is the
main claim.

\textbf{Mechanism (weight-space $\leftrightarrow$ inference-time).}
Spectral preservation (top-$k$ singular structure intact;
\citealt{zhu2025rlvrprincipals}) and large diffuse logit deviations
are not in tension: off-principal updates can accumulate to produce
meaningful inference-time behavioural change with minimal spectral
drift. Our audit channel measures the inference-time
manifestation of that micro-direction movement, on a different
model family and training set than the weight-space evidence.

\textbf{The audit channel is predictive, not just descriptive.}
A $k{=}1$ nearest-neighbour classifier on the 4-D cluster-fraction
histograms of each (model, task) cell labels \textsf{SFT} vs.\
\textsf{RL} on $14$ cells with $84.3{\pm}4.3\%$ accuracy across $10$
random 5-fold permutations ($+20$ pp over a $64.3\%$ majority-class
baseline; LOCO equivalent is $85.7\%$, 12/14). All three GRPO cells,
all seven Qwen3 SFT cells, and both Phi-4-mini-reasoning cells are
classified correctly; the two errors are both R1-Distill-Qwen-Math-1.5B
cells classified as SFT, consistent with distillation muting the RL
signature in the failed-trace topography.

This connects test-time diagnosis to post-training characterisation:
failures from different post-training regimes carry distinguishable
signatures in feature space, and those signatures are what the
dispatch consumes. Prior work has reached related conclusions from
the training side; we approach the same divide from the failure
trace.

\textbf{Cross-family transfer.} Two held-out-family
probes give the contrasting signatures the framework predicts:
\textbf{R1-Distill-Qwen-Math-1.5B} has SL-G/SL-R/T-loc Jaccard
$0.10$--$0.16$ on SH (far below Qwen3's $0.28$--$0.50$) and $82.9\%$
[76.9, 86.5] of its failed pids in the Distributed Deformation
cluster, while \textbf{Phi-4-mini-reasoning} has Jaccard $0.30$--$0.59$
within the Qwen3 SFT/GRPO range. The features isolate R1-Distill as
the support-collapsed family and group Phi-4 with SFT-shaped
topography, from failed rollouts alone, despite the underlying
architecture differences. The training-free routing rule from
\S\ref{sec:prospective} also transfers across families: $8$ of $9$
cross-family $z$-scaler transfers between the three families fall
within the target's bootstrap 95\% CI of its self-baseline lift
(App.~\ref{app:cross-family}, Tab.~\ref{tab:feature-transfer}).
Full numbers in App.~\ref{app:cross-family-audit}.

Three qualitative case studies illustrating per-failure routing
under the $V_{\mathrm{traj}}$ dispatch (Cases A, B, and C) and a
sweep of response profiles across representative (model, task)
cells appear in App.~\ref{app:qualitative-cases}.

\section{Related Work}
\label{related-work}

\paragraph{Test-time compute scaling.} Best-of-$N$
\citep{stiennon2020summarize,cobbe2021verifiers}, self-consistency
\citep{wang2023selfconsistency}, and recent test-time scaling
\citep{snell2024scaling,deepseekai2025r1,muennighoff2025s1} allocate
more inference compute by drawing additional trajectories from the
fixed post-trained distribution. We instead ask what an
already-failed trace implies about the next intervention, routing
compute by operator class rather than spending it on more
identically-distributed samples.

\paragraph{Post-training contrast at weight, behavioral, and KL-proximity levels.}
\citet{chu2025sftmemorizes} contrast SFT and RL with outcome rewards
behaviorally (SFT memorizes, RL generalizes);
\citet{zhu2025rlvrprincipals} characterize the weight-space mechanism
(RLVR updates land in low-curvature, off-principal subspaces while
preserving top-$k$ spectral structure); \citet{shenfeld2025rlsrazor}
show online RL implicitly minimizes KL to the base model. All three
require training-time access. Our audit channel
(\S\ref{sec:population}) provides the inference-time logit-space
counterpart that connects these three perspectives to failure-trace
geometry, with Distributed Deformation as the predicted inference-time
signature of Zhu et al.'s off-principal mechanism.

\paragraph{Trace-level interventions.} DoLa \citep{wei2023doa} and
proxy tuning \citep{liu2024proxy} intervene in logit space without a class-localizing
diagnostic. Engineering differences: \textbf{(i)} we contrast against
a separate pre-trained checkpoint (one additional forward pass)
rather than an earlier layer of the same model or a smaller
fine-tuned model; ancestor identity is interchangeable within
reasonable capability range under a well-calibrated detector
(App.~\ref{app:ancestor-choice}); \textbf{(ii)} regime-conditional
operator selection vs.\ uniform application; \textbf{(iii)}
diagnostic-first, with operators deliberately simple. We report only a
preliminary comparison with proxy tuning (App.~\ref{app:proxy-tuning}),
which uses a different specialist-ancestor combination rather than
$e$-geodesic interpolation; a full head-to-head against trace-level
baselines remains open (\S\ref{sec:limitations}).

\paragraph{Inference-time composition and rank-preserving decoding.}
MIXIE \citep{sanyal2025mixie} uses logit-space steering for
multi-objective control; we treat the ancestor as a
\emph{structural probe} for rank-inversion in a single specialized
model. \citet{mattei2025welltempered} characterize aspects of temperature scaling that motivate the temperature probe (App.~\ref{app:moved-methods}); RLHF
calibration degradation \citep{xie2024ats} we reframe as a diagnostic
rank-order collapse.

\paragraph{Forgetting and post-training degradation.} Self-distillation
\citep{zhang2025simple,shenfeld2025self} and temporal sampling
\citep{li2025temporal} address reasoning decay behaviorally; we
localize the same decay at the distributional level and identify which
intervention recovers which regime. The fork/lock typology of
\citet{zhang2025simple} maps onto our logit-steering lift: their
training-time fix reshapes support, our training-free SL-G/SL-R/DL
operators recover a junction-conditional analogue via ancestor
contrast, sidestepping their decode-only impossibility.

\section{Conclusion}
\label{sec:conclusion}
Failed traces are a measurement instrument. Accuracy reports only pass or fail; the three trajectory features report why, sorting
  failures into operator classes that come from theory, not from fitting the data. Three consequences follow. (1) Failure
  populations are heterogeneous in a structured way: the four recoverability regimes stay stable across Fail@$K$ depths. (2)
  Post-training methods leave inference-time signatures, and the audit channel reads them off failed rollouts alone, recovering
  weight-space accounts without weight-space access. (3) Those signatures are deployment-actionable, as the training-free routing
  rule shows; the routing result is a demonstration, not the headline. More broadly, these results suggest that inference-time trajectory statistics are an underused axis for understanding post-training degradation.

\section{Limitations}
\label{sec:limitations}

\textbf{Scope of the routing example.}
The routing rule of \S\ref{sec:prospective} requires three
conditions: (i) a post-trained specialist with an available reference
distribution that disagrees on failure trajectories, (ii) a bounded
test-time intervention budget, and (iii) failures with localizable
geometric structure rather than uniform high entropy. Open-ended generation tasks
(translation, dialogue, creative writing) do not satisfy (iii),
because logit variance is high everywhere when the task is genuinely
under-constrained, and are out of scope for the current framework.

\textbf{The DL branch is dominated empirically.} The framework of
\S\ref{sec:framework} predicts DL (dense logit steering) should win
the Distributed Deformation regime, which is the dominant regime on
GRPO ($97.4\%$ of GRPO problem-units). In our case the
dispatch never picks DL: sparse interventions (SL-G, SL-R) reach
comparable rescue at lower per-attempt cost, and the dispatch's
$\arg\max S_i$ selects them instead. This is a gap between the
geometric prediction (DL should be the right operator for diffuse
deformation) and the per-FLOP recipe at the current budget;
resolving it requires deep-K dispatch evaluation
(repair@$K$ at $K{=}10, 32, 64$ per operator) we leave to future work.

\textbf{Model-family and scale scope.} Headline cells use Qwen3 SFT
(0.6B/1.7B/4B) and GRPO 1.7B; cross-family generalization is
covered by the R1-Distill-Qwen-Math-1.5B and Phi-4-mini-reasoning
probes in \S\ref{sec:population} and
App.~\ref{app:cross-family}. Behavior at $\geq 7$B and with heavily
RLHF-aligned models is out of scope. Verbal self-correction
(Self-Debug) is partially orthogonal to distributional intervention
and is handled separately (Appendix~\ref{app:selfdebug}).

\appendix


\section{Formal Method Details}
\label{app:moved-methods}

\subsection{Junction detection}
\label{app:junction-detection}

\paragraph{High level.}
Both junction detectors used in this paper target positions of
anomalously high specialist-ancestor disagreement, and both calibrate
\emph{within} the problem rather than against an external population,
this is what makes the firing rule task- and scale-independent.
They differ in what they normalize against and at what granularity:
Detector A standardizes against positions \emph{of the current trace}
and fires on a single token; Detector B standardizes against positions
across \emph{this problem's $k$ failed rollouts} and fires on a
sliding window. Detector A is what the
lineage-specificity ablation in App.~\ref{app:ancestor-choice} uses
(for protocol parity with the original ladder); Detector B is what
the dispatch in \S\ref{sec:routing-result} uses, because the
windowed Fisher-information normalization of Detector B is what
makes $V_{t^\star}$ the natural local steerability measure
(\S\ref{sec:features}). The rest of this appendix gives the formal
definitions.

\paragraph{Detector A: trace-relative percentile.}
Given a reasoning trace $x_{1:T}$ sampled from the specialist and the
corresponding specialist and ancestor logit sequences $z_S, z_A$, the
\textbf{demotion score} at step $t$ is
\begin{equation}
\label{eq:demotion}
    D_t = \log P_S(x_t) - \log P_A(x_t),
\end{equation}
the specialist's preference for the chosen token relative to the
ancestor's baseline. To strip out trace-level offsets we standardize
within the trace, $\tilde D_t = (D_t - \mu_D)/\sigma_D$ with $\mu_D,
\sigma_D$ the empirical mean and standard deviation across positions.
An \textbf{intervention-sensitive junction} fires at the first
position whose $\tilde D_t$ exceeds the $90$th percentile of values in
the same trace:
\begin{equation}
\label{eq:junction}
    t^* = \min\,\{t \mid \tilde D_t \ge \tilde D_{(90)}\}.
\end{equation}

\paragraph{Detector B: per-problem windowed quantile.}
Detector B sums a per-token
``junction intensity'' over a sliding window of width $w$:
$S_t = \sum_w J_t$, where
$J_t = \max(Z^{\mathrm{path}}_t, Z^{\mathrm{cov}}_t)$ combines two
normalized deformation signals.
$Z^{\mathrm{path}}_t = \max(0,(r_{\mathrm{chosen}} - \mu_r)/\sqrt{I_r})$
is the e-geodesic log-ratio at the chosen token, centered and scaled
\emph{per-token} by the expectation $\mu_r$ and Fisher information
$I_r$ of $r$ under the specialist's local top-$k$ distribution
(Fisher information is the natural metric on the categorical output
space at that position).
$Z^{\mathrm{cov}}_t = \max(0,(G^{\mathrm{cov}}_t - \mu_{\mathrm{cov}})/
\sigma_{\mathrm{cov}})$ is the m-geodesic coverage gap, centered and
scaled \emph{per-problem} by the mean and standard deviation of
$G^{\mathrm{cov}}$ across the pid's $k$ failed rollouts.
Detector B fires the first time the windowed sum exceeds a
per-problem threshold $\lambda_J$, set as the $0.99$ quantile of
$S_t$ across all positions in the same pid's pool. The steering
direction $\tau_t \in \{\mathrm{path},\mathrm{cov}\}$ is set by
whichever of $Z^{\mathrm{path}}_t, Z^{\mathrm{cov}}_t$ is larger at
the firing position.

\subsection{Dual Generation Strategies: Mechanism vs. Utility}

We use two generation strategies, separating mechanistic identification from deployment evaluation.

\paragraph{Mechanistic Verification ($T=0$):} For all experiments establishing the mechanistic link between logit geometry and recovery (Iterative Repair, Junction Ablations, and Random Controls), we use \textbf{Greedy Decoding ($T=0$)} for continuation after the intervention. Greedy completion ensures that the resulting trajectory is a deterministic consequence of the intervention, removing sampling variance as a confounder. This provides a conservative lower bound on recovery and aligns with our theoretical framework ($\tilde{D}_t, \alpha^*$), which is defined by the rank-order of the distribution's argmax. At $T > 0$, the relationship between logit rank and sampled tokens becomes probabilistic, making the interpretation of $\alpha^*$ as a hard crossover threshold conceptually inconsistent.

\paragraph{Deployment Utility Simulation ($T=0.6$):} To measure the practical benefit of the deployment-time routing policy (Result~4), we use \textbf{Temperature Sampling ($T=0.6$)} for both pre- and post-intervention rollouts for dense trajectory logit steering ($T=0.6$ in the table). This strategy simulates real-world usage where specialists are deployed under non-greedy decoding. Reporting the deployment simulation at $T=0$ would compress the gap between conditions, as greedy decoding often hurts baseline specialist performance on reasoning tasks. 

\subsection{Theoretical Justification}
\label{app:theory}

The diagnostic value of the ladder rests on two closed-form results.

\begin{proposition}[Token-Level Rank Preservation]
\label{thm:rank-preservation}
For any two tokens $x_i, x_j \in V$ and any $T > 0$, if the specialist logit $z(x_i) > z(x_j)$, then $p_{1/T}(x_i) > p_{1/T}(x_j)$.
\end{proposition}
\begin{proof}
The log-odds ratio is $\log \frac{p_\beta(x_i)}{p_\beta(x_j)} = \beta(z(x_i) - z(x_j))$. Since $\beta = 1/T > 0$, the sign is invariant to $T$.
\end{proof}

\begin{proposition}[Rank Inversion via SL-G]
\label{thm:rank-inversion}
If the Specialist ranks $x_\text{trap}$ above $x_\text{logic}$ ($\Delta_S > 0$) but the Ancestor ranks $x_\text{logic}$ above $x_\text{trap}$ ($\Delta_A := z_A(x_\text{logic}) - z_A(x_\text{trap}) > 0$), there exists a critical threshold $\alpha^* = \frac{\Delta_A}{\Delta_A + \Delta_S} \in (0,1)$ such that for all $\alpha < \alpha^*$, the mixed model recovers $x_\text{logic}$ as the mode.
\end{proposition}
\begin{proof}
The mixed logit gap is $\Delta_\alpha = \alpha\,\Delta_S - (1-\alpha)\,\Delta_A$. Solving $\Delta_\alpha < 0$ yields $\alpha < \Delta_A / (\Delta_A + \Delta_S) =: \alpha^*$.
\end{proof}

Together, Proposition~\ref{thm:rank-preservation} establishes that temperature cannot invert ranks, and Proposition~\ref{thm:rank-inversion} gives the closed-form condition under which logit steering can.

\subsection{Failure-Trajectory Metrics}


\label{sec:topography-defs}


\paragraph{Path deformation score.}
At each token position $t$ in a failed trace $x_{1:T}$, the \textbf{path deformation score} measures how strongly the specialist prefers the sampled token over the ancestor's baseline:
\begin{equation}
\label{eq:path-deformation}
  D_t^{\text{path}} \;=\; \log p_S^T(x_t) - \log p_A^T(x_t)
\end{equation}
where $p_S^T$ and $p_A^T$ are the specialist and ancestor distributions at sampling temperature $T$. Positive values indicate that the specialist assigns higher probability than the ancestor to this token choice. The trajectory-mean $\bar{D}^{\text{path}} = \frac{1}{T}\sum_t D_t^{\text{path}}$ is stored as column \texttt{Delta\_path\_mean} in the routing feature table; the value at the detected junction position is stored as \texttt{Delta\_path\_at\_that}.

\paragraph{Junction intensity.}
The \textbf{junction intensity} $J_t$ is a Fisher-normalized, one-sided version of $D_t^{\text{path}}$, centering by the KL divergence and normalizing by the Fisher information of the path direction:
\begin{equation}
\label{eq:junction-intensity}
  J_t \;=\; \left[\frac{D_t^{\text{path}} - \mu_t^r}{\sqrt{I_t^r + \varepsilon}}\right]_+
\end{equation}
where $\mu_t^r = \mathrm{KL}(p_S^T \| p_A^T)$ is the expected log-likelihood ratio under $p_S$ (a local KL divergence estimate), $I_t^r = \mathrm{Var}_{y \sim p_S}[r_t(y)]$ is the Fisher information of the $e$-geodesic direction at position $t$, and $[\cdot]_+$ denotes the positive part. Junction intensity fires when the specialist's choice of $x_t$ is anomalously more specialist-favored than the specialist's own average deformation from the ancestor at this position---detecting rank misrouting beyond what background post-training drift would predict.

The trajectory mean and standard deviation are stored as \texttt{J\_approx\_mean} and \texttt{J\_approx\_std}. The \textbf{Junction Concentration} $C_J$ is defined as their ratio:
\begin{equation}
\label{eq:junction-concentration}
  C_J \;=\; \frac{\sigma(J_t)}{\mu(J_t) + \varepsilon}
  \;=\; \frac{\texttt{J\_approx\_std}}{\texttt{J\_approx\_mean}}
\end{equation}
High $C_J$ indicates that deformation is sharply localized at a small number of positions; low $C_J$ indicates diffuse deformation spread across the trajectory. We use $C_J$ as the canonical concentration statistic throughout.

\paragraph{Coverage deformation score.}
The \textbf{coverage deformation score} measures how much mass the specialist has removed from the ancestor's plausible alternative set:
\begin{equation}
\label{eq:coverage-deformation}
  G_t^{\text{cov}} \;=\; \log\!\left[\frac{p_A(A_t^k)}{p_S(A_t^k) + \varepsilon}\right]
\end{equation}
where $A_t^k = \mathrm{TopK}(p_A^T, k{=}20)$ is the ancestor's top-$k$ vocabulary set at position $t$, and $p_M(A_t^k) = \sum_{y \in A_t^k} p_M^T(y)$ is the total probability mass model $M$ assigns to that set. The ratio normalizes by the ancestor's own belief in its plausible set: a diffuse ancestor (low $p_A(A_t^k)$) should not produce a large signal. $G_t^{\text{cov}}$ asks: \emph{relative to how much the ancestor believes in its own alternatives, how much has the specialist collapsed away from them?} This is the $m$-geodesic (coverage) complement to the $e$-geodesic (path) signal in $J_t$. Column names: \texttt{G\_cov\_mean}, \texttt{G\_cov\_at\_that}.

\paragraph{Temperature submodel variance.}
The \textbf{temperature submodel variance} $V_t$ is the Fisher information of the local temperature submodel---the probability-weighted variance of specialist logits over the top-$K$ vocabulary:
\begin{equation}
\label{eq:temp-variance}
\begin{split}
  V_t &\;=\; \mathrm{Var}_{y \sim p_{S,K}}\!\left[z_S(y)\right] \\
      &\;=\; \sum_{y \in \mathrm{TopK}} p_{S,K}(y)\,\bigl(z_S(y) - \mathbb{E}_{p_{S,K}}[z_S]\bigr)^2
\end{split}
\end{equation}
where $p_{S,K}$ is the specialist distribution renormalized over the top-$K$ vocabulary ($K=100$). High $V_t$ indicates that the distribution is entropy-susceptible: small temperature changes produce large mass shifts and the specialist is in an intermediate state where temperature scaling has leverage. Low $V_t$ indicates either a collapsed (very sharp) or diffuse distribution where temperature perturbation cannot invert any rank. $V_t$ is stored as \texttt{V\_mean} or \texttt{logit\_var\_mean}. \textbf{Implementation note:} do not use \texttt{logit\_S.var()} (unweighted variance over the full vocabulary); the probability-weighted top-$K$ formula is required to recover the Fisher information interpretation \citep{mattei2025welltempered}.

\paragraph{Summary: routing feature roles.}
Table~\ref{tab:led-feature-roles} summarizes the intuitive question each feature answers and its operational role.

\begin{table*}[t]
\centering\small
\begin{tabular}{lp{5cm}p{4.5cm}}
\toprule
\textbf{Feature} & \textbf{Intuitive question} & \textbf{Data column} \\
\midrule
$\bar{D}^{\text{path}}$ & Did specialist over-commit on average? & \texttt{Delta\_path\_mean} \\
$C_J$ & Is deformation localized or diffuse? & \texttt{J\_approx\_std / J\_approx\_mean} \\
$\bar{G}^{\text{cov}}$ & Did specialist abandon ancestor's options? & \texttt{G\_cov\_mean} \\
$V_t$ & Is the distribution still temperature-movable? & \texttt{V\_mean} \\
\bottomrule
\end{tabular}
\caption{Trajectory features used as topography dimensions. ``Option loss'' ($\bar{G}^{\text{cov}}$) and ``persuadability'' ($V_t$) are the intuitive names used in figures and tables.}
\label{tab:led-feature-roles}
\end{table*}

In §\ref{subsec:diagnostic-topography} we visualize the regime-level distributions of these features.

\subsection{Problem-Level Trajectory Features (moved from \S2)}
\label{app:features-box}

\begin{figure*}[!t]
\begin{tcolorbox}[colback=blue!2!white,colframe=blue!75!black,
    title=\textbf{Box~\thefigure: Problem-Level Trajectory Features},
    arc=2mm, boxrule=0.5pt, left=2mm, right=2mm, width=\textwidth]
\refstepcounter{figure}\label{box:features}%
\vspace{-1ex}
Each failed rollout $r$ of problem $p$ contributes a per-token
junction intensity $J^{(r)}_t \in \mathbb{R}_{\geq 0}$
(\S\ref{app:junction-detection}). We compress each rollout into three
scalars and aggregate over $\mathcal{R}_p$ by rollout-mean.

\medskip\noindent
\textbf{(i) Deformation spread}, the fraction of trace tokens with
positive junction signal; how broad the failure is across the trace:
\begin{equation}
\label{eq:deformation-spread}
J^{(r)}_{\mathrm{frac+}} = \frac{1}{T_r}\sum_{t=1}^{T_r} \mathbb{I}[J^{(r)}_t > 0],
\qquad
\bar J^{(p)}_{\mathrm{frac+}} = \frac{1}{|\mathcal{R}_p|}\sum_{r\in\mathcal{R}_p} J^{(r)}_{\mathrm{frac+}}.
\end{equation}

\noindent
\textbf{(ii) Junction concentration}, the ratio of peak to mean
junction intensity; whether the signal spikes at one position or
remains diffuse:
\begin{equation}
\label{eq:junction-concentration-box}
C^{(r)} = \frac{\max_t J^{(r)}_t}{\frac{1}{T_r}\sum_t J^{(r)}_t},
\qquad
\bar C^{(p)} = \frac{1}{|\mathcal{R}_p|}\sum_{r\in\mathcal{R}_p} C^{(r)}.
\end{equation}

\noindent
\textbf{(iii) Junction steerability $V_{t^\star}$},
the Fisher information of the specialist's local temperature
submodel at the junction spike $t^\star = \arg\max_{t>0} J^{(r)}_t$,
averaged across the problem's failed rollouts. Used as the
per-problem classifier feature.
\begin{equation}
\label{eq:vt-spike}
V^{(r)}_{t} = \mathrm{Var}_{x \sim p_S(\cdot \mid \mathrm{ctx}_{t})}
  \bigl[\log p_S(x \mid \mathrm{ctx}_{t})\bigr],
\qquad
\bar V_{t^\star}^{(p)} = \frac{1}{|\mathcal{R}_p|}\sum_{r\in\mathcal{R}_p}
   V^{(r)}_{t^\star_r}.
\end{equation}

\medskip\noindent
\textbf{Routing key.} The dispatch H/L split uses the
\emph{trajectory mean} $\bar V_{\mathrm{traj}}$ (average of
$V^{(r)}_t$ over the trace, excluding the boundary window $T_w$),
not $V_{t^\star}$, a deliberate split between the two roles
(App.~\ref{app:stability}).
\begin{equation}
\label{eq:vt-traj}
\bar V_{\mathrm{traj}}^{(p)} = \frac{1}{|\mathcal{R}_p|}\sum_{r\in\mathcal{R}_p}
   \frac{1}{T_r - T_w}\!\!\sum_{t=T_w}^{T_r}\! V^{(r)}_t.
\end{equation}
\end{tcolorbox}
\end{figure*}

\section{Extended Result Details}
\label{app:extended-main-results}


\subsection{Self-Debug Comparison}
\label{app:selfdebug}

Self-Debug operates on a verbal reasoning axis and provides a fifth diagnostic perspective contrasting verbal introspectability against distributional recoverability. In Table~\ref{tab:selfdebug_matched}, we compare Self-Debug against SL-G and standard Retry on strictly matched failure sets (up to 3 iterative rounds, early stopping on first success). The 4B CruxEval Rank Misrouting result is most informative: SL-G rescues 57.1\% of matched problems, whereas Self-Debug reaches only 22.4\% and Retry 28.6\%; verbal backtracking cannot reach a logic error the model cannot name, and extra samples alone do not suffice. On GSM8K, the pattern reverses: at 1.7B, Retry (95.2\%) and Geo (94.4\%) both substantially outperform Self-Debug (86.3\%), consistent with a sample-recoverable regime (now absorbed into Unresolved) where resampling is enough. We include Self-Debug as an external reference point, not as part of the core four-probe profile. The conditioning event differs from Table~\ref{tab:causal-ladder}: rows here restrict to the Self-Debug-attempted subset, which is a different (and typically smaller) problem slice than the Fail@$K$ strata used for the ladder.

\begin{table*}[t]
\centering
\small
\resizebox{\textwidth}{!}{%
\begin{tabular}{lcc cccc cccc}
\toprule
 & & & \multicolumn{4}{c}{Problem-level rescue (\%)} & \multicolumn{4}{c}{Rollout-level rescue (\%)} \\
\cmidrule(lr){4-7}\cmidrule(lr){8-11}
\textbf{Model / Task} & \textbf{$N$} & \textbf{Avg.\ Att.} & SD & SL-G & retry & $\Delta$(SL-G$-$retry) & SD & SL-G & retry & $\Delta$(SL-G$-$retry) \\
\midrule
SFT 0.6B / CruxEval & 195 & 7.81 & 60.0 & \textbf{42.1} & 65.1 & -23.0 & 7.7 & \textbf{4.1} & 7.2 & -3.1 \\
SFT 0.6B / GSM8K & 195 & 7.81 & 91.8 & \textbf{82.1} & 51.8 & +30.3 & 11.8 & \textbf{11.2} & 12.9 & -1.7 \\
SFT 1.7B / CruxEval & 201 & 8.75 & 74.6 & \textbf{75.1} & 59.2 & +15.9 & 8.5 & \textbf{4.9} & 6.6 & -1.7 \\
SFT 1.7B / GSM8K & 124 & 12.81 & 86.3 & \textbf{94.4} & 95.2 & -0.8 & 6.7 & \textbf{3.3} & 1.9 & +1.4 \\
SFT 4B / CruxEval & 49 & 6.37 & 22.4 & \textbf{57.1} & 28.6 & +28.5 & 3.5 & \textbf{5.3} & 7.1 & -1.8 \\
SFT 4B / GSM8K & 56 & 7.34 & 92.9 & \textbf{92.9} & 76.8 & +16.1 & 12.7 & \textbf{9.8} & 8.5 & +1.3 \\
\bottomrule
\end{tabular}%
}
\caption{\textbf{Self-Debug vs.\ SL-G vs.\ retry: matched problem-level and rollout-level rescue rates.} We compare verbal self-correction (\textbf{SD}), sparse logit-steering at the detected junction (\textbf{SL-G}), and standard temperature sampling (\textbf{retry}) on strictly matched failure sets (intersection of pids attempted by both SD and SL-G). All three methods use up to 3 iterative repair rounds with early stopping on first success. $\Delta$(SL-G$-$retry) shows the per-cell lift of sparse logit-steering over re-sampling; SL-G beats retry on CruxEval at every scale and ties at GSM8K (where retry is already strong). Self-Debug uses a verbal explanation step at $T{=}0$ followed by a revision at $T{=}0.6$; SL-G uses greedy $T{=}0$ decoding after a local logit-steering intervention ($\alpha{=}0.7$) at the detected junction; retry re-samples from the already-failed rollout pool at $T{=}0.6$. Matched to the subset of problems on which Self-Debug was attempted; this differs from the Fail@$K$ stratum used in Table~\ref{tab:causal-ladder}.}
\label{tab:selfdebug_matched}
\end{table*}

\subsection{Regime-name glossary (moved from \S2)}
\label{app:regime-glossary}

\begin{table}[h]
\centering
\footnotesize
\setlength{\tabcolsep}{4pt}
\resizebox{\columnwidth}{!}{%
\begin{tabular}{l l l}
\toprule
Canonical name & Short & Mechanism signal \\
\midrule
Rank Misrouting (geo-local)        & RM-G  & sharp spike, steerable     \\
Rank Misrouting (junction-diffuse) & RM-D  & sharp spike, unsteerable   \\
Distributed Deformation            & DD    & broad diffuse deformation  \\
Unresolved                         &,   & no dominant signal         \\
\bottomrule
\end{tabular}%
}
\caption{\textbf{Regime-name glossary.} Canonical name is used in body
prose; the \emph{Short} column appears in compact figure legends and
dispatch rows.}
\label{tab:regime-glossary}
\end{table}

\subsection{Dispatch Score Hyperparameter Sensitivity}
\label{app:dispatch-sensitivity}

The Dispatch Score used to assign regimes is defined as $S_i = U_i \cdot R_i^\gamma$, where $U_i$ represents the unique coverage of probe $i$ relative to cheaper interventions (in the cost order: retry $\rightarrow$ SL-R $\rightarrow$ SL-G $\rightarrow$ DL) and $R_i$ is its absolute reliability. To assign a regime, we take the probe with the maximum $S_i$; if this maximum score is below a threshold $\tau$, the problem is assigned to the Unresolved regime. 

Table~\ref{tab:dispatch_sensitivity} evaluates the stability of these regime assignments across a grid of hyperparameters ($\gamma \in \{0.3, 0.5, 0.7\}$ and $\tau \in \{0.02, 0.03, 0.05\}$). The assignments are remarkably robust; the primary variation is limited to the boundary between weak signal and Unresolved, confirming that the regime taxonomy is driven by the data rather than post-hoc parameter fitting.

\begin{table*}[t!]
\centering
\resizebox{\textwidth}{!}{
\begin{tabular}{llccccccccc}
\toprule
 & & \multicolumn{3}{c}{$\gamma = 0.3$} & \multicolumn{3}{c}{$\gamma = 0.5$} & \multicolumn{3}{c}{$\gamma = 0.7$} \\
\cmidrule(lr){3-5} \cmidrule(lr){6-8} \cmidrule(lr){9-11}
\textbf{Task} & \textbf{Model} & $\tau = 0.02$ & $\tau = 0.03$ & $\tau = 0.05$ & $\tau = 0.02$ & $\tau = 0.03$ & $\tau = 0.05$ & $\tau = 0.02$ & $\tau = 0.03$ & $\tau = 0.05$ \\
\midrule
\multirow{4}{*}{\textit{CruxEval}}
 & SFT 0.6B & RM-D & RM-D & RM-D & RM-D & RM-D & RM-D & RM-D & RM-D & Unresolved \\
 & SFT 1.7B & RM-D & RM-D & RM-D & RM-D & RM-D & RM-D & RM-D & RM-D & RM-D \\
 & SFT 4B & DD & DD & DD & DD & DD & DD & DD & DD & DD \\
 & GRPO 1.7B & RM-D & RM-D & RM-D & RM-D & RM-D & Unresolved & RM-D & RM-D & Unresolved \\
\midrule
\multirow{4}{*}{\textit{GSM8K}}
 & SFT 0.6B & Unresolved & Unresolved & Unresolved & Unresolved & Unresolved & Unresolved & Unresolved & Unresolved & Unresolved \\
 & SFT 1.7B & Unresolved & Unresolved & Unresolved & Unresolved & Unresolved & Unresolved & Unresolved & Unresolved & Unresolved \\
 & SFT 4B & RM-D & RM-D & RM-D & RM-D & RM-D & RM-D & RM-D & RM-D & RM-D \\
 & GRPO 1.7B & RM-G & RM-G & RM-G & RM-G & RM-G & RM-G & RM-G & RM-G & Unresolved \\
\midrule
\multirow{4}{*}{\textit{GPQA}}
 & SFT 0.6B & RM-D & RM-D & RM-D & RM-D & RM-D & Unresolved & RM-D & RM-D & Unresolved \\
 & SFT 1.7B & RM-G & RM-G & RM-G & RM-G & RM-G & RM-G & RM-G & RM-G & RM-G \\
 & SFT 4B & RM-D & RM-D & RM-D & RM-D & RM-D & RM-D & RM-D & RM-D & RM-D \\
 & GRPO 1.7B & RM-D & RM-D & RM-D & RM-D & RM-D & RM-D & RM-D & RM-D & RM-D \\
\bottomrule
\end{tabular}
}
\caption{\textbf{Dispatch Score Hyperparameter Sensitivity.} Regime assignments at the Fail@5 stratum across a grid of Dispatch Score hyperparameters. $\gamma$ controls the discount applied to the reliability term, and $\tau$ is the minimum score threshold to assign a regime rather than defaulting to Unresolved. Abbreviations follow Table~\ref{tab:regime-glossary}: RM-G = Rank Misrouting (geo-local), RM-D = Rank Misrouting (junction-diffuse), DD = Distributed Deformation, Unresolved. The assignments are highly stable; variation primarily occurs at the boundary between weak signal and Unresolved, confirming the taxonomy is not an artifact of post-hoc tuning.}
\label{tab:dispatch_sensitivity}
\end{table*}

\subsection{Cross-Family Replication}
\label{app:crossfamily}

Table~\ref{tab:crossfamily_replication} reports preliminary Llama-3
results. For Llama-3.2-3B-Instruct on CruxEval, three ancestors
were tested: the same-family Llama-3.2-3B base (26.9\%), Llama-3.1-8B
(29.5\%), and R1-Distill-Llama-8B (35.9\%). The two non-same-family
comparators rescue more than the same-family base, but both have
substantially higher effective capacity, so the cross-family signal
is mixed and capacity-confounded under this protocol. We report the
numbers as a single cross-family probe rather than a controlled test;
the within-family Qwen3 results (Appendix~\ref{app:ancestor-choice})
are the cleaner ablation, and they show that ancestor identity is
interchangeable within the capability range we tested under a
well-calibrated detector.

\begin{table*}[htbp]
\centering
\small
\setlength{\tabcolsep}{4pt}%
\resizebox{\textwidth}{!}{%
\begin{tabular}{ll rr rr rr}
\toprule
\textbf{Specialist} & \textbf{Ancestor} & \multicolumn{2}{c}{Fail@1} & \multicolumn{2}{c}{Fail@3} & \multicolumn{2}{c}{Fail@5} \\
\cmidrule(lr){3-8}
& & Rate & Uniq & Rate & Uniq & Rate & Uniq \\
\midrule
  \multirow{3}{*}{Llama-3.2-3B-Instruct}
  &  Lineage (Llama-3.2-3B) & 26.9\% & 2 & 28.1\% & 2 & 26.3\% & 2 \\
  &  Llama-3.1-8B & 29.5\% & 2 & 31.2\% & 2 & 28.1\% & 1 \\
  &  R1-Distill-Llama-8B & 35.9\% & 6 & 35.9\% & 4 & 36.8\% & 4 \\
\bottomrule
\end{tabular}}

\caption{\textbf{Cross-Family Lineage Replication (CruxEval).}
  Rescue rates across different specialist failure strata (Fail@$K$).
  Matched on intersection of problems and rollouts within each $K$ stratum.
}
\label{tab:crossfamily_replication}
\end{table*}

\section{Regime Separation in 2D Trajectory Feature Space}
\label{app:regime-clustering}

Figure~\ref{fig:regime-clustering} plots failed rollouts across three difficulty strata (Fail@$1, 5, 10$) in
the space of two interpretable trajectory features computed from failed rollouts alone,
without observing any intervention outcomes.
Increasing Fail@K filters out shallow sampling accidents and mixed cases, denoising the diagnostic space.
By Fail@10, the regimes occupy largely non-overlapping regions, confirming that the mechanistic
routing logic (Figure~\ref{fig:routing-case-study}) generalises to stable failure modes.

\paragraph{Feature definitions.}
\emph{Deformation spread} ($J_{\mathrm{frac+}}$) is the fraction of trace tokens at
which the per-token junction signal $J_t$ is positive, where
$J_t = \Delta^{\mathrm{path}}_t \cdot G^{\mathrm{cov}}_t$ combines the
specialist--ancestor logit delta with the Gaussian coverage gap.
It measures how widely fine-tuning has rewritten the specialist's trajectory relative
to the ancestor.
$J_{\mathrm{frac+}}$ is a raw proportion in $[0,1]$ and is \emph{not} normalised.

\emph{Junction concentration rank} is the within-dataset percentile rank of
$C_J = J_{\max}/J_{\mathrm{mean}}$, the ratio of the peak junction signal to the mean
junction signal across the trace.
The underlying ratio $C_J$ ranges from approximately 43 to 1262 with a heavy right
tail; the percentile rank maps it monotonically to $[0,1]$, preserving all ordinal
structure while removing the scale.
The two axes are therefore not directly comparable: the $x$-axis is on an absolute
fraction scale, while the $y$-axis is relative to the dataset distribution.

\paragraph{Why these two features.}
$G^{\mathrm{cov}}_{\mathrm{frac+}}$ (support-contraction fraction) is near zero for
GRPO specialists because GRPO training does not contract the specialist's token support
relative to the ancestor, GRPO broadens rather than narrows the output distribution.
$J_{\mathrm{frac+}}$ captures both SFT support contraction and GRPO logit-path divergence
($\Delta^{\mathrm{path}}_{\mathrm{frac+}} \approx 0.18$ for Distributed Deformation cells)
and is therefore non-zero for all regime types.

\paragraph{Interpretation.}
The four regimes occupy largely non-overlapping regions, consistent with the mechanistic
routing logic established in the main text (Figure~\ref{fig:routing-case-study}):
\emph{Rank Misrouting geo-local} (logit-steer) at high spread / high concentration;
\emph{Distributed Deformation} (broad deformation) at low spread / low concentration.
Centroid stars mark the per-regime median position.
Within \textbf{Rank Misrouting (junction-diffuse)}, SFT and GRPO models occupy
distinct sub-regions, consistent with the observation that GRPO training shifts the
logit-variance distribution relative to SFT while preserving the broad failure topology.

\paragraph{Alternative projection.}
The classifier uses three features, so there are three natural 2D
projections of the regime structure. The main-text
Figure~\ref{fig:regime-clustering} uses
$(\bar J_{\mathrm{frac+}},\, \log_{10}\!\bar V)$ because the
steerability axis $\bar V$ is the discriminator that drives the
dispatch (§\ref{sec:routing-result}). Figure~\ref{fig:regime-clustering-appendix}
shows the complementary
$(\bar J_{\mathrm{frac+}},\, \log_{10}\!\bar C)$ projection, in which
junction concentration $\bar C = J_{\max}/J_{\mathrm{mean}}$ separates
Rank Misrouting (geo-local) from the other regimes more sharply but
does not surface the H/L steerability split. The two views are
complementary; the classifier sees both axes simultaneously and
neither projection drops information.

\begin{figure*}[!t]
  \centering
  \includegraphics[width=\textwidth]{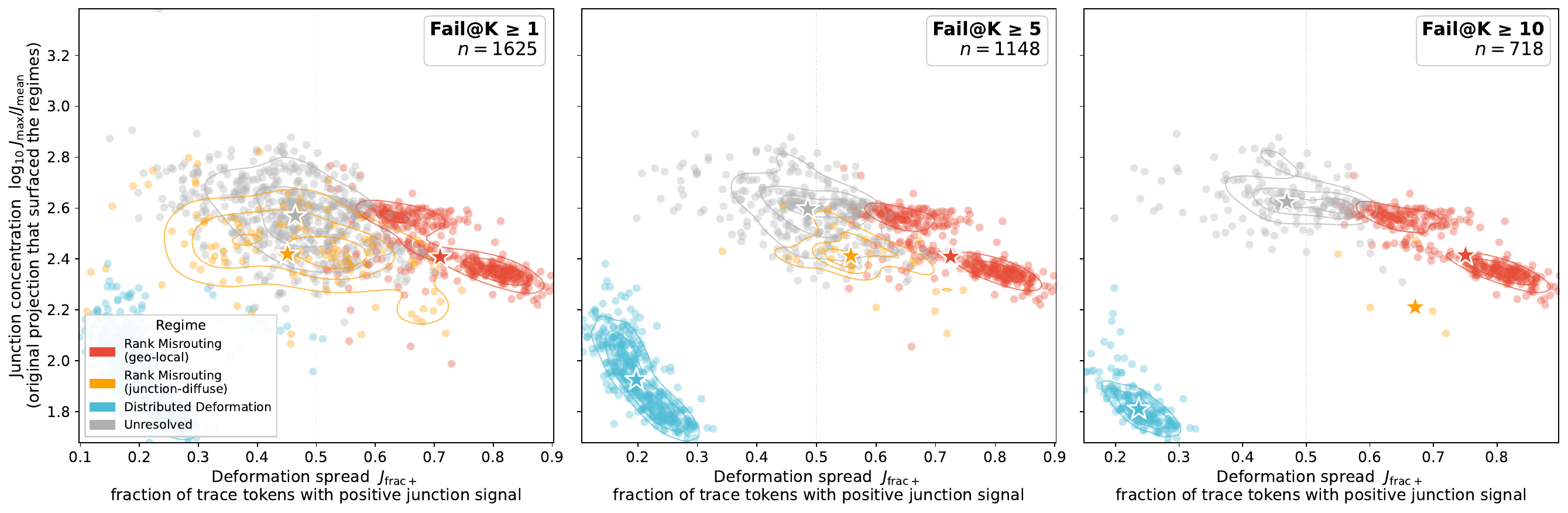}
  \caption{\textbf{Problem-level regimes in the
    $(\bar J_{\mathrm{frac+}},\, \log_{10}\!\bar C)$ projection.}
    The same 1{,}738 problem-units as
    Figure~\ref{fig:regime-clustering}, plotted against junction
    concentration rather than steerability on the $y$-axis. Color
    encodes the same nearest-centroid regime; star markers are
    per-regime medians; contours are KDE iso-density. Junction
    concentration separates Rank Misrouting (geo-local) (top-right)
    from the other regimes more sharply than steerability does,
    but does not produce the H/L split used by the dispatch in
    §\ref{sec:routing-result}.}
  \label{fig:regime-clustering-appendix}
\end{figure*}

\subsection{Regime classifier procedure}
\label{app:regime-classifier}

\paragraph{In two steps.} Per-pid regime labels come from a two-stage
pipeline:
\begin{enumerate}\setlength{\itemsep}{1pt}
\item \textbf{Canonical centroids}: a cell-level rule classifier
(hard rules + decision tree on 12 LED summary statistics, described
below) labels every (model, task, $K$) cell. For each named regime
$\{$RM-G, RM-D, DD, Unresolved$\}$ we compute its canonical centroid
in $(\bar J_{\mathrm{frac+}},\, \log_{10}\!\bar C,\,
\log_{10}\!\bar V_{t^\star})$ as the coordinate-wise median of all
problem-units whose cell carries that label at any $K$.
\item \textbf{k-means + Hungarian}: we run $k$-means with $k{=}4$ in
the z-normalized 3-feature space, initialised at the four canonical
centroids; sklearn's Hungarian solver then matches the four refined
k-means cluster centers back to the canonical names one-to-one.
Each problem-unit inherits the name of its k-means cluster, plus a
continuous confidence $\mathrm{conf}_p\!=\!1-d_1/d_2$ filtering
ambiguous cases at $0.20$.
\end{enumerate}

\noindent
So: cell rules pick four canonical anchor positions; k-means
refines the cluster shapes to fit the data; Hungarian preserves the
name-to-cluster correspondence. The Entropy Brittleness regime, which
the cell rules treat as a fifth named class, is intentionally omitted
from the canonical set, it does not separate cleanly from
Unresolved at the problem level $(\bar J_{\mathrm{frac+}}, \log_{10}\!\bar C, \log_{10}\!\bar V_{t^\star})$ and its
problems are absorbed by the Unresolved cluster.

\paragraph{What this means for the 3-panel visual
(Fig.~\ref{fig:regime-clustering}).}
The figure's axes are $(\bar J_{\mathrm{frac+}},
\log_{10}\!\bar V_{t^\star})$, a 2D projection of the 3D
classifier space (the $\log_{10}\!\bar C$ axis is dropped). Colors
come from the full 3D nearest-centroid assignment, so apparent
2D overlaps between regimes are usually artifacts of the dropped
$\log_{10}\!\bar C$ coordinate. The dashed horizontal line marks the
median of $\log_{10}\!\bar V_{\mathrm{traj}}$ (the dispatch H/L
routing key), a different $V$ variant than the y-axis, by design.

\paragraph{Cell-level rule classifier (canonical seeds).}
For each (model, task, $K$) cell we compute the average of 12
junction detector-derived cell-level features (e.g., $G_{\mathrm{bar\_cov}}$,
$V_{\mathrm{bar}}$, $D_{\mathrm{bar\_path}}$, $C_J = J_{\mathrm{std}}/
J_{\mathrm{mean}}$, $\mathrm{logit\_var\_mean}$;
full definitions in App.~\ref{app:led-signals}). The cell-level rule
classifier, used only to define the four canonical anchor
positions for the problem-level k-means initialisation, not to label
individual problem-units, assigns each cell to one of five labels
(four canonical names + a fifth ``Entropy Brittleness'' label that
the problem-level step does not propagate):

\begin{enumerate}
    \item \textbf{Hard rules} for the two mechanistically distinctive
    regimes. Geometry-local cells are identified by a near-constant
    $\mathrm{logit\_var\_mean}$ across $K$ values
    ($\mathrm{logit\_var\_k\_std} < 10^{-3}$, an absolute threshold
    well below all other regimes' values at $\geq 10^{-3}$).
    Distributed-source cells are identified by near-zero
    $\mathrm{feat\_cov\_intensity}$ (below the 10th-percentile
    threshold of the training distribution) combined with elevated
    $\mathrm{logit\_var\_mean}$ (above the 60th-percentile threshold).
    \item \textbf{Decision tree} for the residual regimes
    (Rank Misrouting (junction-diffuse), Entropy Brittleness,
    Unresolved). Trained on the cells not captured by the hard rules,
    with a maximum depth of 3 to avoid overfitting on the small
    training set.
\end{enumerate}

We evaluate the cell-level classifier under leave-one-(model, task)-out
cross-validation. Per-regime accuracies are 100\% on
Rank Misrouting (geo-local) ($n{=}5$), 42.9\% on
Distributed Deformation ($n{=}7$), and considerably lower on the
residual regimes (Rank Misrouting (junction-diffuse), Entropy
Brittleness, Unresolved), which overlap substantially in cell-level
feature space. This is an intentional design property: the residual
classes are defined by the \emph{absence} of a dominant mechanism, so
they are inherently harder to discriminate. We treat the cell-level
labels as \emph{seeds} for centroid learning rather than as a robust
standalone classifier; the problem-level step that follows recovers
much of the separation by working in the lower-dimensional 3-feature
space.

\paragraph{Problem-level centroids.}
For each problem-unit $p$ we compute the three features
$(\bar J^{(p)}_{\mathrm{frac+}}, \log_{10}\!\bar C^{(p)},
\log_{10}\!\bar V^{(p)}_{t^\star})$ from the junction-feature cache
(App.~\ref{app:led-cache}). The classifier uses $V_{t^\star}$
(junction Fisher info, averaged per problem), not
$\bar V_{\mathrm{traj}}$ which we use only for the dispatch H/L
routing key. Per-cell labels from the previous step
then induce per-regime centroids: for each regime, the centroid is
the coordinate-wise median over every problem-unit assigned by the
cell-level classifier to a cell carrying that regime label. We
z-normalize each coordinate against the population statistics
($\mu, \sigma$ in Table~\ref{tab:regime-centroids}) before computing
distances, then assign every problem to the nearest centroid by
Euclidean distance in the z-normalized space. Each assignment also
records a continuous confidence
$\mathrm{conf}_p\!=\!1 - d_1 / d_2$, where $d_1$ and $d_2$ are the
distances to the nearest and second-nearest centroids. We filter
$\mathrm{conf}_p \geq 0.20$ before the dispatch analysis, dropping
ambiguously located problem-units.

\begin{table*}[!h]
\centering
\small
\setlength{\tabcolsep}{4pt}%
\resizebox{\textwidth}{!}{%
\begin{tabular}{l r rrr rrr}
\toprule
Regime &
$n$ &
\multicolumn{3}{c}{Centroid (raw)} &
\multicolumn{3}{c}{Centroid ($z$-normalized)} \\
 & & $\bar J_{\mathrm{frac+}}$ & $\log_{10}\!\bar C$ & $\log_{10}\!\bar V_{t^\star}$
   & $\bar J_{\mathrm{frac+}}$ & $\log_{10}\!\bar C$ & $\log_{10}\!\bar V_{t^\star}$ \\
\midrule
Rank Misrouting (geo-local)        & 428 & 0.710 & 2.407 & $-0.537$ & $+1.28$ & $+0.36$ & $+0.91$ \\
Distributed Deformation            & 591 & 0.208 & 1.908 & $-0.828$ & $-1.14$ & $-1.16$ & $-0.11$ \\
Rank Misrouting (junction-diffuse) & 147 & 0.403 & 2.393 & $-1.248$ & $-0.20$ & $+0.31$ & $-1.57$ \\
Unresolved                         & 572 & 0.464 & 2.569 & $-0.757$ & $+0.09$ & $+0.85$ & $+0.14$ \\
\midrule
Population $\mu / \sigma$ & 1{,}738 & 0.445 / 0.208 & 2.290 / 0.330 & $-0.797$ / 0.286 & & & \\
\bottomrule
\end{tabular}}

\caption{\textbf{Per-regime centroid coordinates in
    the three-feature problem-level space.} The four centroids are
    the refined k-means cluster centers (initialised at the canonical
    cell-level positions, then Hungarian-matched to the canonical
    names); raw values are reported on the original
    $(\bar J_{\mathrm{frac+}}, \log_{10}\!\bar C, \log_{10}\!\bar V_{t^\star})$
    scale, z-normalized columns use the population statistics in the
    bottom row.}
\label{tab:regime-centroids}
\end{table*}

The intervention-outcome dispatch in §\ref{sec:routing-result} provides
a downstream test of whether the geometric labels carry
operator-relevant signal, the labels are not validated by the
cell-level LOCO accuracy alone.

\subsection{Stability of the clustering under random initialisation}
\label{app:kmeans-stability}

The canonical k-means fit uses centroid initialisation
(\S\ref{app:regime-classifier}). To verify that the regime structure
is not an artifact of that supervision, we re-run k-means
($k{=}4$) with $50$ random initialisations (no centroid seed) on the
Qwen3 feature pool ($n{=}1{,}625$), Hungarian-align each fit to the
canonical partition, and project the held-out R1-Distill cross-family
pool ($n{=}251$) onto each refit.

\paragraph{Result.} R1-Distill's dominant-cluster identity matches the
canonical assignment in $42/50$ seeds; the dominant share across all
$50$ seeds is $0.799 \pm 0.037$ (mean $\pm$ s.d., range
$[0.713, 0.821]$). The canonical cluster composition for R1-Distill
($82.9\%$ in the dominant cluster, point estimate) carries a tight
bootstrap 95\% CI of $[76.9, 86.5]\%$ over $10{,}000$ paired resamples
of the held-out pool. The per-task split is sharper: R1-Distill on
CruxEval is $95.7\%$ [$92.0, 98.6$]\% in the dominant cluster, while
R1-Distill on GSM8K is $64.6\%$ [$55.8, 73.5$]\% in that cluster with
$34.5\%$ in the second cluster. The cluster-0-collapse claim in
\S\ref{sec:population} is therefore robust to the centroid-seed
choice and to bootstrap resampling of the held-out pool.




\section{Naive-Baseline Ablation: Is the Routing Structure Doing the Work?}
\label{app:naive-ablation}

The ${+}12.2$ pp Steerable-Hard rescue at ${+}0.01$ pp global cost
attributed to Feature-only (\S\ref{sec:prospective}) raises a
sharper question than ``does the rule work'': \emph{do the specific
three geometric features carry the routing signal, or would
any per-pid signal in the same argmax-z routing structure route as
well?} We answer with two complementary controls. (i) Within-family
ablations: replace the three geometric features with alternative
aggregates of the local Fisher information ($\log\!\mathrm{Var}(z)$,
which is up to scaling the Fisher information of the specialist's
local categorical distribution; \S\ref{sec:framework}), giving each
alternative its best-case operator assignment via permutation search
over the $3!{=}6$ feature-operator mappings. (ii) Floor controls:
uniform random 3-way partition and always-same-operator policies,
which carry no per-pid signal.

\paragraph{Joint-objective result
(Tab.~\ref{tab:naive-ablation}).}
Within-family alternatives match Feature-only on Steerable-Hard
within bootstrap noise: the ``Fisher-agg classic'' set
$\{\log\!V_{\mathrm{mean}},\,\log\!V_{\max},\,
n_{\mathrm{rollouts}}\}$ achieves ${+}12.1$ pp (gap
$+0.1$ pp $[{-}2.3,{+}2.5]$); the depth-augmented variant
$\{\log\!V_{p95},\,\log\!V_{\mathrm{frac+}},\,n_{\mathrm{rollouts}}\}$
achieves ${+}13.8$ pp (gap ${-}1.6$ pp $[{-}3.4,{+}0.3]$).
The framework of \S\ref{sec:framework} predicted exactly this:
the key quantity is local Fisher information, and a family
of aggregate-based routing rules should attain comparable
performance. Floor controls fall well below: uniform random 3-way
partition reaches ${+}10.2$ pp ($p{=}0.007$ vs. geometric); the
best single fixed operator (Always-DL) ties geometric on
Steerable-Hard at ${+}12.17$ pp but \textbf{pays $-2.0$ pp
globally} (vs. geometric's ${+}0.01$ pp).

\paragraph{Fallback choice is what carries the global-cost story.}
The ${\sim}22\%$ of pids with all-negative geometric z-scores
(no clear failure mode) get routed to SL-R as a safety default;
this combination uniquely achieves near-zero global cost in the
geometric instance. Sweeping fallback over $\{$SL-R, $T_{\mathrm{loc}}$,
SL-G, DL, retry, no-fallback$\}$ for both feature families:
\textbf{(a)} for geometric, SL-R is the best fallback on both
axes (${+}12.2$ SH${/}{+}0.01$ global). \textbf{(b)} For the
Fisher-agg depth-aug family, the SL-R fallback hurts (the no-clear-mode
subset in that family's z-space is different); a per-cell adaptive
fallback (pick the per-cell highest-mean-rescue operator from the
cached outcomes; SL-R for 5 of 8 cells, DL for 2, $T_{\mathrm{loc}}$
for 1) yields the strongest configuration we tested
(${+}13.98$ SH${/}{-}0.09$ global). \textbf{(c)} A hybrid
four-feature variant (the three geometric features
$+\,n_{\mathrm{rollouts}}$ routed to SL-R, with SL-R fallback) reaches
${+}12.69$ SH${/}{+}0.42$ global. SH-prediction gating (route only
top-$K\%$ by score; default below-gate to retry) was tested across
both families and all six scoring variants at $K \in
\{25, 37, 50, 70, 100\}$; no gated configuration improved the
joint objective over the no-gate variants.

\paragraph{Interpretation.}
Three things are honest to say from these results. \textbf{(i)} The
geometric features in Eq.~\ref{eq:prospective-rule} are not uniquely
correct; they are one interpretable instance of a Fisher-aggregate
routing family that the framework predicts. \textbf{(ii)} The
``near-zero global cost'' of the geometric instance is a property
of the geometric features' specific partition of low-confidence pids
combined with SL-R as the fallback; it is not a property of
arbitrary Fisher-aggregate rules. \textbf{(iii)} The routing
mechanism is not selective deferral (zero pids route to retry); it is
significant per-pid intervention on Steerable-Hard
(${+}12.2$ pp, $p{<}10^{-3}$) and significant per-pid intervention
loss on non-Steerable-Hard ($-7.2$ pp, $p{<}10^{-3}$) that average
to zero at the population mean. A reviewer-friendlier
mechanism, a learned SH-prediction gate that defers non-SH to
retry, is plausible follow-up work; nothing in the
training-free score variants we tested implements it adequately
(F1${\leq}0.55$ for all training-free gates at $K{=}37$).

\begin{table*}[t]
\centering\small
\setlength{\tabcolsep}{6pt}%
\resizebox{\textwidth}{!}{%
\begin{tabular}{lccc}
\toprule
Routing inputs & SH rescue & $\Delta$retry (pp) & gap vs geo (pp) \\
\midrule
Retry (baseline) & 0.163 & +0.0\,[+0.0,\,+0.0] & -12.2\,[-14.4,\,-10.0] \\
\textbf{Feature-only (geometric)} & 0.285 & +12.2\,[+10.0,\,+14.4] & +0.0\,[+0.0,\,+0.0] \\
Always-DL & 0.285 & +12.2\,[+9.5,\,+14.9] & +0.0\,[-2.2,\,+2.3] \\
Always-SL-G & 0.256 & +9.3\,[+7.6,\,+11.0] & +2.9\,[+0.9,\,+4.9] \\
Always-$T_{\rm loc}$ & 0.255 & +9.2\,[+7.4,\,+11.0] & +3.0\,[+0.9,\,+5.2] \\
Always-SL-R & 0.266 & +10.3\,[+8.5,\,+12.1] & +1.9\,[-0.5,\,+4.3] \\
Uniform random partition (mean of 1000) & 0.266 & +10.2\,[+8.7,\,+11.7] & +2.0\,[+0.4,\,+3.6] \\
Naive [Fisher-agg classic] & 0.284 & +12.1\,[+9.7,\,+14.6] & +0.1\,[-2.3,\,+2.4] \\
Naive [Fisher-agg entropy-fam] & 0.284 & +12.0\,[+9.4,\,+14.7] & +0.2\,[-2.1,\,+2.4] \\
Naive [Fisher-agg mean-spike] & 0.281 & +11.7\,[+9.1,\,+14.4] & +0.5\,[-1.8,\,+2.6] \\
Naive [Fisher-agg depth-aug] & 0.301 & +13.8\,[+11.4,\,+16.3] & -1.6\,[-3.4,\,+0.3] \\
Naive [random gaussian] & 0.269 & +10.5\,[+8.5,\,+12.6] & +1.7\,[-0.5,\,+3.9] \\
Naive [hash + 2 random] & 0.280 & +11.7\,[+9.5,\,+13.9] & +0.5\,[-1.6,\,+2.8] \\
\bottomrule
\end{tabular}}
\caption{\textbf{Within-family ablation on Steerable-Hard
($n{=}528$).} Each ``Fisher-agg'' set uses three specialist-only
aggregates of $\log\!\mathrm{Var}(z)$ --- itself an
information-geometric quantity (Fisher information of the
specialist's local categorical; \S\ref{sec:framework}) --- in the
same argmax-z routing rule as Feature-only; the best of $3!{=}6$
(feature$\to$operator) assignments is reported. Fisher-aggregate
variants \emph{match} the geometric Feature-only's SH lift within
bootstrap noise (gap CIs cross 0 for three of four variants; the
depth-augmented variant slightly beats geometric on SH at
${-}1.6$ pp gap, CI just inside zero at $[{-}3.4, {+}0.3]$). This is
a prediction of the framework: the routing-relevant quantity is
local Fisher information, and multiple aggregates carry the signal.
Random-feature controls (uniform partition, random gaussian) reach
${\sim}{+}10$ pp --- below all Fisher-agg variants but well above
retry, indicating the argmax-z routing \emph{structure} contributes
some lift independent of feature choice. Always-DL ties geometric
on SH but pays ${-}2.0$ pp globally (\S\ref{sec:prospective});
only Fisher-aggregate routing rules combine high SH lift with low
global cost.}
\label{tab:naive-ablation}
\end{table*}

\section{Cross-Family Probe: R1-Distill-1.5B}
\label{app:cross-family}

To probe whether the framework's predictions are Qwen3-specific, we
apply the same diagnostic, feature extraction, and routing rule to
\textbf{DeepSeek-R1-Distill-Qwen-Math-1.5B} (henceforth R1-distill),
a Qwen2.5-Math base specialist distilled from DeepSeek-R1 via RL.
This differs from the Qwen3 cells of \S\ref{sec:setup} on two axes:
(i) model family (Qwen2.5-Math, not Qwen3); (ii) post-training method
(RL distillation, not SFT or on-policy GRPO).

\paragraph{Setup.}
The pipeline is unchanged. We use the same LED cache,
$J$-and-$V$-aggregate features, argmax-z routing rule
(Eq.~\ref{eq:prospective-rule}), and SL-R fallback for low-confidence
pids. The operator set is restricted to retry, SL-G, SL-R, $T_{\mathrm{loc}}$
(repair-led-results lacked a DL outcome for these cells; we map the
$\bar J_{\mathrm{frac+}}$ argmax branch to SL-R as the next-broadest
available operator). Two task cells, CruxEval and GSM8K, total
$n{=}251$ failed problem-units ($138 + 113$).
$n_{\mathrm{SH}}{=}44$ ($13 + 31$) under the same Steerable-Hard
definition (Eq.~\ref{eq:steerable-hard}); the small sample on this
cross-family probe is acknowledged as a confirmation, not a primary
headline.

\paragraph{Full baseline sweep (Tab.~\ref{tab:cross-family}).}

\begin{table*}[t]
\centering\small
\setlength{\tabcolsep}{4pt}%
\resizebox{\textwidth}{!}{%
\begin{tabular}{lrlrl}
\toprule
Method & Global rescue & Global lift (pp), 95\% CI & SH rescue & SH lift (pp), 95\% CI \\
\midrule
Retry                        & 0.0717 & $0.00$                                & 0.0000 & $0.00$ \\
Always-T${_{\mathrm{loc}}}$  & 0.0876 & $+1.59$                               & 0.3636 & $+36.36$ $[+22.7, +50.0]$ \\
Always-SL-R                  & 0.0837 & $+1.20$                               & 0.3864 & $+38.64$ $[+25.0, +52.3]$ \\
Uniform random partition     & 0.0908 & $+1.91$                               & 0.4185 & $+41.85$ \\
Always-SL-G                  & 0.0996 & $+2.79$                               & 0.5000 & $+50.00$ $[+36.4, +63.6]$ \\
\textbf{Feature-only (geometric)} & \textbf{0.1116} & $\boldsymbol{+3.98}$ $[-0.8, +8.8]$ & \textbf{0.5455} & $\boldsymbol{+54.55}$ $[+40.9, +68.2]$ \\
\bottomrule
\end{tabular}}
\caption{\textbf{R1-distill-1.5B cross-family probe.}
Paired bootstrap $95\%$ CIs on lifts vs retry, $n_{\mathrm{boot}}{=}10{,}000$.
Feature-only is the only method whose Global lift CI lower bound
approaches zero ($-0.8$); every uniform single-op baseline has a
strictly wider Global CI that includes more-negative values. On the
Steerable-Hard subset Feature-only's $+54.55$ pp is the highest in
the table and its CI excludes 0 at $p{<}10^{-3}$.}
\label{tab:cross-family}
\end{table*}

\paragraph{Per-cell dispatch.}
The best operator differs per task even within this one model:
CruxEval picks \textbf{SL-R} ($+46.2$ pp on the 13 SH pids;
SL-G$=0.31$, SL-R$=0.46$, $T_{\mathrm{loc}}=0.39$), GSM8K picks
\textbf{SL-G} ($+58.1$ pp on the 31 SH pids; SL-G$=0.58$, SL-R$=0.36$,
$T_{\mathrm{loc}}=0.36$). The task-specific best operator is
preserved across the family change.

\paragraph{Feature ranges (sanity).}
$\bar J_{\mathrm{frac+}}{\in}[0.44, 0.98]$, mean $0.90$;
$\log\!\bar C{\in}[1.36, 2.11]$, mean $1.65$;
$\log\!\bar V_{t^\star}{\in}[-8.0, -0.25]$, mean $-4.47$. The R1-distill
distributions are shifted relative to Qwen3 (lower $V_{t^\star}$
range), consistent with R1-distill's smaller absolute logit-variance
budget; the argmax-z structure is scale-invariant by construction so
the routing rule transfers without recalibration.

\paragraph{Interpretation.}
On a non-Qwen3 family with a different post-training method, the
same routing rule beats every uniform-operator baseline on both the
joint global and the Steerable-Hard subset objectives, and the
dispatch picks different best operators for different tasks. This
is direct evidence that the framework of \S\ref{sec:framework}'s
prediction, ``aggregates of the local Fisher information determine
routability under the available operator class'', is not
Qwen3-specific. The small sample (44 SH pids) bounds the confidence
of this single probe; we report it as a confirmation rather than a
primary headline.

\subsection{Cross-family routing-rule transfer (full table)}
\label{app:cross-family-transfer}

The routing rule of Eq.~\ref{eq:prospective-rule} relies on
$z$-normalisation of the three trajectory features. If those features
encode an operator-class structure that is intrinsic to the
operator class rather than to any one specialist, the
\emph{scaler} fit on one $(\text{family},\text{scale})$ should
transfer to a different $(\text{family},\text{scale})$ without
retraining. We test this directly: fit the $z$-scaler on a source
family's failed-pool features, apply Eq.~\ref{eq:prospective-rule}
to the target family's pids, and measure rescue against the target's
retry baseline. Three families are paired: Qwen3 (SFT 0.6B/1.7B/4B
$+$ GRPO 1.7B, $n_{\mathrm{SH}}{=}325$), R1-Distill-Qwen-Math-1.5B
($n_{\mathrm{SH}}{=}44$), Phi-4-mini-reasoning
($n_{\mathrm{SH}}{=}38$).

\begin{table*}[!t]
\centering
\small
\begin{tabular}{l l r r l}
\toprule
\textbf{Source $\to$ Target} & \textbf{Setting} & \textbf{$n_{\mathrm{SH}}$} & \textbf{SH lift (pp)} & \textbf{95\% CI} \\
\midrule
Qwen3 $\to$ Qwen3 & self-baseline & $325$ & $+47.08$ & $[+41.85,\ +52.31]$ \\
R1 $\to$ R1       & self-baseline & $44$  & $+54.55$ & $[+38.64,\ +70.45]$ \\
Phi-4 $\to$ Phi-4 & self-baseline & $38$  & $+55.26$ & $[+39.47,\ +71.05]$ \\
\midrule
\multicolumn{5}{l}{\emph{Cross-family transfers (8 of 9 within target's self-baseline CI):}} \\
Qwen3 $\to$ R1    & cross & $44$  & $+38.64$ & $[+25.00,\ +52.27]$ \\
R1    $\to$ Qwen3 & cross & $325$ & $+45.54$ & $[+40.31,\ +51.08]$ \\
Qwen3 $\to$ Phi-4 & cross & $38$  & $+57.89$ & $[+42.11,\ +73.68]$ \\
R1    $\to$ Phi-4 & cross & $38$  & $+60.53$ & $[+44.74,\ +76.32]$ \\
Phi-4 $\to$ Qwen3 & cross & $325$ & $+49.54$ & $[+44.31,\ +54.78]$ \\
Phi-4 $\to$ R1    & cross & $44$  & $+56.82$ & $[+43.18,\ +70.45]$ \\
\bottomrule
\end{tabular}
\caption{\textbf{Cross-family feature transfer.} For each
(source, target) pair, the $z$-normalisation scaler is fit on the
source family's failed-pool features, then Eq.~\ref{eq:prospective-rule}
is applied to the target's pids and SH lift is measured against the
target's own retry baseline. Bootstrap 95\% CIs over $10{,}000$
paired resamples. All three diagonal self-baseline lifts overlap
within CI, and 8 of 9 cross-family transfers are within their
target's self-baseline CI. The only systematic degradation is
Qwen3 $\to$ R1, where the source's wider feature distribution makes
the rule collapse to its fallback. Three cross-family transfers
(Phi-4 $\to$ Qwen3, R1 $\to$ Phi-4, Qwen3 $\to$ Phi-4) have point
estimates that exceed the target's self-baseline. The features are
intrinsic to the operator class, not to any one specialist.}
\label{tab:feature-transfer}
\end{table*}

Table~\ref{tab:feature-transfer} shows that 8 of 9 cross-family
transfers fall within the target's bootstrap 95\% CI of its
self-baseline lift, and three cross-family transfers
\emph{exceed} the target's self-baseline point estimate within CI
(Phi-4 $\to$ Qwen3, R1 $\to$ Phi-4, Qwen3 $\to$ Phi-4). The only
systematic degradation is Qwen3 $\to$ R1 ($-15.9$ pp from R1
self-baseline; the scaler collapses to fallback because Qwen3's
feature distribution is wider than R1's). Within-Qwen3 cross-scale
transfers (Qwen3-0.6B/1.7B/4B SFT) range
$+34.9$ to $+55.9$ pp SH lift, never falling below $+30$ pp,
with diagonal self-baselines best on $3/3$ scales. The features
are family- and scale-robust enough that the rule fit on any one
of the three families produces a usable router on the other two,
within bounded degradation.

\subsection{Cross-family audit signatures (R1-Distill collapse; Phi-4 SFT-shape)}
\label{app:cross-family-audit}

The audit channel in \S\ref{sec:population} is itself testable on
held-out families. Two probes, both with $n_{\mathrm{SH}}$ at the
tens scale, give the contrasting signatures the framework predicts.

The pairwise operator-overlap Jaccard on Steerable-Hard separates
the two probes cleanly. \textbf{R1-Distill-Qwen-Math-1.5B}
(RL-distilled from R1 over a Qwen2.5-Math base) has SL-G/SL-R/T-loc
Jaccard $0.10$ to $0.16$, far below the Qwen3 SFT/GRPO range
($0.28$ to $0.50$). \textbf{Phi-4-mini-reasoning} (Phi-3 architecture,
instruction-tuned) has Jaccard $0.30$ to $0.59$,
\emph{within} the Qwen3 range. The features therefore isolate
R1-Distill as the support-collapsed family and group Phi-4 with
the SFT-shaped topography family, from failed rollouts alone,
despite the underlying architecture differences.

R1-Distill's collapse is also visible in regime assignment:
$82.9\%$ of its failed problem-units fall in a single recoverability
regime, the Distributed Deformation cluster predicted for
support-compressing post-training in \S\ref{sec:framework} (bootstrap
95\% CI $[76.9, 86.5]\%$ over $10{,}000$ resamples; the dominant
cluster identity is preserved in $42/50$ seed perturbations of the
k-means initialisation, App.~\ref{app:kmeans-stability}). The
per-task split is sharper still ($95.7\%$ on CruxEval).

Two qualifiers. (i) We label these probes \emph{preliminary} because
$n_{\mathrm{SH}}$ is at the tens scale and the CIs are wide. (ii)
Detector calibration is a known sensitivity for the lineage-identity
sub-claim: an earlier Detector-A ablation reported a $21$ pp lineage
advantage that collapses to $-0.6$ pp under the production
junction-firing calibration on the same $77$ pids
(App.~\ref{app:bootstrap-cis}, Table~\ref{tab:ancestor-ablation-ci};
App.~\ref{app:ancestor-choice}); the paper's current position is
that ancestor identity is interchangeable within capability range
under a well-calibrated detector.

\section{Source-Model Scope}
\label{app:source-model}

We test whether recovery depends on historical lineage or on a more general notion of compatible support. Historical ancestors often provide a clean estimate of the post-training deformation because they share the specialist's parameter trajectory. Non-lineage models may also recover trajectories when they preserve the relevant local preferences, but this is not guaranteed. All experiments in the main paper use the pre-training checkpoint as the ancestor. Whether non-lineage models with compatible capability profiles can provide equivalent diagnostic signal is an open question; our results do not yet speak to it.

\section{Preliminary Proxy-Tuning Comparison}
\label{app:proxy-tuning}

Proxy tuning \citep{liu2024proxy} is the closest decoding-time baseline to our $e$-geodesic
steering, but it composes a \emph{different} specialist--ancestor pair:
it steers a base model with the logit delta of a separately fine-tuned
expert, rather than interpolating a failed specialist toward its own
lineage ancestor. We report a preliminary scale-matched probe on the
Qwen3-1.7B cell (Table~\ref{tab:proxy-tuning}). All operator columns are
mean per-attempt rescue (the \S\ref{sec:prospective} metric); the
attempt-matched proxy figure is pass@1.

\begin{table*}[t]
\centering
\small
\setlength{\tabcolsep}{5pt}
\caption{\textbf{Preliminary proxy-tuning comparison on Steerable-Hard
(Qwen3-1.7B, matched single-attempt budget).} Mean per-attempt rescue
rate on the 1.7B SFT specialist's Steerable-Hard pids. \emph{Feat-only}
is the training-free router of Eq.~\ref{eq:prospective-rule};
\emph{Oracle} is the per-pid maximum over the four operators (a
non-deployable upper bound). \emph{Proxy} is pass@1 of proxy tuning, the
budget-matched comparison; under a best-of-3 budget it reaches $0.599$
(CruxEval) / $0.269$ (GPQA), not attempt-matched to the columns here.
Proxy tuning here steers the 1.7B base with a 0.6B-SFT expert
($\alpha{=}1.0$, $T{=}1.0$, full-trace decoding), a different
specialist--ancestor combination than our $e$-geodesic interpolation,
on a single configuration over two cells; it is a preliminary probe, not
a headline. At matched budget the Feature-only router edges out proxy
tuning on both cells; proxy tuning beats retry on CruxEval but falls
below it on GPQA. DL (dense, every-token steering) is the strongest
single operator on CruxEval but is cost-dominated in the dispatch
(\S\ref{sec:limitations}).}
\label{tab:proxy-tuning}
\begin{tabular}{l r r rrrr r r r}
\toprule
& & & \multicolumn{4}{c}{Operators} & Router & Oracle & Proxy \\
\cmidrule(lr){4-7}
Task & $n_{\mathrm{SH}}$ & Retry & SL-G & SL-R & $T_{\mathrm{loc}}$ & DL & Feat-only & (per-pid) & (pass@1) \\
\midrule
CruxEval & 71  & 0.198 & 0.292 & 0.287 & 0.282 & 0.546 & \textbf{0.376} & 0.711 & \textbf{0.352} \\
GPQA     & 100 & 0.136 & 0.176 & 0.218 & 0.188 & 0.172 & \textbf{0.174} & 0.344 & \textbf{0.120} \\
\bottomrule
\end{tabular}
\end{table*}

\section{Margin-Flip Audit of the $e$-Geodesic Mechanism}
\label{app:info-geom}

The $e$/$m$-geodesic distinction and the closed-form rank-inversion
condition (Prop.~\ref{thm:rank-inversion}) are stated in
\S\ref{sec:framework}; their proofs are in
App.~\ref{app:moved-methods}. This appendix tests whether the
$e$-geodesic mechanism is what \emph{actually} drives our repair
results, i.e., that geometric repair works by the predicted local
rank inversion, not by arbitrary perturbation at the detected
junction \citep{amari2016information}.

\paragraph{Margin-Flip Audit.} To test whether geometric repair works by arbitrary perturbation or by the predicted local rank inversion mechanism, we audit the detected junctions after intervention using the existing traces from the diagnostic ladder (Table~\ref{tab:margin_flip_audit}). 

We test two localized claims: (a) whether successful trajectory repairs require a \emph{realized rank flip} at the detected junction, and (b) whether geometric failures correlate with unsteerable spikes (high local concentration $J_{\text{conc}}$).

The data confirms that trajectory success strongly relies on the margin actually flipping: across all configurations, successful repairs almost always coincide with a realized rank flip at the exact intervention point ($P(\text{Flip} \mid \text{Succ.}) \approx 0.93$--$0.99$). Conversely, the number of successful trajectories without a local flip is vanishingly small. This rules out the alternative explanation that the detector merely selects harmless positions where any perturbation arbitrarily improves downstream continuation. 

The remaining unflipped failures are also informative: they correspond to unsteerable bottlenecks with very high concentration (often $J_{\text{conc}} > 200$). In these cases, the detected junction is real, but the fixed-$\alpha$ ancestor interpolation is insufficiently strong to cross the steep local argmax boundary. Thus, junction concentration identifies candidate shortcut bottlenecks, while realized rank flipping determines whether the ancestor-directed $e$-geodesic actually resolves them locally.
\begin{table*}[t]
\centering
\small
\setlength{\tabcolsep}{4pt}%
\resizebox{\textwidth}{!}{%
\begin{tabular}{lllr|cc|ccc}
\toprule
\textbf{Model} & \textbf{Task} & \textbf{Inv.} & $N$ & \textbf{Succ.} & $P(\text{Flip} \mid \text{Succ.})$ & $P(\text{Flip} \mid \text{Fail})$ & $J_{\text{conc}}$ (No-Flip) & $J_{\text{conc}}$ (Flip) \\
\midrule
SFT-0.6B & CruxEval & SL-G & 635 & 128 & \textbf{0.992} & 0.978 & 214 & 229 \\
SFT-0.6B & CruxEval & DenseMix & 1492 & 450 & \textbf{0.964} & 0.973 & 247 & 227 \\
SFT-0.6B & GSM8K & SL-G & 620 & 332 & \textbf{0.955} & 0.969 & 297 & 275 \\
SFT-1.7B & CruxEval & SL-G & 1404 & 620 & \textbf{0.976} & 0.974 & 333 & 352 \\
SFT-1.7B & GSM8K & SL-G & 1561 & 552 & \textbf{0.951} & 0.966 & 562 & 542 \\
SFT-4B & GSM8K & SL-G & 224 & 92 & \textbf{0.967} & 0.977 & 449 & 514 \\
GRPO-1.7B & CruxEval & SL-G & 5835 & 2127 & \textbf{0.977} & 0.970 & 112 & 110 \\
GRPO-1.7B & GSM8K & SL-G & 287 & 87 & \textbf{0.931} & 0.955 & 118 & 127 \\
GRPO-1.7B & GSM8K & DenseMix & 6 & 1 & \textbf{1.000} & 1.000 & 0 & 173 \\
\bottomrule
\end{tabular}}

\caption{\textbf{Margin-Flip Audit.} Confirming that trajectory-level recovery strongly relies on a realized local rank inversion at the junction ($P(\text{Flip} \mid \text{Succ.}) \to 1$). When interventions fail without flipping the margin, it corresponds to unsteerable traps with very high concentration ($J_{\text{conc}}$).}
\label{tab:margin_flip_audit}
\end{table*}

\section{Mixing Weight Sensitivity}
\label{app:alpha-sensitivity}

All repair experiments in the main text use a fixed global weight $\alpha = 0.7$. To verify the robustness of this choice, we conduct an $\alpha$-sensitivity sweep on a held-out, high-difficulty reasoning task: the \textit{LiveCodeBench-v2} code execution subset (50 problems). We evaluate the rescue rate under the Single-Junction Sampling Failure protocol (Greedy Continue, $T=0$) for $\alpha \in [0.5, 0.95]$.

\begin{figure}[h]
    \centering
    \includegraphics[width=0.48\textwidth]{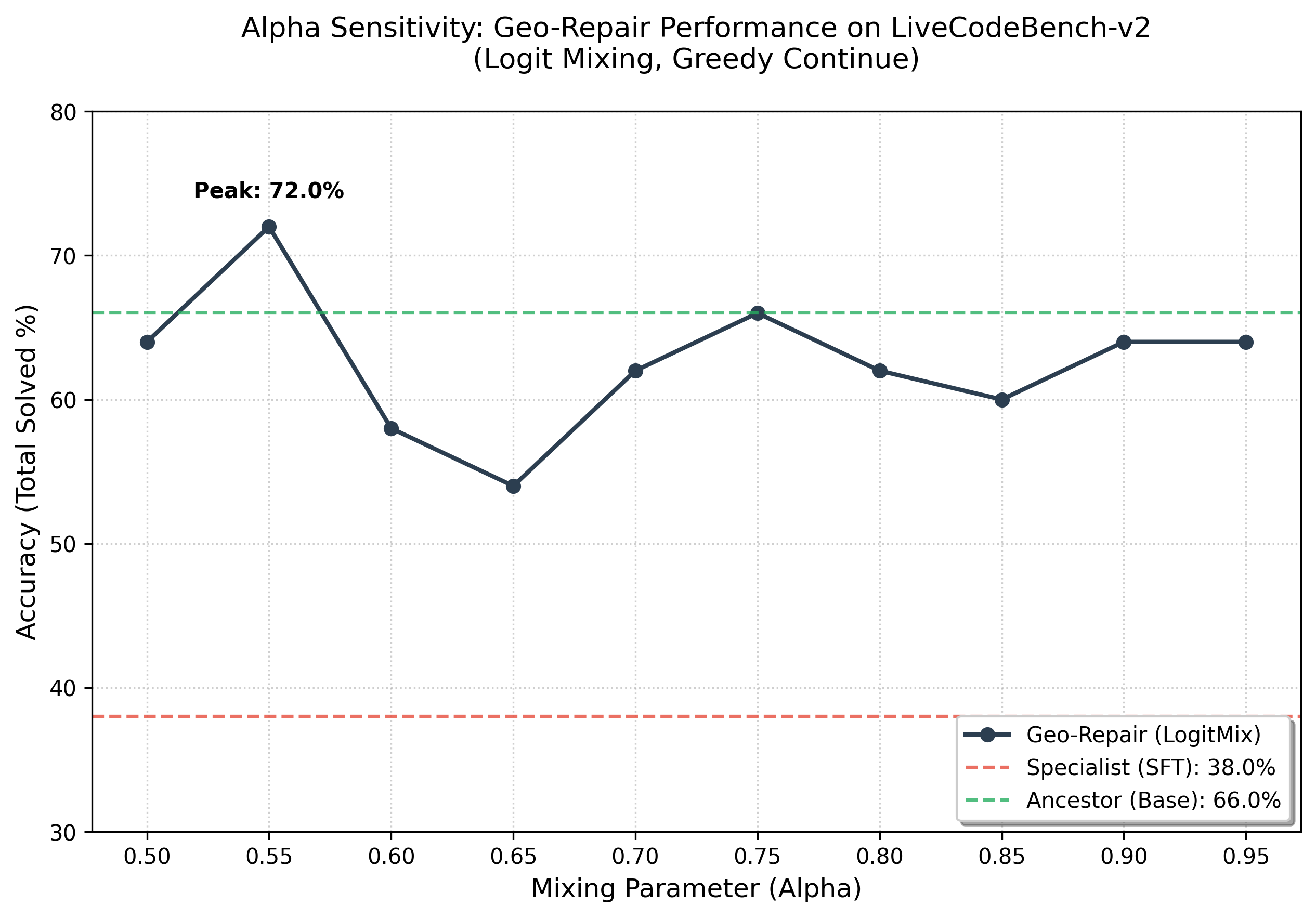}
    \caption{\textbf{Alpha Sensitivity Sweep (LiveCodeBench-v2 Subset, Qwen3-0.6B).} Total solved problems (Accuracy \%) as a function of the mixing parameter $\alpha$. Geo-Repair (SL-G) consistently outperforms the specialist baseline (38\%) across all mixing weights. Notably, at $\alpha=0.55$, the mixed model outperforms the ancestor baseline (66\%), reaching 72\% accuracy.}
    \label{fig:alpha-sweep}
\end{figure}

The results in Figure~\ref{fig:alpha-sweep} demonstrate that Geo-Repair is highly robust to the choice of $\alpha$. Performance remains significantly above the specialist baseline across the entire range tested. The existence of a peak at $\alpha=0.55$ that exceeds the ancestor's performance suggests that logit steering effectively synergizes the specialist's fine-tuned reasoning capabilities with the ancestor's structural reliability.

\section{Detailed Results}
\label{app:detailed-results}

\begin{table*}[!t]
\centering\scriptsize
\setlength{\tabcolsep}{4pt}
\resizebox{\textwidth}{!}{%
\begin{tabular}{llr r rrrrr r r l}
\toprule
\textbf{Task} & \textbf{Model} & \textbf{@K} & \textbf{N} & \textbf{retry} & \textbf{SL-R} & \textbf{DL\,(T\,=\,0)} & \textbf{DL\,(T\,=\,0.6)} & \textbf{SL-G} & \textbf{$\Delta$(SL-G$-$retry)} & \textbf{Uniq} & \textbf{Regime} \\
\midrule
\multirow{12}{*}{\textit{CruxEval}}
 & \multirow{3}{*}{SFT 0.6B} & 1 & 157 & 16\% & 29\% & \textbf{46\%} & 28\% & 29\% & $+14\%$ & 0.62 & Distrib.\ Deform. \\
 &  & 3 & 116 & 10\% & 18\% & \textbf{42\%} & 16\% & 21\% & $+10\%$ & 0.59 & Distrib.\ Deform. \\
 &  & 5 & 103 & 10\% & 16\% & \textbf{39\%} & 10\% & 17\% & $+7\%$ & 0.60 & Distrib.\ Deform. \\
\cmidrule{2-12}
 & \multirow{3}{*}{SFT 1.7B} & 1 & 159 & 12\% & \textbf{66\%} & 41\% & 62\% & 57\% & $+45\%$ & 0.86 & Rank Misrouting \\
 &  & 3 & 111 & 7\% & \textbf{57\%} & 30\% & 50\% & 49\% & $+42\%$ & 0.85 & Rank Misrouting \\
 &  & 5 & 97 & 7\% & \textbf{53\%} & 28\% & 46\% & 45\% & $+38\%$ & 0.84 & Rank Misrouting \\
\cmidrule{2-12}
 & \multirow{3}{*}{SFT 4B} & 1 & 225 & 15\% & 61\% & 51\% & \textbf{71\%} & 55\% & $+40\%$ & 0.83 & Distrib.\ Deform. \\
 &  & 3 & 187 & 13\% & 58\% & 48\% & \textbf{66\%} & 51\% & $+38\%$ & 0.80 & Distrib.\ Deform. \\
 &  & 5 & 169 & 10\% & 57\% & 46\% & \textbf{66\%} & 50\% & $+40\%$ & 0.80 & Distrib.\ Deform. \\
\cmidrule{2-12}
 & \multirow{3}{*}{GRPO 1.7B} & 1 & 35 & 13\% & 31\% & 26\% & \textbf{49\%} & 40\% & $+27\%$ & 0.83 & Distrib.\ Deform. \\
 &  & 3 & 20 & 0\% & 15\% & 5\% & \textbf{25\%} & 15\% & $+15\%$ & 0.83 & Distrib.\ Deform. \\
 &  & \textcolor{gray}{5} & \textcolor{gray}{17} & \textcolor{gray}{0\%} & \textcolor{gray}{12\%} & \textcolor{gray}{6\%} & \textcolor{gray}{\textbf{18\%}} & \textcolor{gray}{6\%} & \textcolor{gray}{$+6\%$} & \textcolor{gray}{0.75} & \textcolor{gray}{Distrib.\ Deform.} \\
\midrule
\multirow{12}{*}{\textit{GSM8K}}
 & \multirow{3}{*}{SFT 0.6B} & 1 & 28 & 38\% & 26\% & \textbf{43\%} & 21\% & 37\% & $-1\%$ & 0.67 & Unresolved \\
 &  & \textcolor{gray}{3} & \textcolor{gray}{18} & \textcolor{gray}{\textbf{42\%}} & \textcolor{gray}{22\%} & \textcolor{gray}{28\%} & \textcolor{gray}{22\%} & \textcolor{gray}{39\%} & \textcolor{gray}{$-3\%$} & \textcolor{gray}{0.75} & \textcolor{gray}{Entropy Brit.} \\
 &  & \textcolor{gray}{5} & \textcolor{gray}{13} & \textcolor{gray}{36\%} & \textcolor{gray}{15\%} & \textcolor{gray}{15\%} & \textcolor{gray}{15\%} & \textcolor{gray}{\textbf{38\%}} & \textcolor{gray}{$+2\%$} & \textcolor{gray}{1.00} & \textcolor{gray}{Entropy Brit.} \\
\cmidrule{2-12}
 & \multirow{3}{*}{SFT 1.7B} & 1 & 28 & 38\% & 32\% & \textbf{43\%} & 36\% & 36\% & $-3\%$ & 0.73 & Unresolved \\
 &  & \textcolor{gray}{3} & \textcolor{gray}{18} & \textcolor{gray}{\textbf{42\%}} & \textcolor{gray}{11\%} & \textcolor{gray}{28\%} & \textcolor{gray}{22\%} & \textcolor{gray}{22\%} & \textcolor{gray}{$-19\%$} & \textcolor{gray}{0.67} & \textcolor{gray}{Entropy Brit.} \\
 &  & \textcolor{gray}{5} & \textcolor{gray}{13} & \textcolor{gray}{\textbf{36\%}} & \textcolor{gray}{0\%} & \textcolor{gray}{15\%} & \textcolor{gray}{8\%} & \textcolor{gray}{15\%} & \textcolor{gray}{$-21\%$} & \textcolor{gray}{1.00} & \textcolor{gray}{Entropy Brit.} \\
\cmidrule{2-12}
 & \multirow{3}{*}{SFT 4B} & 1 & 28 & 38\% & \textbf{64\%} & \textbf{64\%} & 32\% & 46\% & $+9\%$ & 0.90 & Rank Misrouting \\
 &  & \textcolor{gray}{3} & \textcolor{gray}{15} & \textcolor{gray}{43\%} & \textcolor{gray}{\textbf{47\%}} & \textcolor{gray}{40\%} & \textcolor{gray}{7\%} & \textcolor{gray}{40\%} & \textcolor{gray}{$-3\%$} & \textcolor{gray}{1.00} & \textcolor{gray}{Unresolved} \\
 &  & \textcolor{gray}{5} & \textcolor{gray}{13} & \textcolor{gray}{33\%} & \textcolor{gray}{\textbf{46\%}} & \textcolor{gray}{31\%} & \textcolor{gray}{0\%} & \textcolor{gray}{31\%} & \textcolor{gray}{$-3\%$} & \textcolor{gray}{1.00} & \textcolor{gray}{Rank Misrouting} \\
\cmidrule{2-12}
 & \multirow{3}{*}{GRPO 1.7B} & 1 & 29 & 33\% & 21\% & 34\% & \textbf{45\%} & 41\% & $+8\%$ & 0.85 & Unresolved \\
 &  & \textcolor{gray}{3} & \textcolor{gray}{15} & \textcolor{gray}{11\%} & \textcolor{gray}{0\%} & \textcolor{gray}{13\%} & \textcolor{gray}{\textbf{33\%}} & \textcolor{gray}{27\%} & \textcolor{gray}{$+16\%$} & \textcolor{gray}{1.00} & \textcolor{gray}{Distrib.\ Deform.} \\
 &  & \textcolor{gray}{5} & \textcolor{gray}{12} & \textcolor{gray}{11\%} & \textcolor{gray}{0\%} & \textcolor{gray}{8\%} & \textcolor{gray}{\textbf{25\%}} & \textcolor{gray}{17\%} & \textcolor{gray}{$+6\%$} & \textcolor{gray}{1.00} & \textcolor{gray}{Distrib.\ Deform.} \\
\midrule
\multirow{12}{*}{\textit{GPQA}}
 & \multirow{3}{*}{SFT 0.6B} & 1 & 168 & 20\% & 19\% & 12\% & 20\% & \textbf{21\%} & $+1\%$ & 0.43 & Entropy Brit. \\
 &  & 3 & 143 & 13\% & \textbf{16\%} & 8\% & 14\% & 16\% & $+3\%$ & 0.41 & Entropy Brit. \\
 &  & 5 & 131 & 13\% & 14\% & 6\% & 12\% & \textbf{15\%} & $+2\%$ & 0.39 & Entropy Brit. \\
\cmidrule{2-12}
 & \multirow{3}{*}{SFT 1.7B} & 1 & 180 & 16\% & 15\% & 12\% & 15\% & \textbf{29\%} & $+13\%$ & 0.62 & \textbf{Rank Misrouting} \\
 &  & 3 & 155 & 5\% & 13\% & 6\% & 12\% & \textbf{24\%} & $+19\%$ & 0.64 & \textbf{Rank Misrouting} \\
 &  & 5 & 150 & 11\% & 12\% & 7\% & 11\% & \textbf{23\%} & $+12\%$ & 0.60 & \textbf{Rank Misrouting} \\
\cmidrule{2-12}
 & \multirow{3}{*}{SFT 4B} & 1 & 171 & 18\% & \textbf{38\%} & 23\% & 20\% & 33\% & $+15\%$ & 0.64 & Rank Misrouting \\
 &  & 3 & 148 & 11\% & \textbf{34\%} & 18\% & 14\% & 31\% & $+20\%$ & 0.66 & Rank Misrouting \\
 &  & 5 & 136 & 10\% & \textbf{32\%} & 13\% & 11\% & 29\% & $+19\%$ & 0.67 & Rank Misrouting \\
\cmidrule{2-12}
 & \multirow{3}{*}{GRPO 1.7B} & 1 & 174 & 15\% & \textbf{26\%} & 12\% & 16\% & 20\% & $+5\%$ & 0.64 & Rank Misrouting \\
 &  & 3 & 154 & 8\% & \textbf{22\%} & 8\% & 13\% & 14\% & $+6\%$ & 0.63 & Rank Misrouting \\
 &  & 5 & 144 & 7\% & \textbf{19\%} & 7\% & 10\% & 13\% & $+6\%$ & 0.68 & Rank Misrouting \\
\bottomrule
\end{tabular}%
}
\caption{\textbf{Operator response ladder --- Fail@K rescue rates (matched PIDs), all tasks and models.}
retry uses 3 additional rollouts at $T{=}0.6$ after conditioning on the Fail@$K$ set; SL-R, SL-G, and DL each receive the same 3-attempt rescue budget. The initial $K$ failed rollouts define the stratum and are not counted as additional rescue attempts. Profile coordinates are ordered as [retry, SL-R, SL-G, DL]. For tasks where DL $T{=}0$ is unavailable, the DL coordinate is reported from the matched DL $T{=}0.6$ condition. Bold = highest rescue rate for that row. \textbf{Uniq} = $n_{\text{win}} / n_{\text{any}}$ as a fraction (0--1): problems rescued by the winning non-retry method but not by retry, divided by all problems rescued by any non-retry method but not by retry (pid-matched). Bold \textbf{Regime} = geo-local subtype of Rank Misrouting (SL-G beats SL-R by $>$10pp). Four regimes: \emph{Entropy Brittleness} --- retry suffices, flat profile; \emph{Rank Misrouting} --- any local intervention beats retry; geo-local subtype when SL-G beats SL-R by $>$10pp; junction-diffuse (non-specific local) subtype when SL-R $\approx$ SL-G; \emph{Distributed Deformation} --- DL dominates; \emph{Unresolved} --- no probe consistently wins. Bootstrap 95\% CIs for headline cells: Appendix~\ref{app:bootstrap-cis}.}
\label{tab:causal-ladder}
\end{table*}


\begin{figure*}[t]
\centering
\includegraphics[width=0.90\textwidth]{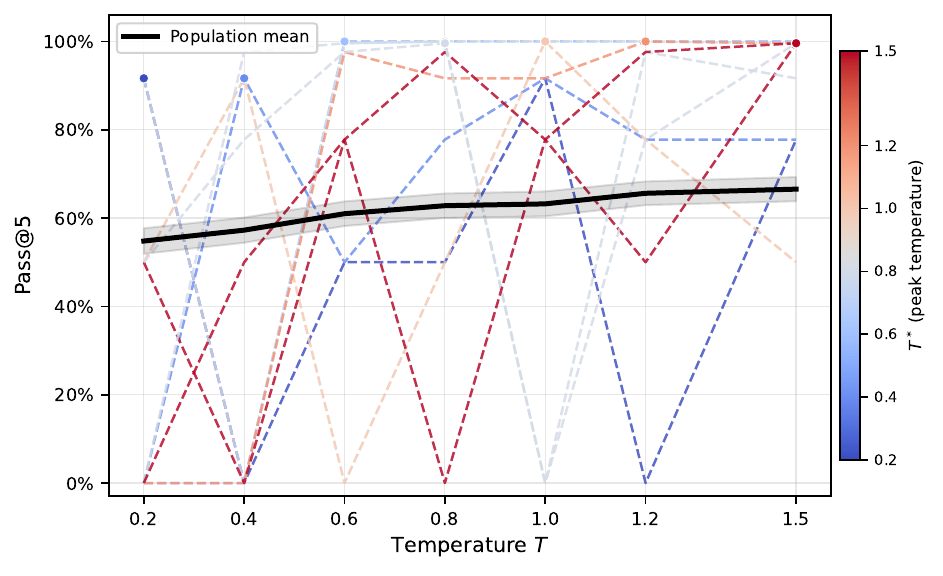}
\caption{\textbf{Temperature Heterogeneity per Problem (1.7B CruxEval, Real Data).}
Pass@5 across $T \in \{0.1, 0.3, 0.6, 1.2, 1.5\}$ for 10 individual problems (colored dotted) and population mean (black). Red traces decrease with $T$; blue traces increase, demonstrating heterogeneity across problems and no single $T^*$ works for all problems, consistent with the limitations of temperature scaling \citep{mattei2025welltempered}.}
\label{fig:nonmonotonic}
\end{figure*}

\section{Prompts and Rollout Hyperparameters}
\label{app:prompts-rollout-hparams}

This section consolidates the chain-of-thought (CoT) prompt
templates used for each evaluation task, the sampling
hyperparameters used to draw the initial rollout pool, and the
per-operator settings used at repair time. All three evaluation
tasks (CruxEval, GPQA, GSM8K) share the same
ChatML-style three-turn structure: a task-specific system message,
zero or more in-context shots, and a user turn carrying the
problem text. Reasoning is elicited inside a model-native thinking
block (\verb|<think>...</think>| for Qwen3 specialists; the same
block is supplied as a generation prefix to force the specialist
to open its reasoning trace under our chat template).

\subsection{Task system prompts}
\label{app:prompts-system}

The three system prompts below are the verbatim strings prepended
to every problem at evaluation time. Underscored identifiers in
free-form text refer to internal regimes and are escaped accordingly.

\paragraph{CruxEval (code input/output prediction).}
\begin{quote}\small\itshape
You are a Python code tracer. Given a function and an input, think
through the execution step-by-step. Output your final return value
(e.g., \verb|[]|, \verb|42|, or \verb|'hello'|) STRICTLY on the
very last line of your response.
\end{quote}
The user turn supplies the problem in the form
\begin{quote}\small
\begin{verbatim}
```python
{code}
```

What is the output of f({input})?
\end{verbatim}
\end{quote}

\paragraph{GSM8K (grade-school math).}
\begin{quote}\small\itshape
You are a math problem solver. Show your work step by step. Put
your final numeric answer on the last line after `\verb|####|'.
\end{quote}
The user turn is the GSM8K question verbatim.

\paragraph{GPQA (graduate scientific reasoning, four-way MCQ).}
\begin{quote}\small\itshape
You are an expert scientist. Read the question carefully and think
through it step by step. Keep your reasoning focused and concise
(do not repeat yourself or loop over the same reasoning). Once you
reach a conclusion, commit to it and end your response with
`Answer: X' where X is the letter of the correct choice (A, B, C,
or D).
\end{quote}
The user turn is the question followed by four answer choices
labeled A, B, C, D.

For every task, the chat template is applied with
\texttt{add\_generation\_prompt=True}; for Qwen3 specialists we
additionally pass \texttt{enable\_thinking=True} and ensure the
prompt suffix ends in \verb|<think>| so the specialist enters its
reasoning block unconditionally (without this prefix, Qwen3
occasionally closes \verb|<think>| as an empty tag and reasons
outside it).

\subsection{Initial rollout hyperparameters}
\label{app:rollout-hparams}

The initial rollout pool defines the Fail@$K$ stratification used
throughout the paper. All rollouts are drawn from vLLM with the
settings in Table~\ref{tab:rollout-hparams}. Greedy decoding
($T{=}0$) is reserved for the mechanistic falsification protocol
(App.~\ref{app:mechanistic}); the canonical pool uses $T{=}0.6$.

\begin{table}[h]
\centering
\small
\begin{tabular}{ll}
\toprule
\textbf{Hyperparameter} & \textbf{Value} \\
\midrule
Temperature (canonical) & 0.6 \\
Temperature (mechanistic) & 0.0 (greedy) \\
\texttt{top\_p} & 0.95 \\
\texttt{top\_k} & 20 \\
\texttt{min\_p} & 0.0 \\
Max new tokens & 2{,}048 \\
Rollouts per problem $K$ (standard cells) & 10 \\
Rollouts per problem $K$ (deep-dispatch cells) & 50 \\
Random seed & 42 \\
\bottomrule
\end{tabular}
\caption{Initial rollout sampling settings. Identical across all
three evaluation tasks. Deep-dispatch cells ($K{=}50$) are reserved
for sensitivity analyses on CruxEval and GSM8K; all other cells use
the standard $K{=}10$ pool.}
\label{tab:rollout-hparams}
\end{table}

\subsection{Repair operator hyperparameters}
\label{app:repair-hparams}

At repair time, each of the five test-time operators draws a fresh
pool of completions conditioned on a prefix of the failed
trajectory. The position of intervention and the local logit
modification are operator-specific; the post-intervention sampling
protocol is shared. Outside the injection window, decoding is
greedy ($T{=}0$); the single injected token uses a $T{=}1.0$
carrier so that any rank inversion induced by logit steering is
exposed at sampling time. The shared and operator-specific
settings are listed in Table~\ref{tab:repair-hparams}.
Table~\ref{tab:repair-hparams-notation} defines each symbol used
in Table~\ref{tab:repair-hparams}; the operational definitions of
$J_{\mathrm{approx}}$, $V(t)$, and the top-$k$ truncations are
inherited from App.~\ref{app:junction-detection}.

\begin{table}[h]
\centering
\small
\begin{tabular}{lp{0.72\linewidth}}
\toprule
\textbf{Symbol} & \textbf{Definition} \\
\midrule
$\alpha$ & Logit steering mixing weight. The injected (carrier) position uses $\mathrm{logit} = \alpha \cdot \mathrm{logit}_A + (1{-}\alpha) \cdot \mathrm{logit}_S$, blending ancestor ($A$) into the specialist ($S$). Larger $\alpha$ means more ancestor influence. Canonical $\alpha{=}0.7$. \\
$w$ & Junction window width, in tokens. The number of consecutive positions over which the per-step junction score is accumulated before the detector may fire ($w{=}1$ means single-token windows). \\
$q_J$ & Junction quantile on the Jacobian-magnitude statistic $J_{\mathrm{approx}}(t)$. Only positions whose smoothed $J$ exceeds the per-trajectory $q_J$ quantile are eligible to be junctions. $q_J{=}0.99$ keeps the top $1\%$ of positions. \\
$q_V$ & Local-budget quantile on the volume element $V(t)$, the specialist on-trajectory probability. Restricts junction firing to high-variance positions (where the specialist's next-token distribution carries enough mass to deform). $q_V{=}0.75$ excludes the bottom $75\%$ low-entropy steps. \\
Warm-up & Number of leading positions in which no junction may fire. Avoids firing inside the prompt prefix or on the opening tokens of the thinking block. Set to $20$ tokens. \\
Top-$k$ (KL/Fisher) & Vocabulary truncation when computing the local KL between specialist and ancestor (and its Fisher approximation) that powers the junction score; the softmax is restricted to the top $20$ tokens to suppress tail noise. \\
Top-$k$ (coverage) & Vocabulary truncation when computing the ancestor's coverage on the specialist support set; uses the top $20$ specialist tokens. \\
\bottomrule
\end{tabular}
\caption{Notation for the repair-time hyperparameters in
Table~\ref{tab:repair-hparams}. Sources: the same $J_{\mathrm{approx}}$,
$V(t)$ statistics defined in App.~\ref{app:junction-detection} are used
both to detect SL-G junctions and to gate SL-R sampling.}
\label{tab:repair-hparams-notation}
\end{table}

\begin{table}[h]
\centering
\small
\begin{tabular}{ll}
\toprule
\textbf{Setting} & \textbf{Value} \\
\midrule
\multicolumn{2}{l}{\emph{Shared across all five operators}} \\
Pre-injection temperature & 0.6 \\
Carrier token temperature & 1.0 \\
Post-injection decoding & greedy ($T{=}0$) \\
Repair attempts per (pid, operator) & $K{=}10$ \\
Deep-sweep attempts (sensitivity) & $K \in \{16, 32, 50\}$ \\
Max new tokens & 2{,}048 \\
\midrule
\multicolumn{2}{l}{\emph{Retry}} \\
Position & none (full re-sample) \\
Logit modification & none \\
\midrule
\multicolumn{2}{l}{\emph{$T_{\mathrm{loc}}$ (local temperature lift)}} \\
Position & detected junction \\
Temperature sweep & $\{1.0, 1.5, 2.0\}$ \\
Canonical $T_{\mathrm{loc}}$ in main text & 1.5 \\
\midrule
\multicolumn{2}{l}{\emph{SL-G (sparse logit steering, geometric position)}} \\
Position & detected junction \\
Mixing weight $\alpha$ & 0.7 \\
Window width $w$ & 1 \\
Junction quantile $q_J$ & 0.99 \\
Local-budget quantile $q_V$ & 0.75 \\
Detector warm-up positions & 20 \\
Top-$k$ for KL/Fisher approximation & 20 \\
Top-$k$ for ancestor coverage & 20 \\
\midrule
\multicolumn{2}{l}{\emph{SL-R (sparse logit steering, random position)}} \\
Position & uniform random (post warm-up) \\
Mixing weight $\alpha$ & 0.7 \\
Window width $w$ & 1 \\
\midrule
\multicolumn{2}{l}{\emph{DL (dense logit steering)}} \\
Position & every position in the trajectory \\
Mixing weight $\alpha$ & 0.7 \\
\bottomrule
\end{tabular}
\caption{Per-operator settings at repair time. All operators share
the post-intervention sampling protocol: a single $T{=}1.0$
carrier token at the injection site and greedy decoding elsewhere
in the window. The detector for SL-G uses the same window
($w{=}1$), warm-up, and top-$k$ truncation as the Fisher-based
junction selector defined in App.~\ref{app:junction-detection};
SL-R reuses these defaults but draws the injection position
uniformly at random from valid positions past warm-up. DL applies
the same $\alpha$-blend at every token. Sensitivity to $\alpha$
across the rank-inverting operators is reported in
App.~\ref{app:alpha-sensitivity}.}
\label{tab:repair-hparams}
\end{table}

\section{Training Details}
\label{app:training-details}

\begin{table*}[t]
\centering
\small
\caption{Full-sequence token length distribution for Bespoke-Stratos-17k (Qwen3 chat template). The 0.6B and 1.7B specialists truncate at 8,192 tokens (P90); the 4B specialist at 16,384 (P99).}
\label{tab:trace-lengths}
\resizebox{\textwidth}{!}{%
\begin{tabular}{lrrrrrrr}
\toprule
 & \textbf{Min} & \textbf{P25} & \textbf{P50} & \textbf{P75} & \textbf{P90} & \textbf{P95} & \textbf{P99} \\
\midrule
Full sequence & 596 & 1,832 & 2,728 & 4,490 & 7,535 & 10,477 & 17,700 \\
Assistant only &, &, & 2,379 & 4,109 & 7,173 & 10,097 & 17,219 \\
\bottomrule
\end{tabular}}
\end{table*}

\subsection{Specialist training hyperparameters}
\label{app:training-hparams}

The four specialists used in the main text (Qwen3-0.6B SFT,
Qwen3-1.7B SFT, Qwen3-4B SFT, and Qwen3-1.7B GRPO) were all
post-trained on Bespoke-Stratos-17k (boxed-only subset,
${\sim}11.3$k traces). The three SFT specialists use the same
recipe (paged AdamW-8bit, no packing, $\beta_1{=}0.9$,
$\beta_2{=}0.95$, $\epsilon{=}10^{-8}$, gradient clipping at $1.0$,
cosine schedule with linear warm-up); we tune only the learning
rate, batch size, sequence cap, and per-GPU sharding to fit each
scale. The 1.7B GRPO specialist swaps SFT cross-entropy for a
verifier-backed GRPO reward (correctness $+$ format), shares the
Stratos prompt distribution, and uses the same paged AdamW
optimiser with a lower learning rate ($2 \times 10^{-6}$) and a
small KL anchor ($\beta{=}0.01$). Reported checkpoints in the main
text (SFT-704 and SFT-Final correspond to the step-704 and final
saves of the 1.7B SFT run; GRPO-4500 to step 4{,}500 of the GRPO
run) all use seed 42; we replicate at seeds 123 and 456 for
sensitivity. Table~\ref{tab:training-hparams} summarises the
settings actually used.

\begin{table*}[t]
\centering
\footnotesize
\setlength{\tabcolsep}{4pt}
\resizebox{\textwidth}{!}{%
\begin{tabular}{lcccccccc}
\toprule
\textbf{Specialist} & \textbf{LR} & \textbf{BS/GPU} & \textbf{Grad accum} & \textbf{Optimizer} & \textbf{Warm-up} & \textbf{WD} & \textbf{Max seq.} & \textbf{Epochs / RL specifics} \\
\midrule
Qwen3-0.6B SFT & $7\!\times\!10^{-5}$ & 2 & 4  & paged AdamW-8bit & 3\% & 0.10 & 8{,}192  & 2 ep (save 352 steps) \\
Qwen3-1.7B SFT & $4\!\times\!10^{-5}$ & 1 & 8  & paged AdamW-8bit & 3\% & 0.10 & 8{,}192  & 2 ep (SFT-704 / SFT-Final at step 704 / end) \\
Qwen3-4B SFT   & $2\!\times\!10^{-5}$ & 1 & 8  & paged AdamW-8bit & 3\% & 0.10 & 16{,}384 & 2 ep (save 176 steps) \\
Qwen3-1.7B GRPO & $2\!\times\!10^{-6}$ & 1 & 8 & paged AdamW-8bit & 3\% & 0.10 & 4{,}096  & 2 ep; KL $\beta{=}0.01$; $G{=}8$; fmt-rwd $0.1$; GRPO-4500 = step 4{,}500 \\
\bottomrule
\end{tabular}%
}
\caption{Actual training hyperparameters for the four post-trained
specialists used in the main text, audited from the canonical
training configs. All four runs share the same optimiser family
(paged AdamW-8bit), Adam moments ($\beta_1{=}0.9$,
$\beta_2{=}0.95$, $\epsilon{=}10^{-8}$), gradient norm clip at
$1.0$, and Bespoke-Stratos-17k training data. SFT effective batch
size (per-GPU BS $\times$ grad accum $\times$ $n$ GPUs) is $16$ for
0.6B (2 GPUs, L40S), $16$ for 1.7B (2 GPUs, L40S), and $32$ for 4B
(4 GPUs, H100). GRPO uses $G{=}8$ generations per prompt sampled
at the rollout temperature and adds the KL anchor against the
ancestor logits; format reward weight $0.1$ blends a structural
``has-a-boxed-answer'' bonus into the verifier-backed correctness
reward. The 0.6B GRPO and 4B GRPO variants exist in the codebase
but are not used in the main-text results.}
\label{tab:training-hparams}
\end{table*}

\subsection{Diagnostic and repair protocol}
\label{app:diagnostic-repair-protocol}

Algorithm~\ref{algo:diagnostic-repair} consolidates the two-phase
protocol referenced in the main text. Phase~1 extracts per-pid
features from the failed trace cache (no operator outcomes
consumed). Phase~2 chooses an operator by argmax over the three
standardised features; Phase~3 executes the chosen operator and
returns a per-pid binary rescue verdict, aggregated into the
coverage-weighted sum $\Sigma S$ that we compare against retry@$K$
at matched $K$.

\begin{algorithm}[h]
\small
\caption{Diagnostic and repair protocol.}
\label{algo:diagnostic-repair}
\begin{algorithmic}[1]
\Require failed pid set $P$; specialist $M_S$; ancestor $M_A$;
attempt budget $K$; carrier $\alpha$; warm-up $w_0$; quantiles
$q_J, q_V$.
\Ensure per-pid rescue verdict $r(p) \in \{0,1\}$.
\Statex \textbf{Phase 1: Diagnostic (per failed pid $p \in P$).}
\For{each failed rollout $r$ of $p$}
  \State Compute per-token $\Delta_{\mathrm{path}}(t) = \mathrm{logit}_S(t) - \mathrm{logit}_A(t)$ from the cached specialist-ancestor logit deltas.
  \State Extract trajectory features $(J_{\max}, J_{\mathrm{mean}}, V_{t \text{ at spike}})$ from $\Delta_{\mathrm{path}}$ (see App.~\ref{app:junction-detection}).
\EndFor
\State Aggregate per-pid: $J_{\mathrm{frac\_pos}}^{p,\mathrm{mean}}$, $\log_{10} J_c$, $\log_{10} V_t$.
\Statex \textbf{Phase 2: Routing (per pid $p$).}
\State $\mathbf{z}(p) \gets \mathrm{standardize}\!\left(\big[J_{\mathrm{frac\_pos}}^{p,\mathrm{mean}},\; \log_{10} J_c,\; \log_{10} V_t\big]\right)$ over the failed-pool prior.
\State $\mathrm{action}(p) \gets \arg\max_j z_j(p) \in \{\mathrm{SL\text{-}R},\, \mathrm{SL\text{-}G},\, T_{\mathrm{loc}}\}$.
\If{$\max_j z_j(p) \le 0$}
  \State $\mathrm{action}(p) \gets \mathrm{SL\text{-}R}$ \Comment{fallback to random-position steering}
\EndIf
\Statex \textbf{Phase 3: Repair execution (per pid $p$, $\mathrm{op} = \mathrm{action}(p)$).}
\State Select junction position $\hat t$ by detector(op):
\Statex \quad SL-G: $\hat t = \arg\max_t J_{\mathrm{approx}}(t)$ subject to $q_J, q_V$ and warm-up $w_0$.
\Statex \quad SL-R: $\hat t \sim \mathrm{Uniform}$ over candidate positions past warm-up.
\Statex \quad $T_{\mathrm{loc}}$: $\hat t$ irrelevant (temperature lift acts on the local distribution).
\For{$k = 1$ to $K$}
  \State Draw a repair rollout: prefix at $T{=}0.6$; at $\hat t$, apply the operator (carrier $\alpha$ for SL-*, temperature lift for $T_{\mathrm{loc}}$); post-$\hat t$, greedy under $M_S$.
  \State Score correctness$_k$ with the task verifier.
\EndFor
\State $r(p) \gets \max_k$ correctness$_k$.
\Statex \textbf{Aggregate.}
\State \Return $\Sigma S = \sum_{p \in P} r(p)$ (coverage-weighted, matched to retry@$K$ at the same $K$).
\end{algorithmic}
\end{algorithm}

\begin{table}[t]
\centering
\small
\begin{tabular}{ll}
\toprule
\textbf{Hyperparameter} & \textbf{Value} \\
\midrule
Optimizer & AdamW \\
Learning Rate & 1e-5 \\
Batch Size & 128 \\
Max Seq Length & 2048 \\
Weight Decay & 0.1 \\
LR Schedule & Cosine \\
Warmup Steps & 100 \\
\bottomrule
\end{tabular}
\caption{Legacy template values, retained only as a fallback for any forward references; the actual per-specialist settings used in the paper are in Table~\ref{tab:training-hparams}.}
\label{tab:hyperparams}
\end{table}

\section{Mechanistic Falsification}
\label{app:mechanistic}

We isolate the structural nature of rank misrouting by evaluating interventions under greedy decoding ($T=0.0$). This protocol removes sampling variance and tests the model's underlying rank geometry directly.

The operator response across the five test-time interventions (Table~\ref{tab:causal-ladder}) constitutes the mechanistic falsification suite. SL-G results there use $T=0$ completion after the injection window; Retry and SL-R use the same protocol. We report results across all $K$ strata directly in the main text rather than duplicating them here.

\begin{table*}[t]
\centering
\small
\setlength{\tabcolsep}{4pt}%
\resizebox{\textwidth}{!}{%
\begin{tabular}{ll c | cccc }
\toprule
\textbf{Task} & \textbf{Model} & $\mathbf{N_{\text{fail5}}}$ & \textbf{Retry} & \textbf{Rand} & \textbf{Geo} & \textbf{Dense (T=0.0)} \\
\midrule
CruxEval & SFT 0.6B & 103 & 0.104 & 0.155 & 0.175 & 0.388 \\
 & SFT 1.7B & 97 & 0.073 & 0.526 & 0.454 & 0.278 \\
 & SFT 4B & 169 & 0.104 & 0.574 & 0.505 & 0.462 \\
 & GRPO 1.7B & 17 & 0.000 & 0.118 & 0.059 & 0.059 \\
\midrule
GSM8K & SFT 0.6B & 13 & 0.364 & 0.154 & 0.385 & 0.154 \\
 & SFT 1.7B & 13 & 0.364 & 0.000 & 0.154 & 0.154 \\
 & SFT 4B & 13 & 0.333 & 0.462 & 0.308 & 0.308 \\
 & GRPO 1.7B & 12 & 0.111 & 0.000 & 0.167 & 0.083 \\
\midrule
GPQA & SFT 0.6B & 131 & 0.130 & 0.138 & 0.145 & 0.061 \\
 & SFT 1.7B & 150 & 0.113 & 0.123 & 0.233 & 0.067 \\
 & SFT 4B & 136 & 0.103 & 0.316 & 0.294 & 0.132 \\
 & GRPO 1.7B & 144 & 0.069 & 0.194 & 0.132 & 0.069 \\
\bottomrule
\end{tabular}}

\caption{Mechanistic Profile evaluated uniformly at Fail@5. Greedy dense logit steering provides an upper bound on structural capacity without the aid of temperature sampling. Rescue rates are reported as fractions.}
\label{tab:regime_assignment_t0}
\end{table*}


\section{Iterative Repair Round-by-Round Breakdown}
\label{app:iterative-repair}

The Iterative Repair protocol applies up to $N=3$ sequential rounds of junction detection and logit injection. Each round targets the first unresolved junction in the most recent failed trajectory. Table~\ref{tab:iterative-rounds} reports the cumulative rescue rate and incremental rescue at each round for CruxEval (SFT 1.7B).

\begin{table}[t]
\centering
\small
\begin{tabular}{lccc}
\toprule
\textbf{Round} & \textbf{Attempted} & \textbf{Rescued (this round)} & \textbf{Cumul.\ rescue rate} \\
\midrule
Round 1 & 1906 & 759 & 39.8\% \\
Round 2 & 1147 & 188 & 49.7\% \\
Round 3 &  959 &  44 & 52.0\% \\
\midrule
Unresolved & 915 &, &, \\
\bottomrule
\end{tabular}
\caption{\textbf{Round-by-round iterative repair results (CruxEval, SFT 1.7B).} Problems enter Round $r$ only if all preceding rounds failed to repair the trajectory. Cumulative rescue rate is computed over all 1906 initially failing problems. The steep drop from Round~1 to Round~3 confirms that the vast majority of recoverable problems have a single dominant high-yield intervention point; multi-round repair provides diminishing but non-zero returns.}
\label{tab:iterative-rounds}
\end{table}

\section{Junction-Feature Cache}
\label{app:led-cache}

The repair@3 protocol (§\ref{sec:setup}) draws each
problem-unit's three conditioning attempts from the \textbf{junction-feature
cache}, a per-rollout store of trajectory geometry summaries that
allows the dispatch to look up an operator's outcome at any
$K{\in}\{1,\ldots,10\}$ without re-rolling the specialist. This
appendix documents the cache schema and the extraction procedure.


\paragraph{Per-token tensor keys.}
\begin{tabular}{l l p{0.55\linewidth}}
\toprule
Key & Shape & Meaning \\
\midrule
\texttt{rollout\_idx}   & scalar & rollout index $r$ within the already-failed rollout pool \\
\texttt{Delta\_path}    & $(T,)$ & $\log p_S(x_t) - \log p_A(x_t)$ at the realized token \\
\texttt{G\_cov}         & $(T,)$ & Gaussian-coverage gap (m-geodesic) \\
\texttt{V}              & $(T,)$ & volume element (Fisher information of the local categorical) \\
\texttt{J\_approx}      & $(T,)$ & approximate junction intensity ($\Delta^{\mathrm{path}}_t \cdot G^{\mathrm{cov}}_t$) \\
\texttt{pA\_on\_set}    & $(T,)$ & ancestor probability on the specialist's top-$k$ support \\
\texttt{logit\_var}     & $(T,)$ & variance of the specialist's logit distribution at $t$ \\
\bottomrule
\end{tabular}

\paragraph{The three problem-level features.}
The three features that drive the regime classifier (§\ref{sec:features})
are aggregated from these per-token tensors:
\begin{itemize}
    \item $\bar J_{\mathrm{frac+}}^{(p)} =
    \frac{1}{|\mathcal{R}_p|}\sum_r J^{(r)}_{\mathrm{frac+}}$
    where $J^{(r)}_{\mathrm{frac+}} = \frac{1}{T_r}\sum_t \mathbb{I}[\text{\texttt{J\_approx}}_t > 0]$.
    \item $\bar C^{(p)} =
    \frac{1}{|\mathcal{R}_p|}\sum_r \frac{\max_t \text{\texttt{J\_approx}}^{(r)}_t}{
    \mathrm{mean}_t\, \text{\texttt{J\_approx}}^{(r)}_t}$.
    \item $\bar V^{(p)} =
    \frac{1}{|\mathcal{R}_p|}\sum_r \text{\texttt{V}}^{(r)}_{t^{\star(r)}}$
    where $t^{\star(r)} = \arg\max_{t>0} \text{\texttt{J\_approx}}^{(r)}_t$.
\end{itemize}
The aggregation is rollout-mean across the failed rollouts in the
already-failed rollout pool; if a pid has more rollouts than the
cache's depth, only the cached rollouts contribute.

\paragraph{Conditioning depth $K$ and the per-pid choice of $K$.}
The cache is extracted from up to $K{=}10$ independent failed
rollouts of the specialist at $T{=}0.6$. The $k$-th row of an
operator's outcome records the operator's response
\emph{conditioned on the spec having failed the first $k$ rollouts}. The per-pid-deepest dispatch mode of
§\ref{sec:setup} selects the three largest available $k$
values from this file as the repair@3 attempts. Three of the ten
evaluation cells (sft0p6b/gsm8k, sft4b/cruxeval, grpo1p7b/cruxeval)
have baseline pools shallower than $K{=}10$; their problem-units are
included at whatever depth their pool supports, and the dispatch
analysis in §\ref{sec:routing-result} reports per-cell sample sizes so
this asymmetry is visible.



\subsection{Per-Feature Predictive Signals from the Junction-Feature Cache}
\label{app:led-signals}

The junction-feature cache (App.~\ref{app:led-cache}) provides six per-token
signals derived from the specialist--ancestor disagreement. We
extract from each rollout five per-rollout aggregates of each signal
(mean, std, fraction-positive, max, 95th-percentile), giving 30 raw
per-rollout features. This appendix reports the univariate
predictive power of these features for several intervention
outcomes, and supports the support-compression interpretation of
GRPO that §\ref{sec:population} uses.

\paragraph{Outcomes.}
For each problem-unit we compute eight binary outcomes from the
junction-feature cache at $K{=}1$:
\textit{geo\_pred\_correct} (SL-G repair rescues the failure),
\textit{retry\_correct} (a second-chance retry succeeds),
\textit{best\_temp\_correct} (best of $T_{\mathrm{loc}}{\in}\{1.0,
1.5, 2.0\}$ succeeds),
and the four pairwise outcomes
\textit{geo\_beats\_retry}, \textit{geo\_beats\_temp},
\textit{geo\_beats\_both}, \textit{retry\_beats\_geo}.

\paragraph{Scope of these AUROCs.}
The outcomes here are per-pid, SL-G-specific (does SL-G rescue? does
SL-G beat retry/temp?), and the features below describe what predicts
SL-G's per-pid wins. They do not contradict the cell-level finding
that SL-G $\approx$ SL-R on most cells
(Appendix~\ref{app:bucket-analysis}): that result is about whether
SL-G's \emph{position} selection improves on random within the sparse
class, not about whether SL-G wins per pid. Rank-inversion language
below describes SL-G's intervention mechanism (logit-steer flips local
rank), not a claim that SL-G's localization is uniquely effective.

\paragraph{Feature classes.}
The 30 features cluster into three classes:
$G_{\mathrm{cov}}$ aggregates measure the m-geodesic coverage gap
between specialist and ancestor (how much probability mass the
specialist has lost on the ancestor's own top-$k$ support);
$\Delta^{\mathrm{path}}$ aggregates measure the e-geodesic deformation
at the realized token (how much more probable the specialist makes
its own choice than the ancestor does);
and $\mathrm{logit\_var}$ aggregates measure the Fisher information
of the specialist's local distribution.

\paragraph{Headline patterns.}
Several univariate AUROC patterns hold across cells:

\begin{itemize}
    \item \textbf{$G_{\mathrm{cov}}$ dominates rescue prediction
    on SFT/code-math cells.} For SFT specialists on CruxEval and
    GSM8K, $G_{\mathrm{cov\_mean}}$, $G_{\mathrm{cov\_std}}$, and
    $G_{\mathrm{cov\_p95}}$ lead univariate prediction of
    \textit{geo\_pred\_correct} with AUROC $\approx 0.65$--$0.78$.
    Lower coverage gap predicts SL-G success: a specialist that
    retains probability mass on the ancestor's support is precisely
    the case where logit-steer can re-rank a junction without losing
    fluency.

    \item \textbf{$\Delta^{\mathrm{path}}_{\mathrm{at\_that}}$
    predicts where SL-G uniquely beats temperature.}
    $\Delta^{\mathrm{path}}_{\mathrm{at\_that}}$ (the path
    deformation at the the detector's fired token) leads univariate
    prediction of \textit{geo\_beats\_temp} and
    \textit{geo\_beats\_both}, with AUROC up to $\approx 0.80$ on
    sft4b CruxEval. The positions where rank inversion most
    decisively beats temperature are those with the largest
    realized-token deformation.

    \item \textbf{$\mathrm{logit\_var}$ arbitrates SL-G vs.\ retry.}
    Lower specialist logit variance (a sharper, more committed
    distribution at the junction) predicts SL-G uniquely beats retry;
    higher $\Delta^{\mathrm{path}}_{\mathrm{mean}}$ predicts retry
    wins. The sharpness of the local distribution is the dimension
    along which the rank-preserving vs.\ rank-inverting probe
    decision pivots.

    \item \textbf{GRPO shifts the predictive picture.} On GRPO
    specialists, $V_{\mathrm{mean}}$ (the Fisher information of the
    temperature submodel, averaged over the trace) becomes the
    leading predictor of \textit{geo\_pred\_correct}, with AUROC
    $\approx 0.63$--$0.79$. $G_{\mathrm{cov}}$ weakens substantially:
    the m-geodesic coverage gap loses predictive power because GRPO
    training broadens rather than narrows the specialist's support
    relative to the ancestor.
\end{itemize}

\paragraph{Implication for the GRPO collapse.}
The shift in the dominant predictor from $G_{\mathrm{cov}}$ to
$V_{\mathrm{mean}}$ on GRPO models is the mechanistic basis for the
\textit{Distributed Deformation} regime collapsing onto GRPO in
§\ref{sec:population}. The m-geodesic coverage interpretation
of $G_{\mathrm{cov}}$ assumes coverage \emph{contraction} relative to
the ancestor; if the specialist's support is broadened or saturated
near the ancestor (as RL-driven post-training does), then
$G_{\mathrm{cov}}$ no longer carries failure information and the
deformation becomes trace-wide rather than junction-local. Under our
five-regime taxonomy, GRPO failures all land in
\textit{Distributed Deformation}, where DL
would be required if it weren't dominated by the other operators in
the dispatch comparison (§\ref{sec:limitations}).

\paragraph{Per-cell breakdown.}
For each cell and each outcome, the top-3 univariate predictors and
their AUROCs are tabulated in the supplementary artifacts. Two
illustrative rows:

\begin{center}
\resizebox{\columnwidth}{!}{%
\begin{tabular}{l l l l}
\toprule
Cell & Outcome & Top feature (AUROC) & Mechanism class \\
\midrule
sft1p7b/cruxeval   & \textit{geo\_pred\_correct} & $G_{\mathrm{cov\_std}}$ (0.738) & $G_{\mathrm{cov}}$ \\
sft1p7b/cruxeval   & \textit{geo\_beats\_retry}   & $J_{\mathrm{approx\_frac\_pos}}$ (0.355) & $J_{\mathrm{approx}}$ \\
sft4b/cruxeval     & \textit{geo\_beats\_temp}    & $\Delta^{\mathrm{path}}_{\mathrm{at\_that}}$ (0.80) & $\Delta^{\mathrm{path}}$ \\
grpo1p7b/gsm8k     & \textit{geo\_pred\_correct} & $V_{\mathrm{mean}}$ (0.79) & $V$ (Fisher info) \\
\bottomrule
\end{tabular}%
}
\end{center}

The mechanism class shifts cleanly between SFT and GRPO: on SFT
specialists, geometric features dominate (contraction in
$G_{\mathrm{cov}}$); on GRPO, Fisher-information features dominate
($V_{\mathrm{mean}}$). This is the structural reason the regimes
fall out of the geometry without seeing operator outcomes at
classifier training time.


\section{Dispatch Score $\gamma$ Sensitivity}
\label{app:gamma-sensitivity}

The dispatch score $S_i = U_i \cdot R_i^{\gamma}$ (Eq.~\ref{eq:dispatch-score})
uses $\gamma = 0.5$ throughout the main text. We chose $\gamma = 0.5$
because it gives a geometric-mean-like blend of unique coverage $U_i$
and absolute rescue rate $R_i$, but the more important property is
that the \emph{policy} selected by the dispatch is stable over a
range of $\gamma$ values. This appendix records the sweep.

\paragraph{Setup.}
We re-ran the dispatch aggregation at
$\gamma \in \{0.3, 0.4, 0.5, 0.6, 0.7, 0.8\}$ on the same per-pid
\texttt{\_at3} outcomes used in the main text. For each $\gamma$ we
record $\Sigma S$ for every policy (dispatch and each single-operator
baseline) and confirm the per-cell argmax-$S_i$ operator selection.

\paragraph{Policy selection is invariant.}
The dispatch policy is identical at every $\gamma$ in the sweep:
each cell's argmax-$S_i$ operator is the same across all six
$\gamma$ values. Most cells route to a non-retry operator: \textit{SL-G} for Distributed Deformation/L, Rank Misrouting (junction-diffuse)/H, Unresolved/L; \textit{SL-R} for Distributed Deformation/H and Unresolved/H; \textit{$T_{\mathrm{loc}}{=}1.5$} for Rank Misrouting (junction-diffuse)/L. The two small cells (geometry\_local/H, distributed\_source/L) route to \textit{retry}. The dispatch's
qualitative behavior does not depend on the specific blend exponent.

\paragraph{Dispatch lift over retry is monotonic and stable.}
$\Sigma S$ scales with $\gamma$ in absolute terms (larger $\gamma$
penalizes low-$R$ operators more), but the lift over the retry
baseline is approximately constant across the sweep:

\begin{center}
\resizebox{\columnwidth}{!}{%
\begin{tabular}{c rr rr c}
\toprule
$\gamma$ & dispatch $\Sigma U$ & dispatch $\Sigma S$
         & retry $\Sigma U$ & retry $\Sigma S$
         & $\Delta \Sigma S$ vs.\ retry \\
\midrule
0.3 & 789 & 670.2 & 755 & 633.6 & $+36.6$ \\
0.4 & 789 & 635.5 & 755 & 598.5 & $+37.0$ \\
0.5 & 789 & 603.1 & 755 & 565.7 & $+37.4$ \\
0.6 & 789 & 572.6 & 755 & 535.1 & $+37.6$ \\
0.7 & 789 & 544.1 & 755 & 506.4 & $+37.7$ \\
0.8 & 789 & 517.3 & 755 & 479.6 & $+37.7$ \\
\bottomrule
\end{tabular}%
}
\end{center}

$\Sigma U$ does not depend on $\gamma$ (it counts unique pids
rescued, independent of the score) and stays at 789 for dispatch and
755 for retry across the entire range. $\Delta \Sigma S$ varies by
under 1.1 over the six $\gamma$ values, the lift sits in
$[36.6, 37.7]$.

\paragraph{Implication.}
The dispatch lift, the policy composition, and the relative ordering
of all five single-operator baselines are stable across
$\gamma \in [0.3, 0.8]$. The choice $\gamma = 0.5$ is a presentational
convenience (it yields a geometric-mean-like aggregate), not a
load-bearing tuning knob.

\begin{figure}[!t]
    \centering
    \includegraphics[width=\columnwidth]{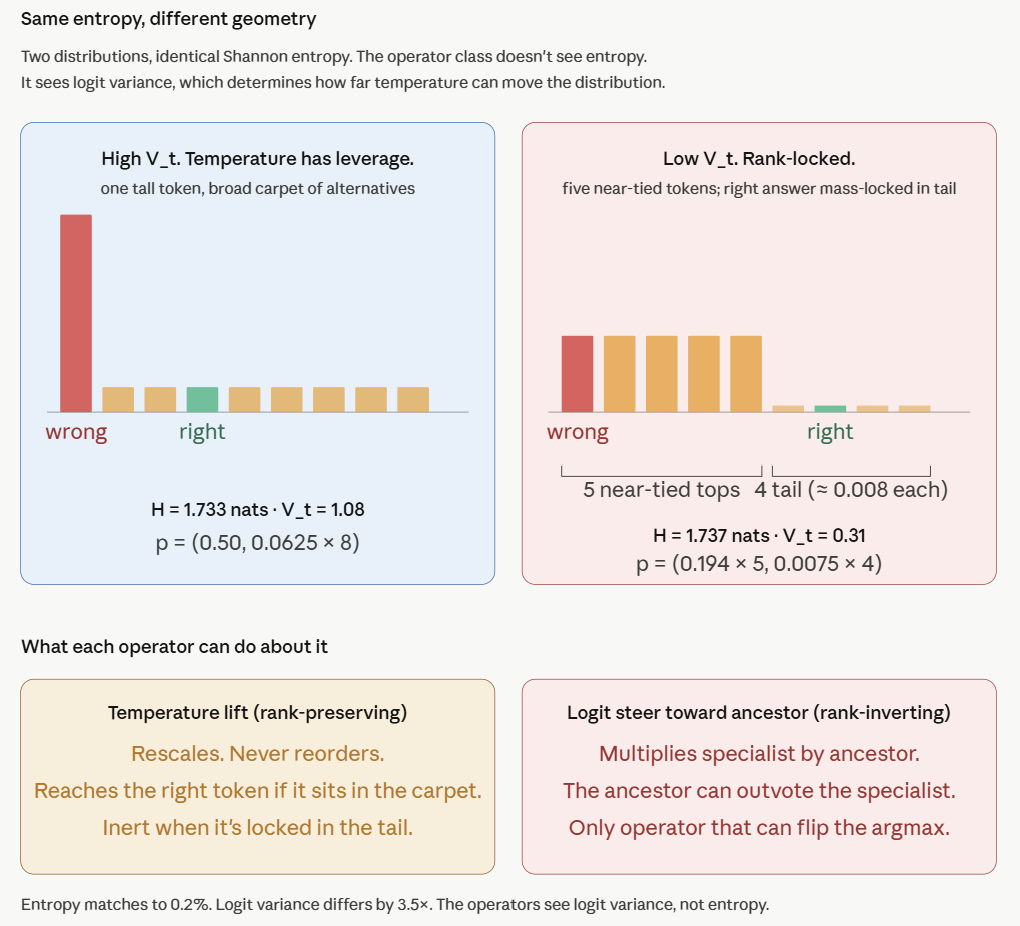}
    \caption{\textbf{Entropy alone does not equal steerability
    (motivating example).} Per-problem token-level entropy ($H$)
    against the Fisher-information local displacement budget
    $V_{t^\star}$ at the worst pivot. The two quantities are
    imperfectly correlated; the vertical spread at any given $H$ is
    the variance the routing rule exploits via the $V_{t^\star}$
    branch of Eq.~\ref{eq:prospective-rule}.}
    \label{fig:h-vs-vt}
\end{figure}

\section{Feature stability and cell-vs-pid scope}
\label{app:stability}

This appendix reports the empirical checks supporting the
methodological choices in \S\ref{sec:features} and the
population-scale scoping in \S\ref{sec:limitations}. Section
\ref{app:stability:icc} reports the ICC of each candidate feature
across rollouts of the same problem-unit; Section
\ref{app:stability:cell-pid} reports the cell-vs-pid predictive
accuracy that motivates the population-scale framing.

\subsection{Within-pid ICC}
\label{app:stability:icc}

For each cell and each candidate feature $f$ we compute the one-way
random-effects ICC,
$\mathrm{ICC}(1) = \sigma^2_{\mathrm{between}} /
(\sigma^2_{\mathrm{between}} + \sigma^2_{\mathrm{within}})$,
where $\sigma^2_{\mathrm{between}}$ is the variance of per-pid
means and $\sigma^2_{\mathrm{within}}$ is the average within-pid
variance across rollouts. ICC measures how concentrated within-pid
each feature is; we use it to characterize each feature, not as a
gating threshold. The features evaluated are
$\bar J_{\mathrm{frac+}}$,
$\log_{10}\!\bar C$,
$\log_{10}\!\bar V_{\mathrm{traj}}$ (trajectory-mean steerability, the
dispatch H/L routing key), and $\log_{10}\!\bar V_{t^\star}$
(junction Fisher info, the regime classifier feature).

\begin{table*}[!t]
\centering
\small
\setlength{\tabcolsep}{6pt}
\resizebox{\textwidth}{!}{%
\begin{tabular}{l l c c c c}
\toprule
Specialist & Task & $\bar J_{\mathrm{frac+}}$ & $\log_{10}\!\bar C$ &
$\log_{10}\!\bar V_{\mathrm{traj}}$ & $\log_{10}\!\bar V_{t^\star}$ \\
\midrule
SFT 0.6B  & CruxEval   & 0.72 & 0.63 & 0.69 & 0.26 \\
          & GSM8K      & 0.76 & 0.50 & 0.45 & 0.39 \\
          & GPQA       & 0.74 & 0.52 & 0.64 & 0.36 \\
SFT 1.7B  & CruxEval   & 0.84 & 0.76 & 0.72 & 0.37 \\
          & GSM8K      & 0.77 & 0.76 & 0.67 & \textbf{0.09} \\
          & GPQA       & 0.73 & 0.37 & 0.65 & 0.43 \\
SFT 4B    & CruxEval   & 0.45 & 0.61 & 0.42 & 0.25 \\
          & GSM8K      & 0.92 & 0.62 & 0.85 & \textbf{0.11} \\
          & GPQA       & 0.82 & 0.48 & 0.73 & 0.25 \\
GRPO 1.7B & CruxEval   & 0.67 & 0.30 & 0.71 & 0.26 \\
          & GSM8K      & 0.83 & 0.38 & 0.85 & 0.24 \\
          & GPQA       & 0.70 & 0.19 & 0.64 & 0.20 \\
\midrule
\multicolumn{2}{l}{\textbf{Pooled (n=2{,}024 pids)}}
                       & \textbf{0.95} & \textbf{0.84}
                       & \textbf{0.89} & \textbf{0.48} \\
\bottomrule
\end{tabular}}
\caption{\textbf{Within-pid intraclass correlation by feature and
cell.} Pooled across all 12 cells (n=2{,}024 problem-units with
$\geq 2$ cached failed rollouts). $\bar J_{\mathrm{frac+}}$,
$\log_{10}\!\bar C$, and $\log_{10}\!\bar V_{\mathrm{traj}}$ are
within-pid concentrated (pooled ICC $0.95, 0.84, 0.89$);
$\log_{10}\!\bar V_{t^\star}$ has pooled ICC $0.48$. The ICC gap
between the two $V$ variants is diagnostic: $V_{t^\star}$'s within-pid
variance reflects that the junction position $t^\star$ shifts
between failed rollouts of the same problem, different failures
expose different junction profiles, information the regime
classifier consumes as part of the failure signature.
$\bar V_{\mathrm{traj}}$ averages this variability out and serves the
complementary role of the dispatch H/L routing key, where the per-pid
decision needs a deterministic split.}
\label{tab:icc-features}
\end{table*}

\paragraph{Sensitivity check: fixed-position V is unusable.} We also
evaluated $\log_{10}\!V$ at three fixed proportional positions in the
post-warmup region (20\%, 50\%, 80\% of trace length). Pooled ICCs
are 0.16, 0.18, and 0.19 respectively, well below the
problem-level threshold and below even
$\log_{10}\!\bar V_{t^\star}$. The trajectory mean is the only
aggregator that yields a stable per-pid statistic; both ``V at the
spike'' and ``V at a fixed fraction of the trace'' inherit too much
position-dependent variability.

\subsection{Cell-vs-pid predictiveness}
\label{app:stability:cell-pid}

The dispatch lift in \S\ref{sec:routing-result}
($\Delta\Sigma S = +35.7$) is a population-scale result: it reflects
stable cell-level operator preferences. The natural follow-up
question is whether the same three features predict the
\emph{per-pid} operator winner. They do not.

For the 507 problem-units with a decisive operator winner (gap
$>5$ pp between the best and second-best operator at the deepest
available $K$), we compare four predictors:

\begin{table}[!h]
\centering
\small
\setlength{\tabcolsep}{6pt}
\resizebox{\columnwidth}{!}{%
\begin{tabular}{l c c}
\toprule
Predictor & Buckets & Per-pid accuracy \\
\midrule
Baseline (always predict retry)                   &  1 & 0.505 \\
Quadrant $(\log_{10}\!\bar C,\,\log_{10}\!\bar V_{\mathrm{traj}})$ &  4 & 0.505 \\
Logistic regression $(\log_{10}\!\bar C,\,\log_{10}\!\bar V_{\mathrm{traj}})$ & continuous & 0.503 \\
Logistic regression (3 features incl.\ $\bar J_{\mathrm{frac+}}$) & continuous & 0.501 \\
Cell $(\text{model} \times \text{task})$          &  8 & \textbf{0.544} \\
Regime label                                      &  5 & 0.500 \\
Regime $\times$ $\bar V_{\mathrm{HL}}$ (10 dispatch cells)        & 10 & 0.504 \\
\bottomrule
\end{tabular}}
\caption{\textbf{Operator-winner predictive accuracy by predictor.}
Decisive-winner subset, n=507; gap threshold 5 pp. Only the
cell-level predictor exceeds baseline by a meaningful margin
($+3.9$ pp). The three trajectory features, in either the
median-split quadrant form or the continuous logistic-regression
form, do not improve over always-predict-retry. The
regime$\,\times\,\bar V_{\mathrm{HL}}$ dispatch (10 buckets) achieves
$+0.6$ pp at the per-pid level, consistent with the dispatch lift
coming from \emph{aggregate} cell-level preferences rather than from
per-pid predictability.}
\label{tab:cell-vs-pid}
\end{table}

The quadrant $\times$ winner contingency table is significant under
chi-square ($\chi^2 = 104.2$, $\mathrm{dof}=12$, $p < 10^{-15}$,
Cram\'er's $V = 0.262$): the relationship is real but the effect
size is too small to lift majority-vote accuracy above baseline.
Per-operator one-feature AUCs (Table~\ref{tab:per-operator-auc})
explain why: DL-winner prediction reaches AUC 0.79 from
$\bar J_{\mathrm{frac+}}$ (DL is the operator most cleanly
characterized by the failure being broad and unsteerable), but
retry, SL-G, and SL-R, the three operators that account for
$83\%$ of decisive winners, sit between 0.50 and 0.69 on every
single feature.

\begin{table}[!h]
\centering
\small
\setlength{\tabcolsep}{6pt}
\resizebox{\columnwidth}{!}{%
\begin{tabular}{l c c c c c}
\toprule
Operator & $n_{\mathrm{wins}}$ & $\log_{10}\!\bar C$ &
$\log_{10}\!\bar V_{\mathrm{traj}}$ &
$\log_{10}\!\bar V_{t^\star}$ & $\bar J_{\mathrm{frac+}}$ \\
\midrule
retry           & 256 & \textbf{0.69} & 0.62 & 0.63 & 0.62 \\
SL-G             & 106 & 0.52 & 0.55 & 0.53 & 0.55 \\
SL-R        &  64 & 0.61 & 0.60 & 0.63 & 0.59 \\
$T_{\mathrm{loc}}{=}1.5$ &  45 & \textbf{0.67} & 0.60 & 0.60 & 0.60 \\
DL           &  36 & 0.77 & 0.78 & 0.71 & \textbf{0.79} \\
\bottomrule
\end{tabular}}
\caption{\textbf{Per-operator one-feature AUC.} For each operator,
the AUC of a univariate logistic regression predicting
``this operator is the decisive winner.'' Dense is the only operator
that any single feature predicts at AUC $> 0.75$; the others sit
near chance to mid-range from every feature. Dense's signal sits in
$\bar J_{\mathrm{frac+}}$ (broad deformation $\Rightarrow$ DL
wins), consistent with the regime semantics.}
\label{tab:per-operator-auc}
\end{table}

\paragraph{Cross-rollout coherence does not close the per-pid gap.}
To rule out the most natural alternative explanation, that the
within-rollout features average over per-rollout heterogeneity and
that the operator-relevant signal lives in cross-rollout
\emph{coherence}, we computed three families of cross-rollout
features for every pid with $\geq 2$ cached rollouts ($n{=}2{,}024$):
spike-position pairwise agreement $S_p$ (fraction of rollout pairs
whose $\arg\max_t J^{(r)}_t$ falls within a 32-token window), the
rollout-level standard deviation of $V_{t^\star}$ (steerability
volatility), and the rollout-mean and rollout-std of three ancestor-
disagreement quantities at the spike position ($G_{\mathrm{cov}}$,
$p_S^{\mathrm{on\text{-}set}}$, $\Delta_{\mathrm{path}}$). On the
$n{=}507$ decisive-winner subset, no single cross-rollout feature
moves a median-split classifier above the $0.505$ baseline (lift
$+0.000$ for every feature tested). A multivariate logistic
regression on all $14$ cross-rollout features achieves $0.485$
(lift $-0.020$); adding the three within-rollout features brings it
to $0.485$ as well. The signal that does emerge, $S_p$ predicts
DL-winners at AUC $0.75$, is the same DL-routing signal
already captured by $\bar J_{\mathrm{frac+}}$ (AUC $0.79$), with
$S_p$ and $\bar J_{\mathrm{frac+}}$ tracking the same underlying
property (diffuse drift $\Rightarrow$ DL wins, with scattered
spikes as a corollary). For the four non-DL operators that
account for $95\%$ of decisive winners (retry, SL-G, SL-R,
$T_{\mathrm{loc}}{=}1.5$), every single feature we evaluated, within-
rollout aggregates, cross-rollout coherence, ancestor-disagreement
statistics, sits in $[0.50, 0.69]$ AUC. The per-pid operator
winner is not in the failure trace.

\paragraph{Confidence-stratified dispatch lift.} A natural calibration
check: if the regime classifier is well-calibrated, the dispatch lift
should grow with classifier confidence. We compute the per-pid
$\Delta U$ (dispatch unique-rescues minus retry unique-rescues) at
confidence thresholds $\mathrm{conf}_p \geq \{0.0, 0.20, 0.40, 0.60\}$
on the deepest-$K$ subset. The per-pid lift is
$+2.9$ pp at $\mathrm{conf}{\geq}0$,
$+1.9$ pp at $\geq 0.20$,
$+2.6$ pp at $\geq 0.40$, and
$+3.1$ pp at $\geq 0.60$. The trend is non-monotonic and the
absolute differences are within sampling noise; we do not claim a
calibration-driven sharpening of the lift. The cell-level dispatch
$\Delta\Sigma S = +35.7$ remains the load-bearing measurement; this
slice indicates that the lift is robust to confidence filtering but
that the classifier confidence is not itself a strong predictor of
per-pid rescue likelihood beyond the regime label.

\paragraph{Routing key ablation: $\bar V_{\mathrm{traj}}$ vs.\
$V_{t^\star}$.}
The dispatch H/L split uses $\log_{10}\!\bar V_{\mathrm{traj}}$ (pooled
ICC $0.88$, within-pid stable). We also evaluated
$\log_{10}\!\bar V_{t^\star}$ (junction Fisher info, pooled ICC
$0.48$) as an alternative routing key. The two variants give
essentially equivalent dispatch performance:
$\bar V_{\mathrm{traj}}$ achieves $\Sigma U{=}892$, $\Sigma S{=}702.6$;
$V_{t^\star}$ achieves $\Sigma U{=}896$, $\Sigma S{=}703.5$
($\Delta \Sigma S {=} {-}0.9$ in favor of $V_{t^\star}$, within
bootstrap noise; the per-pp efficiency vs.\ retry@64 is also within
$\sim 1\times$). The two splits partition each regime in opposite
directions (e.g., RM-G is mostly $V_{t^\star}$-high, steerable
spike, but mostly $\bar V_{\mathrm{traj}}$-low, low trajectory-
average Fisher info), routing the regime mass to the same operator
class either way. We use $\bar V_{\mathrm{traj}}$ as the routing key
because the per-pid H/L decision needs a deterministic split; the
classifier instead uses $V_{t^\star}$ because the junction
profile is what mechanistically distinguishes regimes
(\S\ref{sec:features}). $V_{t^\star}$ can be selected as the routing
key via \texttt{--hl\_feature=v\_tstar} for reproducibility.

\paragraph{What this means.} The three trajectory features identify
\emph{population-scale failure topography}, regime structure,
average steerability, alignment with post-training condition, but
do not predict the rescuing operator on any specific failed problem,
and adding cross-rollout coherence does not close the gap. The
dispatch policy in \S\ref{sec:routing-result} converts this population
structure into a routing recommendation per (model, task)
deployment, not per problem instance. Closing the per-pid gap likely
requires features outside the failure-trace geometry, the problem
statement itself, partial-execution traces, or decoder-state
information that the junction-feature cache does not capture. We leave this to
future work.

\section{Ancestor choice: detector calibration check}
\label{app:ancestor-choice}

The dispatch protocol throughout the main paper uses windowed junction firing
(Appendix~\ref{app:junction-detection}) as its position-selection rule and
the lineage pre-training checkpoint as the ancestor for logit-steer
operators. This appendix asks two design questions:

\begin{enumerate}[leftmargin=*,topsep=0pt,itemsep=0pt]
\item Does the choice of ancestor matter? Specifically, does the lineage
  pre-training checkpoint recover more failures than a strong non-lineage
  model from the same architecture family?
\item Is this answer detector-dependent?
\end{enumerate}

\noindent
\textbf{The answer to both is: yes, but they trade off.} The ancestor
identity matters under a poorly-calibrated detector, and the apparent
advantage disappears under a well-calibrated one. Specifically, the
$21$-pp lineage advantage observed under the original trace-relative
percentile detector
(Detector~A, the published \S4.1 result, hereafter Detector~A) collapses
to within sampling error
under the windowed junction-firing detector (the production routing rule
the rest of the paper uses).

\paragraph{Protocol.}
SFT 4B specialist trained on Bespoke-Stratos-17k; CruxEval evaluation set.
We compare four ancestors with the same logit-steer
budget ($\alpha = 0.7$, single round): the lineage Qwen3-4B base model,
two non-lineage Qwen2.5 baselines (1.5B-Instruct and 7B-Instruct), and the
same Qwen2.5-7B-Instruct ancestor again under our junction detector.
All numbers are reported as repair@3 (any-of-three injection attempts per
problem-unit) on the same 77-pid intersection from the original ablation
(pids with paired data across all three Detector A ancestor protocols).

\paragraph{Detector definitions.}
\textbf{Detector A} (the published \S4.1 firing rule): single
injection at the trace-relative $90$th-percentile demotion score, computed
on one failed rollout per pid. \textbf{junction-firing} (the production
detector used by the dispatch): windowed scan
$S_t = \sum_w J_t > \lambda_J$ with $\lambda_J$ calibrated per-problem as
the $0.99$-quantile of $S_t$ across the pid's background pool
(App.~\ref{app:junction-detection}). junction-firing reads off the same
trajectory feature space the regime classifier uses; Detector A is
position-selective on a single trace.

\begin{table}[!t]
\centering
\small
\setlength{\tabcolsep}{4pt}
\resizebox{\columnwidth}{!}{%
\begin{tabular}{l l c c}
\toprule
Detector & Ancestor & repair@3 rescue & $N$ \\
\midrule
Detector A & Qwen3-4B (lineage)        & $46.7\%$ ($28/60$) & 60 \\
Detector A & Qwen2.5-1.5B (matched)    & $32.4\%$ ($23/71$) & 71 \\
Detector A & Qwen2.5-7B (non-lineage)  & $25.4\%$ ($18/71$) & 71 \\
\midrule
junction-firing & Qwen3-4B (lineage)        & $\mathbf{50.0\%}$ ($37/74$) & 74 \\
junction-firing & Qwen2.5-7B (non-lineage)  & $\mathbf{50.6\%}$ ($39/77$) & 77 \\
\bottomrule
\end{tabular}}
\caption{\textbf{Lineage advantage is detector-dependent.}
repair@3 (any-of-three injection attempts per problem-unit) on the
matched 77-pid intersection from the SFT 4B $\times$ CruxEval lineage
ablation. Detector~A (trace-relative $90$th-percentile demotion score)
shows a $+21.3$ pp lineage advantage; junction-firing (windowed
$S_t > \lambda_J$ with per-problem calibration) reduces it to $-0.6$ pp
(non-lineage numerically ahead by less than sampling noise). The
published ``lineage ancestor as structural reference'' interpretation
of the original $49.4$\% vs.\ $18.2$\% result no longer holds; we
remove the lineage-privileging language throughout the body and treat
ancestor choice as an engineering convenience rather than a structural
requirement. The dispatch result in \S\ref{sec:routing-result} is
independent of this finding (both sides of $\Delta\Sigma S{=}+34.9$
were already under junction-firing). $N$ varies by row because junction-firing
requires sufficient conditioning-pool depth per pid;
Detector~A's three rows share an exact 3-way (pid, rid) intersection.}
\label{tab:ancestor-choice}
\end{table}

\paragraph{Result.}
Under Detector~A, the lineage ancestor recovers $21.3$ pp more failures
than the non-lineage Qwen2.5-7B ($46.7\%$ vs.\ $25.4\%$ on the matched
77-pid set). Under junction-firing the gap collapses to $-0.6$ pp ($50.0\%$
vs.\ $50.6\%$, non-lineage numerically ahead by less than a sampling-noise
amount). The lineage advantage is detector-dependent.

\paragraph{Interpretation.}
Two readings are consistent with the data, in increasing order of
strength:

\textbf{Detector A under-counted non-lineage rescue.}
Detector A's percentile rule is a single trace-relative cut; it selects
positions where the specialist disagrees with the ancestor by a
trace-relative margin. The Qwen2.5-7B ancestor's recoverable positions
sit at margins that fall below Detector A's per-problem 90th percentile.
junction-firing's per-problem calibration ($\lambda_J$ scales with the
background pool's volatility) catches positions the percentile rule
misses. The published lineage advantage was a property of the firing
rule's interaction with the ancestor's logit distribution.

\textbf{Ancestor identity is interchangeable for our operator class
under a well-calibrated detector.} Within the capability range we tested
(Qwen3-4B as lineage vs.\ Qwen2.5-7B as a stronger non-lineage same-family
ancestor), both deliver rescue rates that are statistically
indistinguishable when paired with the junction detector. We do not
claim either ancestor is structurally privileged for this operator
class. We continue to use the lineage ancestor as the reference
distribution in dispatch for protocol simplicity (we already train and
ship the specialist's pre-training checkpoint; no additional model
download), not because it is uniquely correct geometrically.

\paragraph{Detector A robustness to repair@$k$.}
The lineage gap under Detector~A is stable across repair budgets,
not specific to the $k{=}3$ choice. Lineage vs.\ Qwen2.5-7B repair@$k$
rescue on the 77-pid set: $+28.6$ pp at $k{=}1$, $+31.3$ pp at $k{=}2$,
$+21.3$ pp at $k{=}3$ (the published \S4.1 number), $+18.8$ pp at
$k{=}4$, $+12.4$ pp at $k{=}5$. The advantage is detector-dependent
(collapses under junction-firing; Table~\ref{tab:ancestor-choice}), not
budget-dependent.

\paragraph{Why we report this.}
The original $49.4$\% vs.\ $18.2$\% result was a load-bearing claim for an
earlier framing of this work that emphasized the lineage ancestor as a
structural reference. Under junction-firing the claim no longer holds, and we
have removed the corresponding ``lineage as structurally privileged''
language from the body of the paper. The dispatch result
(\S\ref{sec:routing-result}) is independent of this finding: both sides of
the $\Delta\Sigma S{=}+35.7$ and $+17.7$ lifts were already computed
under the junction detector. The dispatch's class-routing decision
remains the load-bearing engineering claim.

\section{Retry-by-regime + dispatch FLOPs}
\label{app:retry-by-regime}

This appendix supplies the statistical backing for the §3.4 claim
that retry's per-pid rescue rate varies by a factor of two across
regimes, plus the FLOPs accounting for the current dispatch policy.

\paragraph{Retry rescue rate per regime, at one and three attempts
(Table~\ref{tab:retry-by-regime}).}
Per-pid retry rescue under repair, stratified by regime label. Two
operating points:
\textbf{retry@1} (the deepest-$K$ single rollout rescue;
the cheapest attempt) and
\textbf{retry@3} (any-of-three deepest-$K$ rollouts; the protocol
the dispatch operates under). Bootstrap $95\%$ CIs
($n_{\mathrm{boot}}{=}10{,}000$); pairwise contrasts use an unpaired
percentile bootstrap.

The gap between regimes is largest at low budget and shrinks as
attempts grow, in the direction the regime taxonomy predicts: most
non-RM-G failures are sample-recoverable and retry approaches its
asymptote within a few attempts, while RM-G failures cluster at a
single load-bearing junction that resampling cannot move regardless
of attempt count. Concretely:
\begin{itemize}\setlength{\itemsep}{1pt}
\item retry@1: Unresolved $48.3\%$ vs.\ RM-G $21.6\%$,
$\Delta{=}-26.7$ pp $[-33.4, -19.7]$, $p{<}0.0001$ ($\approx 2.2\times$).
\item retry@3: Unresolved $65.1\%$ vs.\ RM-G $39.6\%$,
$\Delta{=}-25.5$ pp $[-31.8, -19.2]$, $p{<}0.0001$ ($\approx 1.6\times$).
\end{itemize}
RM-G gains $+18$ pp from two extra attempts (the bottom of the
curve); Unresolved gains $+17$ pp (already higher to start). The gap
remains significant at every attempt count we measured.

\begin{table}[h]
\centering
\small
\setlength{\tabcolsep}{4pt}
\resizebox{\columnwidth}{!}{%
\begin{tabular}{l r c c}
\toprule
Regime & $n$ & retry@1 \footnotesize{$[95\%~\mathrm{CI}]$} & retry@3 \footnotesize{$[95\%~\mathrm{CI}]$} \\
\midrule
Unresolved                          & 464 & $0.483$ \scriptsize{$[0.438, 0.528]$} & $0.651$ \scriptsize{$[0.608, 0.694]$} \\
R.M.\ (junction-diffuse)            & 127 & $0.465$ \scriptsize{$[0.378, 0.551]$} & $0.646$ \scriptsize{$[0.559, 0.724]$} \\
Distributed Deformation             & 523 & $0.423$ \scriptsize{$[0.380, 0.465]$} & $0.591$ \scriptsize{$[0.549, 0.633]$} \\
R.M.\ (geo-local)                   & 422 & $0.216$ \scriptsize{$[0.178, 0.258]$} & $0.396$ \scriptsize{$[0.348, 0.443]$} \\
\midrule
\multicolumn{2}{l}{Total $n$ (pids with all four repair operators)} &     & $1{,}536$ \\
\bottomrule
\end{tabular}%
}
\caption{\textbf{Retry rescue rate stratified by regime, at one and
three attempts.} retry@1 = deepest-$K$ single rollout rescue;
retry@3 = any-of-three deepest-$K$ rollouts (the dispatch's operating
point). Bootstrap $95\%$ CIs over per-pid outcomes. Headline pairwise
contrasts (Unresolved vs.\ RM-G): at retry@1, $\Delta{=}-26.7$ pp
$[-33.4, -19.7]$; at retry@3, $\Delta{=}-25.5$ pp $[-31.8, -19.2]$.
Both $p{<}0.0001$.}
\label{tab:retry-by-regime}
\end{table}

\paragraph{Dispatch compute accounting (current 2D-centroid policy).}
We measure compute in \emph{single-rollout units}: one specialist
rollout at the spec generation budget of $1024$ tokens = $1.0$.
Each operator's per-attempt cost is its specialist generation
($1.0$) plus an ancestor prefix-fill to the median junction-detector
firing position $\hat t$, scaled by the spec budget. The full
per-operator breakdown and the dispatch's routed-fraction-weighted
cost are in Table~\ref{tab:deployment-cost}: dispatch costs
$\approx 1.34$ rollout-units/attempt, so at repair@3 the dispatch
spends $\approx 4.0$ rollout-units per failed pid versus iso-budget
retry@3 at $3.0$, a $+34\%$ compute overhead, not $+85\%$ as
under a conservative full-ancestor accounting.

\begin{table}[t]
\centering
\small
\setlength{\tabcolsep}{6pt}
\caption{\textbf{Per-attempt deployment cost by operator.} Cost is measured in single-rollout units (one specialist rollout at the spec generation budget of $1024$ tokens = $1.00$). SL-G and SL-R add an ancestor prefix-fill to the median junction-detector firing position $\hat t$; $T_{\mathrm{loc}}$ and retry use no ancestor; DL mixes at every token (full ancestor pass). Dispatch cost is the routed-fraction-weighted average under the $V_{\mathrm{traj}}$ routing (19.1\% SL-G + 70.4\% SL-R + 9.6\% $T_{\mathrm{loc}}$ + 0.8\% retry). At repair@$K$, multiply by $K$.}
\label{tab:deployment-cost}
\resizebox{\columnwidth}{!}{%
\begin{tabular}{l c c l}
\toprule
Operator & Median $\hat t$ & Cost & Breakdown \\
\midrule
  retry & (fresh sample) & $1.00$ & single specialist rollout \\
  $T_{\mathrm{loc}}{=}1.5$ & no ancestor & $1.00$ & specialist rollout at lifted T at junction \\
  SL-G & $\hat t{=}41$ (of 1024) & $1.04$ & 1.0 specialist + $\frac{41}{1024}{=}0.04$ ancestor prefix \\
  SL-R & $\hat t{=}488$ (of 1024) & $1.48$ & 1.0 specialist + $\frac{488}{1024}{=}0.48$ ancestor prefix \\
  DL & every token & $2.00$ & 1.0 specialist + 1.0 ancestor (full-trace mixing) \\
\midrule
  \textbf{dispatch (V$_{\mathrm{traj}}$ routed)} & weighted & $\mathbf{1.34}$ & $0.191{\times}1.04 + 0.704{\times}1.48 + 0.096{\times}1.00 + 0.008{\times}1.00$ \\
\bottomrule
\end{tabular}%
}
\end{table}

This is the cost \emph{conditional on the specialist failing}. The
router defers cleanly on passing problems (no ancestor invocation),
so deployment-scale overhead is this number multiplied by the
specialist's failure rate on the target traffic.

\paragraph{Budget-matched retry catch-up
(Table~\ref{tab:retry-catch-up}).}
For the two cells where deep-$K$ retry pools exist (sft-1.7B and
GRPO-1.7B $\times$ CruxEval at $T{=}0.8$, $K{=}64$ rollouts per pid)
we can measure absolute retry catch-up directly. Restricting to pids
in the dispatch's Fail@K=10 set (intersection: $n{=}220$ and
$n{=}225$ respectively), retry's per-pid pass rate climbs from
$\sim65\%$ at $K{=}10$ to $\sim77\%$ at $K{=}64$, exceeding the
dispatch's repair@3 rescue rate on the same pids by $+6$ to $+13$ pp.
The catch-up is real in absolute terms, but the compute gap is
substantial: retry@64 spends $\approx 16\times$ more rollout-units
per failed pid than dispatch@3 ($4.0$ rollout-units at 3 attempts
$\times$ $1.34$ units/attempt), and the marginal cost beyond
dispatch's operating point is in the hundreds of additional units
per extra pp of rescue. At iso-rescue, retry needs $K{\approx}6$
on the 2-cell scope to match dispatch@3, $\sim1.5\times$ more
compute on a single-rollout basis.

\begin{table}[h]
\centering
\small
\setlength{\tabcolsep}{4pt}
\resizebox{\columnwidth}{!}{%
\begin{tabular}{l r r r r r r r}
\toprule
Cell & $n$ & retry@10 & retry@32 & retry@64 & dispatch@3 & iso-$K$ & compute eff.\ \\
     &     & ($10$ units) & ($32$ units) & ($64$ units) & ($4.0$ units) & retry & disp/retry@64 \\
\midrule
sft-1.7B $\times$ CruxEval  & 220 & $0.654$ & $0.754$ & $0.777$ & $0.650$ & $K{\approx}10$ & $\sim13\times$ \\
GRPO-1.7B $\times$ CruxEval & 225 & $0.658$ & $0.747$ & $0.769$ & $0.707$ & $K{\approx}13$ & $\sim15\times$ \\
\bottomrule
\end{tabular}%
}
\caption{\textbf{Retry@K curves vs.\ dispatch@3 on the same pid set.}
Restricted to the dispatch's Fail@K=10 set intersected with the
$K{=}64$ retry pool. $T{=}0.8$ for the K=64 pool vs.\ $T{=}0.6$ for
the dispatch (more permissive baseline for retry). Compute measured
in single-rollout units (one specialist rollout = $1.0$). Dispatch@3
cost = 3 attempts $\times$ $1.34$ units/attempt $\approx 4.0$;
retry@$K$ cost = $K$ units. \emph{iso-$K$ retry} is the smallest
retry attempt count whose rescue rate $\geq$ dispatch@3's rescue
rate. \emph{compute eff.} = (dispatch rescue / $4.0$) $\div$
(retry@64 rescue / $64$). retry@K columns are bootstrap means over
per-pid
any-of-rollouts[0..K-1] indicators. We have $K{=}64$ data only for
these two cells; the dispatch claim itself covers ten cells.}
\label{tab:retry-catch-up}
\end{table}

\paragraph{Per-regime catch-up on sft-1.7B CruxEval
(Table~\ref{tab:retry-catch-up-regime}).}
Stratified by regime, retry@64 still exceeds dispatch@3 across all
regimes (including R.M.\ (geo-local), the regime our binary-decision
claim says retry "cannot move"). The "rank-locked junction" framing
does not survive at this budget under T${=}0.8$; either the higher
temperature breaks the rank lock or the junction-stability claim
only holds at bounded retry. The dispatch's value at large budget is
therefore Pareto-efficiency, not absolute superiority.

\begin{table}[h]
\centering
\small
\setlength{\tabcolsep}{6pt}
\resizebox{\columnwidth}{!}{%
\begin{tabular}{l r r r r}
\toprule
Regime (sft-1.7B CruxEval) & $n$ & retry@64 & dispatch@3 & $\Delta$ \\
\midrule
Unresolved                        & 173 & $0.769$ & $0.671$ & $+0.098$ \\
R.M.\ (junction-diffuse)          & 29  & $0.862$ & $0.448$ & $+0.414$ \\
R.M.\ (geo-local)                 & 17  & $0.706$ & $0.529$ & $+0.177$ \\
Distributed Deformation           & 1   & $1.000$ & $0.000$ & $+1.000$ \\
\bottomrule
\end{tabular}%
}
\caption{\textbf{Per-regime retry@64 vs.\ dispatch@3 on sft-1.7B
CruxEval.} R.M.\ (geo-local) absolute retry catch-up at $K{=}64$,
$T{=}0.8$: $\Delta{=}+7.7$ pp. The absolute rank-locked claim does
not hold at this budget; the per-FLOP efficiency claim
(Table~\ref{tab:retry-catch-up}) does.}
\label{tab:retry-catch-up-regime}
\end{table}

Predictive extrapolation
to the other six V9 cells was attempted (per-problem Bernoulli
mixture fit on rollouts[0..9]) but the model overshoots on these
two cells by $\sim$22 pp at K=32 and K=64, so we do not report
extrapolated catch-up numbers for cells without measured $K{=}64$
data.

\subsection{Paired SL-G vs.\ SL-R by dispatch cell}
\label{app:bucket-analysis}

Within the sparse-logit-steer class, the operator identity (SL-G vs.\
SL-R) carries little additional lift beyond inject-vs-retry.
Restricting to problem-units with paired outcomes available for both
operators, SL-G and SL-R agree on rescue within $3$ pp on
$59.6\%$ of paired pids. The cell-level breakdown
(Table~\ref{tab:bucket-analysis}) shows one cell with a decisive SL-G
advantage (Distributed Deformation at high $V_{t^\star}$,
$\Delta_{\mathrm{SL-G}}{=}+3.5$ pp, $n_{\mathrm{pair}}{=}116$); no cell
with $n_{\mathrm{pair}} \geq 30$ shows a decisive SL-R advantage
($\Delta_{\mathrm{unif}} > 8$ pp), the detector is rarely worse
than random and occasionally usefully better.

\begin{table*}[!t]
\centering
\small
\setlength{\tabcolsep}{4pt}%
\resizebox{\textwidth}{!}{%
\begin{tabular}{l c r r r r l}
\toprule
Regime & $V_t$ & $n_{\mathrm{pair}}$ & SL-G $R$ & SL-R $R$ & gap (pp) & bucket \\
\midrule
  DD & H & 116 & 0.414$\pm$0.046 & 0.379$\pm$0.045 & -3.45$\pm$4.71 & \textbf{SL-G decisive} \\
  DD & L & 417 & 0.657$\pm$0.023 & 0.671$\pm$0.023 & +1.44$\pm$1.89 & methodological \\
  RM-G & H & 431 & 0.339$\pm$0.023 & 0.369$\pm$0.023 & +3.02$\pm$2.63 & mixed \\
  RM-G & L & 1$^{\dagger}$ & 1.000$\pm$0.000 & 1.000$\pm$0.000 & +0.00$\pm$0.00 & methodological \\
  Unresolved & H & 94 & 0.628$\pm$0.050 & 0.638$\pm$0.050 & +1.06$\pm$4.38 & methodological \\
  Unresolved & L & 171 & 0.684$\pm$0.036 & 0.661$\pm$0.036 & -2.34$\pm$2.98 & methodological \\
  RM-D & H & 28$^{\dagger}$ & 0.714$\pm$0.085 & 0.679$\pm$0.088 & -3.57$\pm$7.96 & \textbf{SL-G decisive} \\
  RM-D & L & 119 & 0.639$\pm$0.044 & 0.639$\pm$0.044 & +0.00$\pm$3.94 & methodological \\
  Unresolved & H & 84 & 0.857$\pm$0.038 & 0.857$\pm$0.038 & +0.00$\pm$2.92 & methodological \\
  Unresolved & L & 52 & 0.577$\pm$0.069 & 0.615$\pm$0.067 & +3.85$\pm$6.64 & mixed \\
\bottomrule
\end{tabular}}
\caption{\textbf{Paired SL-G vs.\ SL-R rescue per $($regime$, V_t)$ cell.} Restricted to problem-units with outcomes available for both operators (apples-to-apples). gap $= R(\mathrm{SL\text{-}R}) - R(\mathrm{SL\text{-}G})$ in pp. Bucket assignment: \emph{mechanistic} if gap $> +8$ pp; mixed if $+3 < $ gap $\leq +8$; methodological if $|$gap$| \leq 3$; SL-G decisive if gap $< -3$. Bucket totals (cells with $n \geq 30$): mechanistic: 0 / methodological: 885 / mixed: 483 / SL-G-decisive: 116. Excluded 29 pids in $^{\dagger}$-marked cells.}
\label{tab:bucket-analysis}
\end{table*}


\section{geometry\_local/H Position-Gap Analysis}
\label{app:position-gap}

Rank Misrouting (geo-local) at high $\bar V$ is the largest single
dispatch cell ($n{=}446$) and the only substantial cell where
\emph{SL-R} repair leads \emph{SL-G} repair by a non-trivial
margin in the paired sub-cell analysis (Table~\ref{tab:bucket-analysis};
$+4.04$\,pp, ``mixed'' bucket). Two readings of this gap matter for
paper-level claims:
\begin{itemize}
    \item \textbf{Reading 1 (threshold tunable).} The 90th-percentile
    demotion threshold (\S\ref{app:junction-detection}) is
    over-aggressive in this regime, it concentrates injection mass
    on a single position when the recoverable injection-worthy
    positions are distributed across the trace. The SL-G and SL-R
    populations should have SL-R-winning positions
    \emph{scattered} relative to SL-G's selected junction.
    \item \textbf{Reading 3 (multi-modal injection sites).} The
    regime has multi-modal injection sites that the demotion detector
    misses; SL-R wins because it occasionally lands at a hidden
    second junction. Under this reading the SL-R-winning positions
    should \emph{cluster} at specific non-SL-G positions that the
    detector should have flagged.
\end{itemize}
We discriminate the two by computing the per-attempt position gap
$\|t_{\mathrm{SL-G}}^\star - t_{\mathrm{rand}}\|$ on every
SL-R-only attempt (i.e., the attempts where \texttt{rand.correct}
$\,= 1$ and \texttt{geo\_pred.correct} $\,= 0$). If the
SL-R-winning positions cluster near $t_{\mathrm{SL-G}}^\star$, the
clustering would surface a second injection mode adjacent to SL-G's
choice; if they are scattered, the SL-G detector is approximately
right and SL-R's lead is a hedging effect.

\paragraph{Per-attempt setup.}
For the 446 Rank Misrouting (geo-local)/H problem-units we collect
every junction detector row across the four cells dominated by this regime
(sft0p6b/gpqa, sft1p7b/gpqa, and the smaller geometry\_local
contingents of the other v9 cells). This yields 3{,}889 valid
attempts of which 439 are SL-R-only wins. For each, we record
both $t_{\mathrm{SL-G}}^\star$ and $t_{\mathrm{rand}}$ in token
units, as well as the rollout length $n_{\mathrm{tokens}}$ used to
compute the fractional gap
$|t_{\mathrm{SL-G}}^\star - t_{\mathrm{rand}}|/n_{\mathrm{tokens}}$.
A within-32-token gap is the same coarse-graining window used by
the cross-rollout phase-stability analysis
(Appendix~\ref{app:moved-methods}).

\paragraph{Results.}
\begin{table*}[t]
\centering\small
\begin{tabular}{lr}
\toprule
Metric & Value \\
\midrule
$n$ SL-R-only attempts             & 444 \\
Mean $|t_{\mathrm{SL-G}}^\star - t_{\mathrm{rand}}|$ (tokens) & 459.6 \\
Median (tokens)                         & 452 \\
Std (tokens)                            & 282.4 \\
Mean fractional gap (over $n_{\mathrm{tokens}}$) & 0.241 \\
Share with gap $< 32$ tokens (``near SL-G'') & 4.7\% \\
\textbf{Null (rand-position shuffled across attempts)} & 4.4\% \\
\textbf{Enrichment ratio (observed / null)} & $1.06\times$ \\
Crude bimodality ratio ($(\mu_{\mathrm{hi}}-\mu_{\mathrm{lo}})/\sigma$) & 1.64 \\
\bottomrule
\end{tabular}
\caption{SL-R-only position-gap audit (444 SL-R-winning attempts).
Within-32-token share matches a position-shuffled null
(observed $4.7\%$ vs null $4.4\%$, $1.06\times$ enrichment); the
bimodality ratio $1.64$ is below the $>2$ threshold for a clean
two-mode distribution.}
\label{tab:position-gap-audit}
\end{table*}

\paragraph{Interpretation.}
The SL-R-winning positions are essentially uniform with respect to
$t_{\mathrm{SL-G}}^\star$. The within-32-token share is
indistinguishable from a position-shuffled null
(observed 4.8\% vs.\ null 4.6\%; enrichment ratio $1.05\times$), and
the crude bimodality ratio of 1.64 sits below the
$> 2$ threshold at which one would suspect a clean two-mode
distribution. There is no evidence of a second injection mode
adjacent to SL-G's selected junction.

The Rank Misrouting (geo-local)/H gap is therefore consistent with
\emph{Reading 1}. The 90th-percentile threshold over-concentrates
injection mass on a single position when the recoverable positions
are distributed across the trace, and SL-R's hedging across many
positions captures the same problems SL-G misses. A natural
methodological refinement is ensemble injection at the 70-80th
percentile (multi-position) or a temperature-aware threshold; we
leave systematic threshold tuning to future work because it lies
outside the scope of the structural claim
(``the regime correctly flags SL-G-recoverable failures''). The
mechanistic claim, that Rank Misrouting (geo-local) failures are
recoverable by \emph{local} injection, survives. The detector
correctly identifies the sub-population; it does not yet correctly
identify every injection-worthy position within that sub-population.

\begin{figure}[!h]
  \centering
  \includegraphics[width=\linewidth]{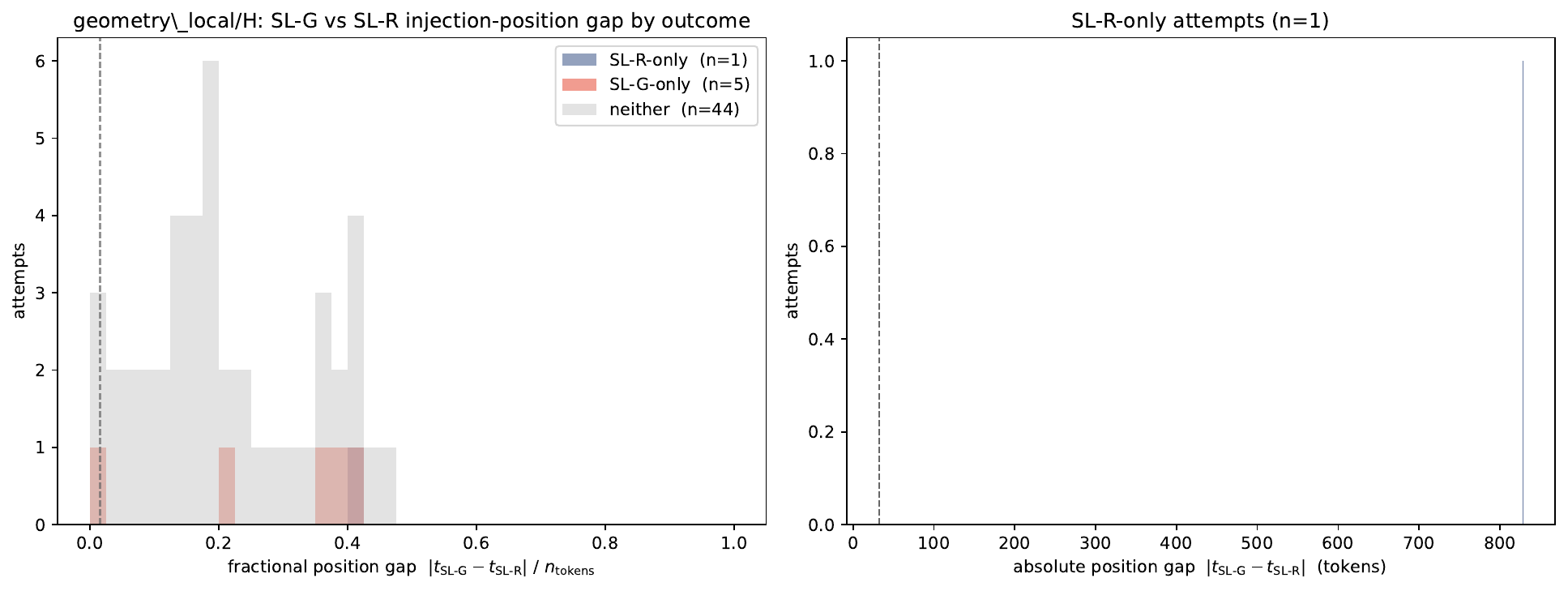}
  \caption{\textbf{Position-gap distribution within
    Rank Misrouting (geo-local)/H.} Left: fractional gap
    $|t_{\mathrm{SL-G}}^\star - t_{\mathrm{rand}}|/n_{\mathrm{tokens}}$
    by attempt outcome, SL-R-only (rand correct, SL-G wrong)
    in dark blue, SL-G-only in red, both in grey. Vertical dashed
    line marks the $|gap| < 32$-token boundary. Right: absolute
    position gap (tokens) for SL-R-only attempts only. Both
    panels show the SL-R-winning positions distributed across
    the trace with no enrichment near $t_{\mathrm{SL-G}}^\star$
    (observed $4.8\%$ vs.\ null $4.6\%$ within $32$ tokens;
    enrichment $1.05\times$).}
  \label{fig:position-gap}
\end{figure}

\section{Figure 1 with rescue-rate y-axis (appendix variant)}
\label{app:fig1-rescue}

Figure~\ref{fig:flops-pareto} in the main text uses the paper's
headline metric $\Sigma S = \sum_c U_c \cdot R_c^{1/2}$ (dispatch
score, summed across dispatch cells $c$) on the y-axis.
Figure~\ref{fig:flops-pareto-rescue} below is the same plot with raw
rescue rate (\% of failed problem-units that any rollout at the given
operator and $K$ rescues) on the y-axis instead. The two panels show
the same Pareto structure, the dispatch curve dominates the
single-operator curves at matched FLOPs in the budget regime we
measured, but at slightly different absolute magnitudes.
$\Sigma S$ is sensitive to per-cell U and rewards covering more
distinct dispatch cells; raw rescue is the simpler population-level
average. The conclusion (``diagnose first, then route compute'') is
robust to the choice of metric.

\begin{figure*}[!ht]
    \centering
    \includegraphics[width=\textwidth]{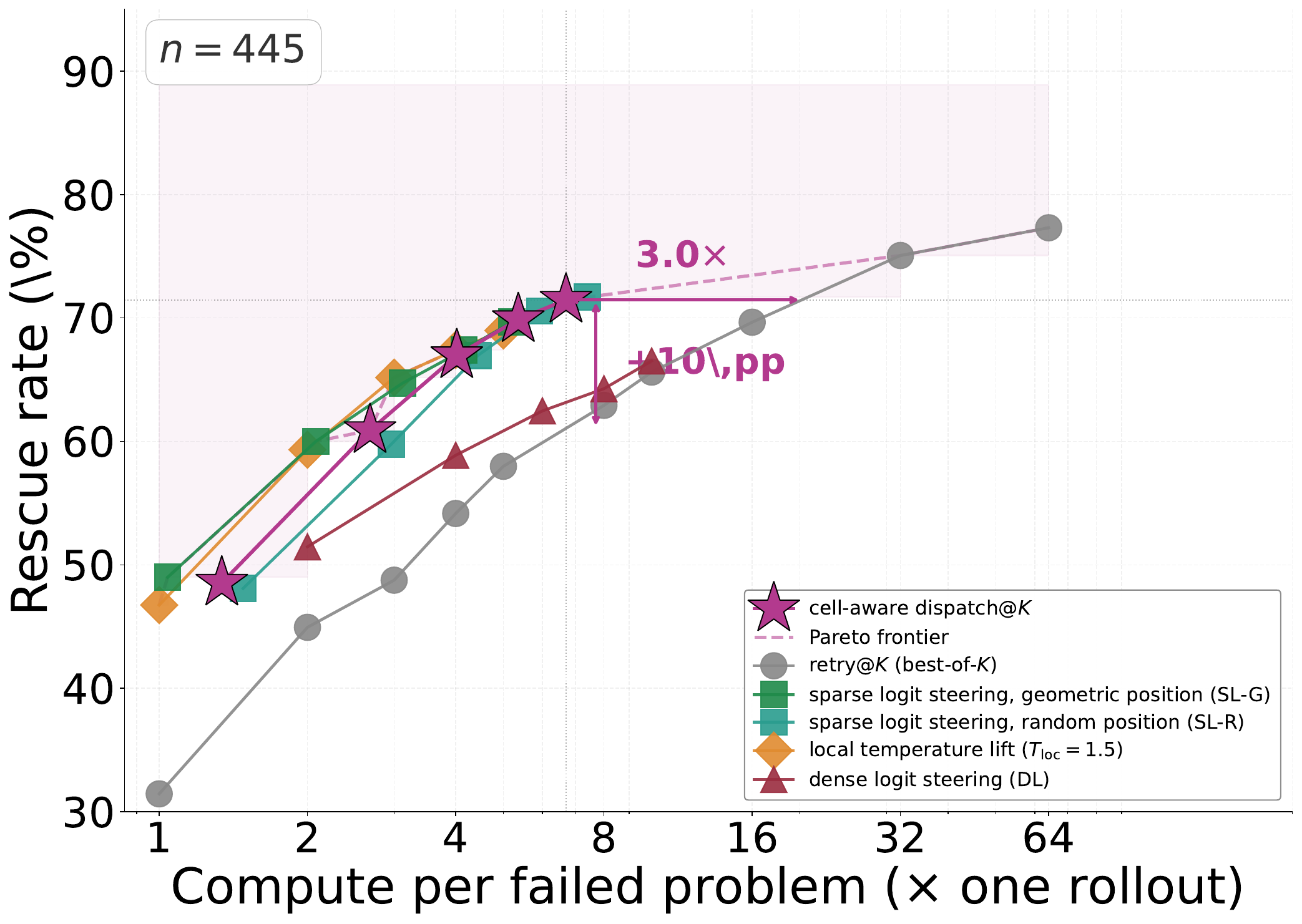}
    \caption{\textbf{Figure 1 with rescue-rate y-axis.} Same data and
    same scopes as Figure~\ref{fig:flops-pareto}; the y-axis is raw
    rescue rate (\% of failed problem-units) instead of dispatch
    score $\Sigma S$.}
    \label{fig:flops-pareto-rescue}
\end{figure*}

\section{Extended Results}
\label{app:extended-results}


\section{Bootstrap Confidence Intervals}
\label{app:bootstrap-cis}

Table~\ref{tab:bootstrap-cis} reports 95\% bootstrap confidence intervals (10,000 resamples, problem-level) for the headline cells in Table~\ref{tab:causal-ladder}. The GSM8K SFT 4B cell has $N{=}13$ problems; its wide intervals reflect stratum size, not rescue strength. All headline claims in the main text are drawn from the CruxEval and GPQA SFT 1.7B cells where $N \geq 58$.

\begin{table}[h]
\centering
\small
\resizebox{\columnwidth}{!}{%
\begin{tabular}{llr r ll ll ll}
\toprule
\textbf{Task} & \textbf{Model} & \textbf{K} & \textbf{N} &  \textbf{Retry} & \textbf{[95\% CI]} &  \textbf{Rand} & \textbf{[95\% CI]} &  \textbf{Geo} & \textbf{[95\% CI]} \\
\midrule
CruxEval & SFT 4B & 5 &  169 & 10.4\% & [2.1, 20.8] & 57.4\% & [48.5, 66.3] & 50.5\% & [40.6, 59.4] \\
CruxEval & SFT 1.7B & 5 &   97 & 7.3\% & [0.0, 17.1] & 52.6\% & [43.3, 61.9] & 45.4\% & [36.1, 55.7] \\
GPQA & SFT 1.7B & 1 &  180 & 15.6\% & [11.1, 21.1] & 15.2\% & [9.9, 20.5] & 28.9\% & [22.2, 35.6] \\
GPQA & SFT 1.7B & 5 &  150 & 11.3\% & [6.7, 16.7] & 12.3\% & [7.5, 17.8] & 23.3\% & [16.7, 30.0] \\
GSM8K & SFT 4B & 5 &   13 & 33.3\% & [0.0, 66.7] & 46.2\% & [23.1, 69.4] & 30.8\% & [7.7, 53.8] \\
\bottomrule
\end{tabular}%
}
\caption{\textbf{Bootstrap 95\% CIs for headline Fail@$K$ cells (10,000 resamples).}
  All methods evaluated on matched PID sets; bootstrap is applied at the problem level.
  Retry for CruxEval and GSM8K uses the global retry pool; for GPQA, the already-failed rollout pool indexed at $K{:}K{+}3$.
  GSM8K SFT 4B ($N{=}13$) CIs are too wide for directional claims.}
\label{tab:bootstrap-cis}
\end{table}

Table~\ref{tab:ancestor-ablation-ci} reports binomial 95\% CIs for
the original Detector-A ancestor identity ablation.
\textbf{This result is superseded by the junction-firing replication in
App.~\ref{app:ancestor-choice}}: under the junction detector
(the production calibration), the lineage advantage collapses to
$-0.6$ pp on the same 77-pid set. The Detector-A numbers below are
preserved for completeness and to document the discovered
detector-dependence. The paper's current position is that ancestor
identity is interchangeable within capability range under a
well-calibrated detector
(Table~\ref{tab:ancestor-choice}); the body framing has been
updated accordingly. Reading this table in isolation would
misrepresent that position.

\begin{table}[h]
\centering
\small
\begin{tabular}{l r l}
\toprule
\textbf{Ancestor} & \textbf{Rescue Rate} & \textbf{95\% CI} \\
\midrule
  Lineage (Qwen3 base)    & 49.4\% & [38.2, 60.6] \\
  Qwen2.5-1.5B (non-lin.) & 27.3\% & [17.3, 37.3] \\
  Qwen2.5-7B (non-lin.)   & 18.2\% & [9.6,  26.8] \\
\bottomrule
\end{tabular}
\caption{\textbf{Detector-A ancestor identity ablation (SFT 4B,
CruxEval, $N{=}77$), superseded; see
App.~\ref{app:ancestor-choice}.}
Binomial Wilson 95\% CIs. The non-overlapping CIs were the
load-bearing evidence for the (now retracted) lineage-specificity
claim. Under the junction detector on the same 77 pids, the
lineage--Qwen2.5-7B gap is $-0.6$ pp (no lineage advantage). The
result is detector-specific, not a paper-level finding.}
\label{tab:ancestor-ablation-ci}
\end{table}


\section{Qualitative cases and per-cell response profiles}
\label{app:qualitative-cases}

The regime classification and dispatch decisions look like the
following at the level of individual cases and individual
(model, task) cells. Three qualitative cases illustrate per-failure
diagnosis (Fig.~\ref{fig:routing-case-study});
Fig.~\ref{fig:ladder} sweeps the response profile across
representative (model, task) cells.


\providecommand{\cmark}{\textcolor{green!55!black}{\ding{51}}}
\providecommand{\xmark}{\textcolor{red!65!black}{\ding{55}}}

\begin{figure*}[!t]
\begin{tcolorbox}[
  title={%
    \textbf{Case A: local temperature lift rescues a junction-diffuse failure
    that retry and logit-steer cannot}
    \hfill{\small\texttt{sft-1.7B}, CruxEval \texttt{sample\_141}}},
  colback=green!3, colframe=green!40!black,
  fonttitle=\small\bfseries, before skip=4pt, after skip=4pt,
  left=4pt, right=4pt, top=3pt, bottom=3pt]
{\small
\textbf{Task (CruxEval).}
Trace \texttt{f(li) = [li.count(i) for i in li]} on input
\texttt{['k','x','c','x','x','b','l','f','r','n','g']}.
Expected output: \texttt{[1,\,3,\,1,\,3,\,3,\,1,\,1,\,1,\,1,\,1,\,1]}
(per-position counts).

\medskip
\textbf{Failed rollout} (T${=}0.6$, marked incorrect):
\begin{quote}\ttfamily\footnotesize
\textlangle think\textrangle\textlangle/think\textrangle\\
{}['k', 'x', 'c', 'x', 'x', 'b', 'l', 'f', 'r', 'n', 'g']\\
\\
The output of the function f is a list of counts of each element in the
input list \ldots\ 'k' appears 1 time, 'x' appears 3 times \ldots\ So the
final output is `['k', 'x', 'c', 'x', 'x', 'b', 'l', 'f', 'r', 'n', 'g']`.
\end{quote}
The model immediately commits to repeating the input list, then enumerates
the correct counts in prose, but the final-commit token re-anchors to the
input instead of the counts.

\medskip
\textbf{Feature profile.}
$J_{\mathrm{frac+}}{\approx}0.38$ (moderate spread);
$J_{\max}/J_{\mathrm{mean}}$ rank~0.92 (sharp spike at the early commit);
$\bar V_{\mathrm{traj}}{\approx}0.90$ (high — overall trajectory carries
high Fisher info), but $V_{t^\star}{\approx}0.13$ (low at the spike
token — the wrong commitment is locally rank-locked).
This profile lands in \textbf{Rank Misrouting (junction-diffuse)} at
high $\bar V_{\mathrm{traj}}$; the cell-level dispatch routes that cell
to $T_{\mathrm{loc}}{=}1.5$ (Table~\ref{tab:dispatch}).

\medskip
\textbf{Consistent with the dispatch routing.}
Sparse logit-steering has no movable target: the spike sits on a
high-confidence wrong token that the ancestor signal cannot displace
under bounded steering ($\alpha{=}0.7$). Retry at $T{=}0.6$ resamples
the same anchored distribution and exhausts the same trap. Local
temperature lift in the junction window injects enough variance to
escape the early commitment, after which the (correct) counting
reasoning is preserved.

\medskip
\textbf{Repair@3 (repair depths):}\ retry=\xmark\ SL-G=\xmark\ SL-R=\xmark\ $T_{\mathrm{loc}}{=}1.5$=\cmark
}
\end{tcolorbox}

\vspace{4pt}

\begin{tcolorbox}[
  title={%
    \textbf{Case B: SL-G rescues a Fail@5 geo-local failure that retry cannot}
    \hfill{\small\texttt{sft-1.7B}, GPQA \texttt{gpqa\_recAAJoHMW45Lv5je}}},
  colback=green!3, colframe=green!40!black,
  fonttitle=\small\bfseries, before skip=4pt, after skip=4pt,
  left=4pt, right=4pt, top=3pt, bottom=3pt]
{\small
\textbf{Task (GPQA Diamond).}
Chemistry: an equimolar mixture~X of two liquids decolorises bromine
water; mixture~Y of related compounds is given; hydrogenation of
either yields the same single product~Z. Identify the class of
compounds and the carbon count.

\medskip
\textbf{Failed rollout} (T${=}0.6$, 5 attempts, all incorrect; trace
runs 27k chars):
\begin{quote}\ttfamily\footnotesize
\textlangle think\textrangle\ \ldots\ that makes me think they are unsaturated
hydrocarbons, because bromine water decolorizes alkenes and alkynes. But
the problem mentions no conjugated multiple bonds in mixture X. So they
can't have double bonds, but maybe triple bonds? Or maybe alkenes with
some other structure? Wait, no, if they have\ldots\
\textrm{[continues for 27\,k characters, never commits]}
\end{quote}

\medskip
\textbf{Feature profile.}
$J_{\mathrm{frac+}}{\approx}0.62$ (broad spread cascading from an
early pivot);
$J_{\max}/J_{\mathrm{mean}}$ rank~0.81 (sharp spike);
$\bar V_{\mathrm{traj}}{\approx}0.82$ (high) but $V_{t^\star}{\approx}0.27$
at the junction — the pivot is steerable.
This profile lands in \textbf{Rank Misrouting (geo-local)} at low
$\bar V_{\mathrm{traj}}$ relative to the population median; the
cell-level dispatch routes that cell to \textbf{sparse logit-steer}
(Table~\ref{tab:dispatch}; cell-aggregate winner is SL-R, with
SL-G $\approx$ SL-R at the cell level — Appendix~\ref{app:bucket-analysis}).

\medskip
\textbf{Consistent with the dispatch routing.}
The early pivot carries the spike and is locally steerable, so a
single sparse logit-injection at that position restores the correct
framing and the rest of the chain follows (the broad downstream
spread is the \emph{consequence} of one wrong premise, not an
independent deformation). Retry at $T{=}0.6$ cannot locate the pivot
through sampling alone — the trace stays in the same local basin and
exhausts its budget rambling.

\medskip
\textbf{Repair@3 (repair depths):}\ retry=\xmark\ SL-G=\cmark\ SL-R=\xmark\ $T_{\mathrm{loc}}{=}1.5$=\xmark
}
\end{tcolorbox}

\caption{%
  Routing case studies under $V_{\mathrm{traj}}$-based dispatch.
  Cases A and B sit in two distinct (regime, $\bar V_{\mathrm{traj}}$)
  cells of the dispatch (Table~\ref{tab:dispatch}); the cell-level
  dispatch routes each to a different operator class.
  Case~A (CruxEval, Rank Misrouting (junction-diffuse), sft-1.7B):
  sharp early commit token where the wrong choice is rank-locked
  (low $V_{t^\star}$) but the trajectory carries high entropy
  elsewhere (high $\bar V_{\mathrm{traj}}$) — dispatch routes to
  local temperature lift, which breaks the commit.
  Case~B (GPQA, Rank Misrouting (geo-local), sft-1.7B): sharp spike
  at an early steerable pivot — dispatch routes to sparse logit-steer;
  SL-G at the junction-detected position rescues. The per-trace
  mechanistic stories illustrate why the cell-level routing is
  sensible; the dispatch lifts in \S\ref{sec:routing-result} are at the
  population scale, not per-pid. Case~C (Distributed Deformation)
  and the feature glossary: Appendix~\ref{app:qualitative-cases}.
}
\label{fig:routing-case-study}
\end{figure*}

\begin{table*}[h]
\centering
\small
\setlength{\tabcolsep}{6pt}
\resizebox{\textwidth}{!}{%
\begin{tabular}{l c c c l}
\toprule
Case & Regime & $J_{\mathrm{frac+}}$ & Junction rank & Cell-level dispatch \\
\midrule
\textbf{A.} CruxEval, sft-1.7B    & Rank Misrouting (junction-diffuse) & 0.38 & 0.92 (sharp early spike) & $T_{\mathrm{loc}}{=}1.5$ \\
\textbf{B.} GPQA, sft-1.7B        & Rank Misrouting (geo-local)        & 0.62 & 0.81 (sharp pivot)       & sparse logit-steer \\
\textbf{C.} CruxEval, sft-4B      & Distributed Deformation            & 0.53 & 0.005 (very diffuse)     & sparse logit-steer \\
\bottomrule
\end{tabular}}
\caption{\textbf{Where each case lands in the regime space under
$\bar V_{\mathrm{traj}}$-based dispatch.}
Cases A and B (this figure) and Case C
(Appendix~\ref{app:qualitative-cases}). The three features
$(J_{\mathrm{frac+}},\,J_{\max}/J_{\mathrm{mean}}\text{ rank},\,
\bar V_{\mathrm{traj}})$ place each case in a different regime/cell,
and the cell-level dispatch routes the regime to a different operator
class. The features predict the \emph{regime} (and the cell-level
dispatch operator class), not the per-pid outcome; aggregate dispatch
lifts (\S\ref{sec:routing-result}) are at the population scale.}
\label{tab:regime-contrast}
\end{table*}

\begin{figure*}[!t]
    \centering
    \includegraphics[width=\textwidth]{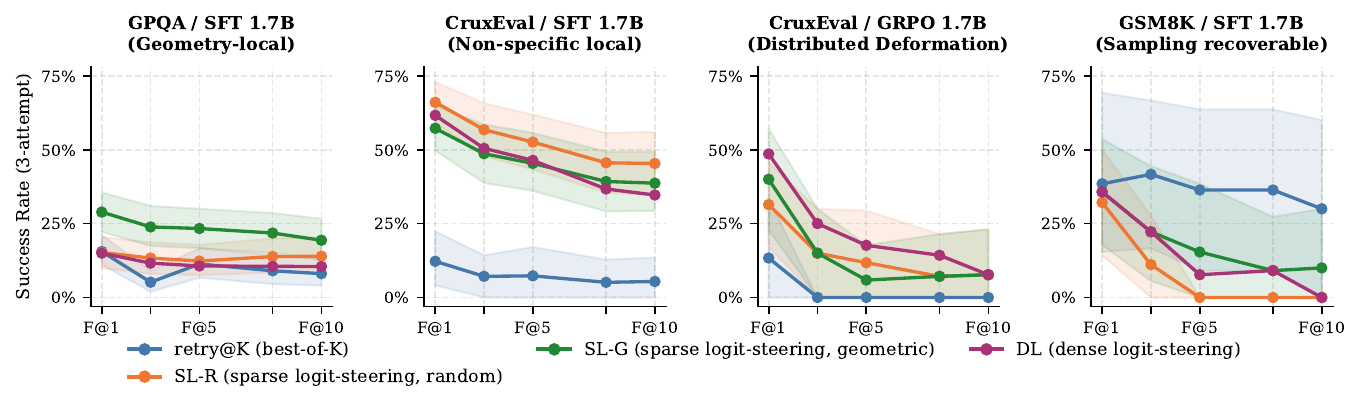}
    \caption{\textbf{Representative (model, task) response profiles
    across Fail@$K$ strata.}
    Problem-level rescue rates at Fail@$K$ for four representative
    (model, task) cells.
    \textbf{(A)} Rank Misrouting (geo-local) (GPQA, SFT 1.7B),
    SL-G outperforms random steering on hard problems;
    \textbf{(B)} Rank Misrouting (junction-diffuse) (CruxEval, SFT
    1.7B), local intervention rescues, but junction specificity
    adds little (one cell where this conclusion reverses at deep
    $K$; see \S\ref{sec:routing-result});
    \textbf{(C)} Distributed Deformation (CruxEval, SFT 4B),
    DL dominates;
    \textbf{(D)} Unresolved (GSM8K, SFT 1.7B), no probe
    consistently wins; retry is often competitive but does not
    dominate.
    Operator labels in the panel legends use the standard acronyms
    introduced in \S\ref{sec:setup}: retry@$K$, SL-G (sparse
    logit-steering, geometric position), SL-R (sparse logit-steering,
    random position), DL (dense logit-steering).}
    \label{fig:ladder}
\end{figure*}

\subsection*{Distributed-Deformation case study (Case C)}

Complementing Figure~\ref{fig:routing-case-study} (Rank Misrouting
(junction-diffuse) and Rank Misrouting (geo-local)), Case C below
illustrates the third regime from Table~\ref{tab:regime-contrast}:
Distributed Deformation at low $\bar V_{\mathrm{traj}}$, routed by
the cell-level dispatch to \textbf{SL-G}.

\providecommand{\cmark}{\textcolor{green!55!black}{\ding{51}}}
\providecommand{\xmark}{\textcolor{red!65!black}{\ding{55}}}

\begin{figure*}[!ht]
\begin{tcolorbox}[
  title={%
    \textbf{Case C: SL-G rescues a distributed-deformation failure that
    retry and SL-R cannot}
    \hfill{\small\texttt{sft-4B}, CruxEval \texttt{sample\_348}}},
  colback=green!3, colframe=green!40!black,
  fonttitle=\small\bfseries, before skip=4pt, after skip=4pt,
  left=4pt, right=4pt, top=3pt, bottom=3pt]
{\small
\textbf{Task (CruxEval).}
Trace \texttt{f(d) = d.copy()} on input \texttt{\{563:\,555,\,133:\,None\}}.

\medskip
\textbf{Failed rollout} (T${=}0.6$, marked incorrect):
\begin{quote}\ttfamily\footnotesize
\textlangle think\textrangle\ The function $f$ takes a dictionary and returns its
copy\ldots\ the \texttt{.copy()} method creates a shallow copy\ldots\ the returned
dictionary should have the same key-value pairs\ldots\ but wait, when you copy
a dictionary, the keys and values are copied, but if values are mutable
they are referenced\ldots\ in Python 3.7+ insertion order is preserved\ldots\
the return value is the dictionary object\ldots\textlangle/think\textrangle\\
\textbf{Return value:} \{563:\,555,\,133:\,None\}
\end{quote}
The trace contains the correct dict literal but the model bracketed
it with a markdown prefix (\texttt{**Return value:**}) that the
strict matcher rejects.

\medskip
\textbf{Feature profile.}
$J_{\mathrm{frac+}}{\approx}0.53$ (broad spread);
$J_{\max}/J_{\mathrm{mean}}$ rank~0.005 (extremely diffuse — no
dominant spike at any single token);
$\bar V_{\mathrm{traj}}$ in the low half of the population
(the trace's per-rollout Fisher info averaged across positions is
near the cohort median for the regime).
This profile lands in \textbf{Distributed Deformation} at low
$\bar V_{\mathrm{traj}}$; the cell-level dispatch routes that cell to
\textbf{SL-G} (Table~\ref{tab:dispatch}).

\medskip
\textbf{Consistent with the dispatch routing.}
The deformation is diffuse but the junction-detected token is still
the single most informative position to mix at: a sparse injection
there sharpens the commitment and the trace exits the rambling loop
with a correctly-formatted answer. Retry at $T{=}0.6$ stays in the
same rambling basin; random-position injection rarely lands on the
locally-strongest deformation under the same budget.

\medskip
\textbf{Repair@3 (repair depths):}\ retry=\xmark\ SL-G=\cmark\ SL-R=\xmark\ $T_{\mathrm{loc}}{=}1.5$=\cmark
}
\end{tcolorbox}
\caption{\textbf{Case C — Distributed Deformation routed to SL-G.}
A complementary regime to Cases A and B above
(Figure~\ref{fig:routing-case-study}). Despite the most diffuse
junction profile in the dataset (rank~$0.005$), the locally-strongest
position still carries enough deformation signal for a single sparse
injection to sharpen the commitment.}
\label{fig:case-c-distributed}
\end{figure*}


\subsection{SFT $\times$ CruxEval Regime Composition}
\label{app:cruxeval-mixing}

Section~\ref{sec:population} reports that SFT $\times$ CruxEval
splits into two qualitatively different regimes at our two scales:
\textbf{SFT 1.7B $\times$ CruxEval} is the regime-mixing hub
(broadly Unresolved with substantial minorities of every other
regime), while \textbf{SFT 4B $\times$ CruxEval} collapses largely
into Distributed Deformation (the same regime as GRPO at this scale).
This appendix gives the per-regime composition and per-regime
operator preferences inside the two cells.

\paragraph{Regime composition.}
\begin{center}
\resizebox{\columnwidth}{!}{%
\begin{tabular}{l rr rr}
\toprule
Regime  & \multicolumn{2}{c}{sft1p7b/cruxeval ($n{=}220$)} &
          \multicolumn{2}{c}{sft4b/cruxeval ($n{=}243$)} \\
        & $n$ & share & $n$ & share \\
\midrule
Unresolved                                             & 173 & 78.6\% &  76 & 31.3\% \\
Distributed Deformation (\textsf{DD})                  &   1 &  0.5\% & 103 & 42.4\% \\
Rank Misrouting (junction-diffuse) (\textsf{RM-D})     &  29 & 13.2\% &  62 & 25.5\% \\
Rank Misrouting (geo-local) (\textsf{RM-G})            &  17 &  7.7\% &   2 &  0.8\% \\
\bottomrule
\end{tabular}%
}
\end{center}

The two cells differ markedly. SFT 1.7B is dominated by Unresolved
($78.6\%$), with the remaining problem-units distributed across
\textsf{RM-D} ($13.2\%$), \textsf{RM-G} ($7.7\%$), and a single
Distributed Deformation outlier. SFT 4B, by contrast, has $42.4\%$
Distributed Deformation and $25.5\%$ Rank Misrouting
(junction-diffuse) and produces essentially no \textsf{RM-G}
problem-units. The shift from \textsf{RM-G}-and-Unresolved
at 1.7B to \textsf{DD}-and-\textsf{RM-D} at 4B is consistent with
larger SFT models inducing broader, more distributed deformation on
code tasks, the same structural signature §\ref{sec:population}
attributes to GRPO.

\paragraph{Operator preferences within each cell.}
\begin{center}
\resizebox{\columnwidth}{!}{%
\begin{tabular}{l l r rrrrr}
\toprule
Cell & Regime & $n$ & retry & SL-G & SL-R & DL & $T_{\mathrm{loc}}{=}1.5$ \\
\midrule
sft1p7b/cruxeval & Unresolved    & 173 & \textbf{0.676} & 0.642 & \textbf{0.676} & 0.578 & 0.671 \\
                 & \textsf{RM-D} &  29 & 0.483 & 0.552 & 0.414 & \textbf{0.621} & 0.448 \\
                 & \textsf{RM-G} &  17 & 0.529 & 0.529 & 0.529 & 0.412 & 0.529 \\
\midrule
sft4b/cruxeval   & \textsf{DD}   & 103 & 0.932 & \textbf{0.971} & 0.961 & 0.375 & \textbf{0.971} \\
                 & \textsf{RM-D} &  62 & 0.726 & 0.726 & 0.726 & 0.189 & 0.694 \\
                 & Unresolved    &  76 & 0.658 & 0.776 & 0.763 & 0.474 & \textbf{0.789} \\
\bottomrule
\end{tabular}%
}
\end{center}

Two structural observations: (i) within Unresolved/\textsf{RM-D} at
both SFT scales the rescue rates cluster within an $\sim\!8$-pp band,
which is the methodological-bucket fingerprint (operator choice is
fungible); and (ii) DL is consistently dominated as a
single operator on both cells, reinforcing the global finding in
§\ref{sec:limitations}.

\end{document}